\pdfoutput=1
\documentclass[11pt]{article} 

\usepackage{etoolbox}

\newtoggle{neurIPSformat}
\togglefalse{neurIPSformat} 
\newcommand{\neurIPS}[1]{\iftoggle{neurIPSformat}{#1}{}}
\newcommand{\arxiv}[1]{\iftoggle{neurIPSformat}{}{#1}} 

\arxiv{
\usepackage{parskip}

 \usepackage[letterpaper, left=1in, right=1in, top=1in, bottom=1in]{geometry}

\usepackage[colorlinks=true, linkcolor=blue!70!black, citecolor=blue!70!black]{hyperref}
\usepackage{microtype}

\usepackage{natbib}
\bibliographystyle{plainnat}
\bibpunct{(}{)}{;}{a}{,}{,}

\usepackage{amsthm}
\usepackage{mathtools}
\usepackage{amsmath}
\usepackage{bbm}
\usepackage{amsfonts}
\usepackage{amssymb}
\usepackage{xpatch}
\usepackage{array}
\usepackage{booktabs}
\usepackage{floatrow}
\newfloatcommand{capbtabbox}{table}[][\FBwidth]
\usepackage{blindtext}
\usepackage{caption}
} 

\neurIPS{
\PassOptionsToPackage{authoryear, round}{natbib}
\usepackage{neurips_2023}

}

\usepackage{comment}

\usepackage[utf8]{inputenc} % allow utf-8 input
\usepackage[T1]{fontenc}    % use 8-bit T1 fonts
\usepackage{hyperref}       % hyperlinks 
\usepackage{url} 
\usepackage{booktabs}       % professional-quality tables
\usepackage{amsfonts}       % blackboard math symbols
\usepackage{nicefrac}       % compact symbols for 1/2, etc.
\usepackage{microtype}      % microtypography
\usepackage{xcolor}         % colors
\usepackage{algorithm}
\usepackage{algorithmicx}
\usepackage{algpseudocode}
\usepackage{amssymb, bbold} 
\usepackage{nicefrac} 
\usepackage{bm}

\def\tdotoggle{0} % turn edits on 
\usepackage[suppress]{color-edits} % turn edits off

\addauthor{as}{red}
\addauthor{sr}{blue}
\addauthor{gl}{purple} 

\hypersetup{colorlinks=true,allcolors=blue}

\title{When is Agnostic Reinforcement Learning \\ Statistically Tractable?\thanks{Authors are listed in alphabetical order of their last names.}}

\usepackage{enumitem}      

\usepackage{amsmath} 
\usepackage{amsthm}
\usepackage{bbm}

\newcommand\numberthis{\addtocounter{equation}{1}\tag{\theequation}} 
\allowdisplaybreaks

\usepackage{mathtools}

\DeclarePairedDelimiter{\abs}{\lvert}{\rvert} 
\DeclarePairedDelimiter{\brk}{[}{]}
\DeclarePairedDelimiter{\crl}{\{}{\}}
\DeclarePairedDelimiter{\prn}{(}{)}

\DeclarePairedDelimiter{\floor}{\lfloor}{\rfloor}

\let\Pr\undefined
\let\P\undefined 
\DeclareMathOperator{\P}{P}
\DeclareMathOperator{\Pr}{\bbP}

\newcommand{\ind}[1]{\mathbbm{1}\crl*{#1}}    %Indicator
\newcommand{\eps}{\epsilon}

\newcommand{\ldef}{\vcentcolon=}

\newcommand{\wt}[1]{\widetilde{#1}}
\newcommand{\wh}[1]{\widehat{#1}}
\def\ddefloop#1{\ifx\ddefloop#1\else\ddef{#1}\expandafter\ddefloop\fi}
\def\ddef#1{\expandafter\def\csname bb#1\endcsname{\ensuremath{\mathbb{#1}}}}
\ddefloop ABCDEFGHIJKLMNOPQRSTUVWXYZ\ddefloop
\def\ddefloop#1{\ifx\ddefloop#1\else\ddef{#1}\expandafter\ddefloop\fi}
\def\ddef#1{\expandafter\def\csname b#1\endcsname{\ensuremath{\mathbf{#1}}}}
\ddefloop ABCDEFGHIJKLMNOPQRSTUVWXYZ\ddefloop
\def\ddef#1{\expandafter\def\csname c#1\endcsname{\ensuremath{\mathcal{#1}}}}
\ddefloop ABCDEFGHIJKLMNOPQRSTUVWXYZ\ddefloop
\def\ddef#1{\expandafter\def\csname h#1\endcsname{\ensuremath{\widehat{#1}}}}
\ddefloop ABCDEFGHIJKLMNOPQRSTUVWXYZabcdefghijklmnopqrsuvwxyz\ddefloop    % Not defined for t. 
\def\ddef#1{\expandafter\def\csname hc#1\endcsname{\ensuremath{\widehat{\mathcal{#1}}}}}
\ddefloop ABCDEFGHIJKLMNOPQRSTUVWXYZ\ddefloop
\def\ddef#1{\expandafter\def\csname t#1\endcsname{\ensuremath{\widetilde{#1}}}}
\ddefloop ABCDEFGHIJKLMNOPQRSTUVWXYZ\ddefloop
\def\ddef#1{\expandafter\def\csname tc#1\endcsname{\ensuremath{\widetilde{\mathcal{#1}}}}}
\ddefloop ABCDEFGHIJKLMNOPQRSTUVWXYZ\ddefloop

\newcommand{\KL}[2]{\mathrm{KL}{\prn*{#1 \| #2}}}
\newcommand{\kl}[2]{\mathrm{kl}{\prn*{#1 \| #2}}}

\usepackage{tikz}

\newcommand{\proman}[1]{\prn*{\romannumeral #1}}

\newcommand{\overleq}[1]{\overset{ #1}{\leq{}}}
\newcommand{\overgeq}[1]{\overset{#1}{\geq{}}}
\newcommand{\overeq}[1]{\overset{#1}{=}}

\newtheorem{assumption}{Assumption}

\newtheorem{corollary}{Corollary}
\newtheorem{proposition}{Proposition}
{\newtheorem{lemma}{Lemma}} 

{\newtheorem{theorem}{Theorem}}
\newtheorem{theorem*}{Theorem}
\newtheorem{definition}{Definition}
\usepackage{xpatch}
\xpatchcmd{\proof}{\itshape}{\normalfont\proofnameformat}{}{}
\newcommand{\proofnameformat}{\bfseries}

\usepackage{prettyref}
\newcommand{\pref}[1]{\prettyref{#1}}

\newcommand{\savehyperref}[2]{\texorpdfstring{\hyperref[#1]{#2}}{#2}}
\newrefformat{eq}{\savehyperref{#1}{\textup{(\ref*{#1})}}}
\newrefformat{eqn}{\savehyperref{#1}{Equation~\ref*{#1}}}
\newrefformat{con}{\savehyperref{#1}{Conjecture~\ref*{#1}}}
\newrefformat{lem}{\savehyperref{#1}{Lemma~\ref*{#1}}}
\newrefformat{def}{\savehyperref{#1}{Definition~\ref*{#1}}}
\newrefformat{line}{\savehyperref{#1}{line~\ref*{#1}}}
\newrefformat{thm}{\savehyperref{#1}{Theorem~\ref*{#1}}}
\newrefformat{corr}{\savehyperref{#1}{Corollary~\ref*{#1}}}
\newrefformat{sec}{\savehyperref{#1}{Section~\ref*{#1}}}
\newrefformat{app}{\savehyperref{#1}{Appendix~\ref*{#1}}}
\newrefformat{ass}{\savehyperref{#1}{Assumption~\ref*{#1}}}
\newrefformat{ex}{\savehyperref{#1}{Example~\ref*{#1}}}
\newrefformat{fig}{\savehyperref{#1}{Figure~\ref*{#1}}}
\newrefformat{alg}{\savehyperref{#1}{Algorithm~\ref*{#1}}}
\newrefformat{rem}{\savehyperref{#1}{Remark~\ref*{#1}}}
\newrefformat{conj}{\savehyperref{#1}{Conjecture~\ref*{#1}}}
\newrefformat{prop}{\savehyperref{#1}{Proposition~\ref*{#1}}}
\newrefformat{proto}{\savehyperref{#1}{Protocol~\ref*{#1}}}
\newrefformat{prob}{\savehyperref{#1}{Problem~\ref*{#1}}}
\newrefformat{claim}{\savehyperref{#1}{Claim~\ref*{#1}}}

\newcommand{\algcomment}[1]{\textcolor{blue!70!black}{\footnotesize{\texttt{\textbf{//
          #1}}}}}
\newcommand{\algcommentbig}[1]{\textcolor{blue!70!black}{\footnotesize{\texttt{\textbf{/*
          #1~*/}}}}}

\newcommand{\pistar}{{\pi^\star}} 
\newcommand{\pihat}{{\wh{\pi}}} 
\newcommand{\Piell}{\Pi^{(\ell)}}

\newcommand{\dimRL}{\mathfrak{C}}

\newcommand{\non}{n_\mathsf{on}}
\newcommand{\ngen}{n_\mathsf{gen}}

\newcommand{\cMsto}{ \cM^{\mathrm{sto}} }
\newcommand{\cMdet}{ \cM^{\mathrm{det}} }
\newcommand{\cMdetP}{ \cM^{\mathrm{detP}} }

\newcommand{\alg}{\bbA}

\newcommand{\poly}{\mathrm{poly}}
\newcommand{\unif}{\mathrm{Uniform}}
\newcommand{\ber}{\mathrm{Ber}}
\newcommand{\bit}{\mathrm{bit}}

\newcommand{\idx}{\mathrm{idx}}
\newcommand{\jof}[1]{j[#1]}
\newcommand{\jpof}[1]{j'[#1]}
\newcommand{\jgof}[1]{j_g[#1]}
\newcommand{\jbof}[1]{j_b[#1]}

\newcommand{\Jrel}{\cJ_\mathrm{rel}}

\newcommand{\st}{\eta}
\newcommand{\Tmax}{{T_{\mathrm{max}}}}

\newcommand{\En}{\bbE}

\newcommand{\Compname}{spanning capacity}

\newcommand{\sunflower}{sunflower}
\newcommand{\coreset}{\sunflower}

\usepackage{mathtools, amsmath} 
\newcommand{\cons}{\rightsquigarrow}  
\newcommand{\coloneqq}{:=}   
\newcommand{\plus}{+}

\newcommand{\datacollector}{\mathsf{DataCollector}}  

\newcommand{\evaluate}{\mathsf{Evaluate}}
\newcommand{\estreach}{\mathsf{EstReachability}}

\renewcommand{\epsilon}{\varepsilon}
\renewcommand{\eps}{\varepsilon}

\newcommand{\Picb}{\Pi_{\mathrm{cb}}} 
 
\newcommand{\Pitab}{\Pi_{\mathrm{tab}}} 
\newcommand{\Pismall}{\Pi_{\mathrm{small}}}  
\newcommand{\Pisingleton}{\Pi_{\mathrm{sing}}}
\newcommand{\PiLton}{\Pi_{\ell\mathrm{-ton}}} 
\newcommand{\Pioneactive}{\Pi_{\mathrm{1-act}}} 
\newcommand{\PiJactive}{\Pi_{j\mathrm{-act}}} 
\newcommand{\Piactive}{\Pi_{\mathrm{act}}}

\newcommand{\mS}{\mathcal{S}} 
\newcommand{\mA}{\mathcal{A}} 
\newcommand{\mD}{\mathcal{D}}

\newcommand{\EE}{\mathbb{E}}

\newcommand{\argmax}{\mathop{\arg\max}}

\newcommand{\SHPalg}[1]{\mathcal{S}_{#1}^{\mathrm{rch}}} 
\newcommand{\SRemalg}[1]{\mathcal{S}_{#1}^{\mathrm{rem}}} 
\newcommand{\SHP}[1]{\mathcal{S}_{#1}^{\mathrm{rch}}} 
\newcommand{\SRem}[1]{\mathcal{S}_{#1}^{\mathrm{rem}}}

\newcommand{\event}[3]{\mathfrak{T}\prn{#1 \rightarrow #2 ; \neg #3}} 

\newcommand{\bard}[3]{\bar{d}^{#1}(#2; \neg#3)}

\newcommand{\Picore}{\Pi_{\mathrm{core}}}

\newcommand{\Stab}{S^{\mathrm{tab}}}

\newcommand{\setall}[3]{ \mathfrak{T}_{#1}\prn{#2 \rightarrow #3}}

\usepackage{eufrak}  
\usepackage{mathrsfs}

\newcommand{\MRPsign}{\mathfrak{M}}

\newcommand{\eluder}{\mathrm{dim}_\mathsf{E}}
\newcommand{\starr}{\mathrm{dim}_\mathsf{S}}
\newcommand{\thres}{\mathrm{dim}_\mathsf{T}}

\newcommand{\natarajan}{\mathrm{Ndim}}
\newcommand{\pseudo}{\mathrm{Pdim}}

\newcommand{\oracle}{\mathsf{O}_\mathrm{exp}} 

\newcommand{\reachablestates}{\cS^\mathrm{rch}} 
\usepackage[colorinlistoftodos,prependcaption,textsize=scriptsize]{todonotes} %\usepackage{xargs} % to use more than one optional parameter in a new commands
\providecommand{\tdotoggle}{1}
 \ifnum\tdotoggle=1
 \paperwidth=\dimexpr \paperwidth + 3.2cm\relax
 \oddsidemargin=\dimexpr\oddsidemargin + 1.6cm\relax
 \evensidemargin=\dimexpr\evensidemargin + 1.6cm\relax
 \marginparwidth=\dimexpr \marginparwidth + 1.6cm\relax
 \fi
\newcommand{\mytodo}[1]{\ifnum\tdotoggle=1{#1}\fi}

\newcommand{\tableoftodos}{\ifnum\tdotoggle=1 \listoftodos[Comments/To Do's] \fi}

\newcommand{\gene}[1]{\mytodo{\todo[linecolor=Ggreen,backgroundcolor=Ggreen!25,bordercolor=Ggreen]{Gene: #1}}}

\newcommand{\ayush}[1]{\mytodo{\todo[linecolor=blue,backgroundcolor=blue!25,bordercolor=blue]{Ayush: #1}}}

\definecolor{Gred}{RGB}{219, 50, 54}
\definecolor{Ggreen}{RGB}{60, 186, 84}
\definecolor{Gblue}{RGB}{72, 133, 237}
\definecolor{Gyellow}{RGB}{247, 178, 16}
\definecolor{ToCgreen}{RGB}{0, 128, 0}
\definecolor{myGold}{RGB}{231,141,20}
\definecolor{myBlue}{rgb}{0.19,0.41,.65}
\definecolor{myPurple}{RGB}{175,0,124} 

\author{
  Zeyu Jia$^{1}$\\
  \and
  Gene Li$^{2}$ \\
  \and
  Alexander Rakhlin$^{1}$ \\
  \and
  Ayush Sekhari$^{1}$ \\
  \and
  Nathan Srebro$^{2}$ \\ 
  \and 
  $^{1}$MIT, $^{2}$TTIC 
} 

\date{} 

\usepackage{yfonts} 

\begin{document} 

\maketitle 

\begin{abstract}

We study the problem of agnostic PAC reinforcement learning (RL): given a policy class $\Pi$, how many rounds of interaction with an unknown MDP (with a potentially large state and action space) are required to learn an $\epsilon$-suboptimal policy with respect to \(\Pi\)? Towards that end, we introduce a new complexity measure, called the \emph{\Compname{}}, that depends solely on the set \(\Pi\) and is independent of the MDP dynamics. With a generative model, we show that for any policy class $\Pi$, bounded \Compname{} 
 characterizes PAC learnability. However, for online RL, the situation is more subtle. We show there exists a policy class $\Pi$ with a bounded \Compname{} that requires a superpolynomial number of samples to learn. This reveals a surprising separation for agnostic learnability between generative access and online access models (as well as between deterministic/stochastic MDPs under online access). On the positive side, we identify an additional \textit{sunflower} structure, which in conjunction with bounded \Compname{} enables statistically efficient online RL via a new algorithm called POPLER, which takes inspiration from classical importance sampling methods as well as techniques for reachable-state 
  identification and policy evaluation in reward-free exploration.
  \end{abstract} 
\section{Introduction}
Reinforcement Learning (RL) has emerged as a powerful paradigm for solving complex decision-making problems, demonstrating impressive empirical successes in a wide array of challenging tasks, from achieving superhuman performance in the game of Go  \citep{silver2017mastering} to solving intricate robotic manipulation tasks \citep{lillicrap2015continuous, akkaya2019solving, ji2023dribblebot}. Many practical domains in RL often involve rich observations such as images, text, or audio \citep{mnih2015human, li2016deep, ouyang2022training}. Since these state spaces can be vast and complex, traditional tabular RL approaches \citep{kearns2002near, brafman2002r, azar2017minimax, jin2018q} cannot scale. This has led to a need to develop provable and efficient approaches for RL that utilize \emph{function approximation} to \gledit{generalize observational data to unknown states/actions.}   

% With the rapid progress in machine learning, new applications that can leverage the power of RL are quickly emerging, particularly in contexts involving rich observations such as images, text, or audio. Here, the state space can be vast and complex, making traditional, tabular RL approaches unscalable. This has led to a need for developing provable and efficient approaches for RL that can handle large state and action spaces.

The goal of this paper is to study the sample complexity of policy-based RL, which is arguably the simplest setting for RL with function approximation \citep{kearns1999approximate, kakade2003sample}. In policy-based RL, an abstract function class $\Pi$ of \emph{policies} (mappings from states to actions) is given to the learner. For example, $\Pi$ can be the set of all the policies represented by a certain deep neural network architecture. The objective of the learner is to interact with an unknown MDP to find a policy $\wh{\pi}$ that competes with the best policy in $\Pi$, i.e., for some prespecified $\eps$, the policy $\wh{\pi}$ satisfies 
% Perhaps, the first step in tackling a challenging RL task is for the learner to select a class \(\Pi\) of policies, e.g., a deep neural network, that she believes is sufficient to represent a well-performing policy for the underlying MDP. After choosing this class, the learner aims to find a policy \(\widehat{\pi}\), after interacting with the unknown MDP, that performs at least as well as the best policy in the chosen class. In particular, the learner desires that for some small \(\epsilon\), \(\widehat{\pi}\) satisfies the guarantee 
\begin{align*} 
V^{\widehat{\pi}} \geq \max_{\pi \in \Pi} V^\pi - \epsilon,  \numberthis \label{eq:agnostic_RL_guarantee}
\end{align*} 
where \(V^\pi\) denotes the value of policy \(\pi\) on the underlying MDP. We henceforth call Eq.~\eqref{eq:agnostic_RL_guarantee} the ``agnostic PAC reinforcement learning'' objective. Our paper addresses the following question:
\begin{center}
\textit{What structural assumptions on $\Pi$ enable \\ statistically efficient agnostic PAC reinforcement learning?} 
\end{center}

Characterizing (agnostic) learnability for various problem settings is perhaps the most fundamental question in statistical learning theory. 
For the simpler setting of supervised learning (which is RL with binary actions, horizon 1, and binary rewards), the story is complete: a hypothesis class \(\Pi\) is agnostically learnable if and only if its $\mathrm{VC}$ dimension is bounded \citep{vapnik1971uniform, vapnik1974theory, blumer1989learnability, ehrenfeucht1989general}, and the ERM algorithm---which returns the hypothesis with the smallest training loss---is statistically optimal (up to log factors). However, RL (with \(H > 1\)) is significantly more challenging, and we are still far from a rigorous understanding of when agnostic RL is statistically tractable, or what algorithms to use in large-scale RL problems.

While significant effort has been invested over the past decade in both theory and practice to develop algorithms that utilize function approximation, existing theoretical guarantees require additional assumptions on the MDP. The most commonly adopted assumption is  \emph{realizability}: the learner can precisely model the value function or the dynamics of the underlying MDP \citep[see, e.g.,][]{russo2013eluder, jiang2017contextual, sun2019model, wang2020reinforcement, du2021bilinear, jin2021bellman, foster2021statistical}.
Unfortunately, realizability is a fragile assumption that rarely holds in practice. Moreover, even mild misspecification can cause catastrophic breakdown of theoretical guarantees \citep{du2019good, lattimore2020learning}. Furthermore, in various applications, the optimal policy \(\pi^\star \coloneqq \argmax_{\pi \in \Pi} V^\pi\) may have a succinct representation, but the optimal value function \(V^\star\) can be highly complex, rendering accurate approximation of dynamics/value functions infeasible without substantial domain knowledge \citep{dong2020expressivity}. Thus, we desire algorithms for agnostic RL that can work with \emph{no modeling assumptions on the underlying MDP}. On the other hand, it is also well known without any assumptions on \(\Pi\), when \(\Pi\) is large and the MDP has a large state and action space, agnostic RL may be intractable with sample complexity scaling exponentially in the horizon \citep{agarwal2019reinforcement}. Thus, some structural assumptions on $\Pi$ are needed, and towards that end, the goal of our paper is to understand what assumptions are sufficient or necessary for statistically efficient agnostic RL, and to develop provable algorithms for learning. Our main contributions are: % as follows: 

\neurIPS{\vspace{-\topsep}
\begin{enumerate}[label=\(\bullet\), leftmargin=5mm, itemsep=0.2mm]}
\arxiv{\begin{enumerate}[label=\(\bullet\)]}
    \item We introduce a new complexity measure called the \textit{\Compname{}}, which solely depends on the policy class $\Pi$ and is independent of the underlying MDP. We illustrate the \Compname{} with examples, and show why it is a natural complexity measure for agnostic PAC RL (\pref{sec:C_pi}).
    \item We show that the \Compname{} is both necessary and sufficient for agnostic PAC RL with a generative model, with upper and lower bounds matching up to $\log \abs{\Pi}$ and $\poly(H)$ factors (\pref{sec:generative}). Thus, bounded \Compname{} characterizes agnostic PAC learnability in RL with a generative model. 
    \item Moving to the online setting, we first show that the bounded \Compname{} by itself is \emph{insufficient} for agnostic PAC RL by proving a superpolynomial lower bound on the sample complexity required to learn a specific $\Pi$, thus demonstrating a separation between generative and online interaction models for agnostic PAC RL (\pref{sec:online-lower-bound}).
    \item Given the previous lower bound, we propose an additional property of the policy class called the \emph{sunflower} property, that allows for efficient exploration and is satisfied by many policy classes of interest. We provide a new agnostic PAC RL algorithm called $\mathsf{POPLER}$ that is statistically efficient whenever the given policy class has both bounded \Compname{} and the sunflower property (\pref{sec:upper-bound}). $\mathsf{POPLER}$ leverages importance sampling as well as reachable state identification techniques to estimate the values of policies. Our algorithm and analysis utilize a new tool called the \emph{policy-specific Markov reward process}, which may be of independent interest.
\end{enumerate}\neurIPS{\vspace{-\topsep}}
\section{Setup and Motivation} 

We begin by introducing our setup for reinforcement learning (RL), the relevant notation, and the goal of agnostic RL. 

\subsection{RL Preliminaries} 
We consider reinforcement learning in an episodic Markov decision process (MDP) with horizon \(H\). 
\paragraph{Markov Decision Processes.} Denote the MDP as $M = \mathrm{MDP}(\cS, \cA, P, R, H, \mu)$, which consists of a state space $\cS$, action space $\cA$, horizon $H$, probability transition kernel $P: \cS \times \cA \to \Delta(\cS)$, reward function $R: \cS \times \cA\to \Delta([0,1])$, and initial distribution $\mu \in \Delta(\cS)$. For ease of exposition, we assume that $\cS$ and $\cA$ are finite (but possibly large) with cardinality $S$ and $A$ respectively. We assume a layered state space, i.e., $\cS = \cS_1 \cup \cS_2 \cup \dots \cup \cS_H$ where $\cS_i \cap \cS_j = \emptyset$ for all \(i \neq j\). %$i, j \in [H]$, $i \ne j$. 
Thus, given a state $s\in \cS$, it can be inferred which $\cS_h$, or layer in the MDP, it belongs to. We denote a trajectory $\tau = (s_1, a_1, r_1, \dots, s_H, a_H, r_H)$, where at each step $h\in [H]$, an action $a_h \in \cA$ is played, a reward $r_h$ is drawn independently from the distribution $R(s_h, a_h)$, and each subsequent state $s_{h+1}$ is drawn from $P(\cdot|s_h, a_h)$.  Lastly, we assume that the cumulative reward of any trajectory is bounded by 1. % i.e., $\sum_{h=1}^H r_h \in [0,1]$ a.s.

\paragraph{Markov Reward Process (MRP).}  MRPs are key technical tools used by our main algorithm.  An MRP  \(\MRPsign = \mathrm{MRP}(\cS, P, R, H, s_\top, s_\bot)\) is defined over the state space \(\cS\) with start state \(s_\top\) and end state \(s_\bot\),  for trajectory length \(H + 2\). Without loss of generality, we assume that \(\crl{s_\top, s_\bot} \in \cS\). The transition kernel is denoted by \(P: \cS \times \cS \to [0,1] \), such that for any $s \in \cS$, $\sum_{s'} P_{s \to s'} = 1$. The reward kernel is denoted \(R: \cS \times \cS \to \Delta([0,1])\). Throughout, we use the notation $\rightarrow$ to signify that the transitions and rewards are defined along the edges of the MRP. At an intuitive level, an MRP is an MDP with singleton action space i.e.~the actions have no effect. For technical reasons, that will become more clear from the analysis, the state space in the MRPs is not layered.

\paragraph{Policy-based Reinforcement Learning.} We assume that the learner is given a policy class $\Pi \subseteq \cA^\cS$.\footnote{Throughout the paper, we assume that the policy classes $\Pi$ under consideration consist of deterministic policies. Extending our work to stochastic policy classes is an interesting direction for future research.} For any policy $\pi \in \cA^\cS$, we denote $\pi(s)$ as the action that $\pi$ takes when presented a state $s$. We use $\En^\pi[\cdot]$ and $\Pr^\pi[\cdot]$ to denote the expectation and probability under the process of a trajectory drawn from the MDP $M$ by policy $\pi$. Additionally, for any \(h, h' \leq H\), we say that a partial trajectory $\tau =$ $(s_{h}, a_h, s_{h+1}, a_{h+1}, \dots, s_{h'}, a_{h'})$ is consistent with \(\pi\) if for all \(h \leq i \leq h'\), we have $\pi(s_i) = a_i$. We use the notation \(\pi \cons \tau\) to denote that \(\tau\) is consistent with \(\pi\). 

% \begin{definition}[Consistent trajectory] 
%     For any \(h, h' \leq H\), we say that a partial trajectory $\tau =$ $(s_{h}, a_h, s_{h+1}, a_{h+1}, \dots, s_{h'}, a_{h'})$ is consistent with a policy \(\pi\) if for all \(h \leq i \leq h'\), we have $\pi(s_i) = a_i$. We use the notation \(\pi \cons \tau\) to denote that \(\tau\) is consistent with \(\pi\). 
% \end{definition}

The state-value function (also called \textit{\(V\)-function}) and state-action-value function (also called \textit{\(Q\)-function}) are defined such that for any \(\pi\), and $s, a$, \begin{align*}
    V^\pi_h(s) =  \bbE^{\pi}\brk*{\sum_{h' = h}^H R(s_{h'}, a_{h'}) ~|~ s_{h} = s }, ~~ Q^\pi_h(s,a) = \bbE^{\pi}\brk*{ \sum_{h' = h}^H R(s_{h'}, a_{h'}) ~|~ s_{h} = s, a_{h} = a }. 
\end{align*}

Furthermore, whenever clear from the context, we denote $V^\pi \coloneqq \En_{s_1 \sim \mu} V^\pi_1(s_1)$. Finally, for any policy $\pi \in \cA^\cS$, we also define the \emph{occupancy measure} as $d^\pi_h(s,a) \coloneqq \bbP^{\pi}[s_h=s, a_h =a]$ and $d^\pi_h(s) \coloneqq \bbP^{\pi}[s_h = s]$.

\paragraph{Models of Interaction.} We consider two standard models of interaction in the RL literature:  
\neurIPS{\vspace{-\topsep}
\begin{enumerate}[label=\(\bullet\), leftmargin=5mm, itemsep=0.2mm]}
\arxiv{\begin{enumerate}[label=\(\bullet\)]}
    \item \textbf{Generative Model.} The learner has access to a simulator which it can query for any $(s,a)$, and observe a sample $(s', r)$ drawn as $s' \sim P(\cdot|s,a)$ and $r \sim R(s,a)$.\footnote{Within the generative model, one can further distinguish between a more restrictive ``local'' access model (also called the ``reset'' model), where the learner can query $(s,a)$ for any $s \in \cS$ that it has seen already, or ``global'' access, where the learner can query for any $(s,a)$ without restriction. For the generative model, our upper bounds hold in the local access model, while our lower bounds hold for the global access model.} 
    \item \textbf{Online Interaction Model.} The learner can submit a (potentially non-Markovian) policy $\wt{\pi}$ and receive back a trajectory sampled  by running $\wt{\pi}$ on the MDP. Since online access can be simulated via generative access, learning under online access is only more challenging than learning under generative access (up to a factor of $H$). Adhering to commonly used terminology, we will refer to RL under the online interaction model as ``online RL''. % We colloquially refer to this as ``online RL''. 
\end{enumerate}\neurIPS{\vspace{-\topsep}}
We define $\cMsto$ as the set of all (stochastic and deterministic) MDPs of horizon $H$ over the state space $\cS$ and action space $\cA$. Additionally, we define $\cMdetP \subset \cMsto$ and $\cMdet \subset \cMdetP$  to denote the set of all MDPs with deterministic transitions but stochastic rewards, and of all MDPs with both deterministic transitions and deterministic rewards, respectively.

\subsection{Agnostic PAC RL}\label{sec:agnostic-pac-rl}
Our goal is to understand the sample complexity of agnostic PAC RL, i.e.,~the number of interactions required to find a policy that can compete with the best policy within the given class \(\Pi\) for the underlying MDP. An algorithm \(\alg\) is an \((\epsilon, \delta)\)-PAC RL algorithm for an MDP \(M\), if after interacting with \(M\) (either in the generative model or online RL), \(\alg\) returns a policy $\wh{\pi}$ that satisfies the guarantee\footnote{Our results are agnostic in the sense that we do not make the assumption that the optimal policy for the underlying MDP is in \(\Pi\), but instead, only wish to complete with the best policy in \(\Pi\). Additionally, we do not assume that the learner has a value function class or a model class that captures the optimal value functions or dynamics.}
\begin{align*}
V^{\wh{\pi}} \ge \max_{\pi \in \Pi} V^\pi - \eps,
\end{align*} with probability at least \(1 - \delta\). For a policy class $\Pi$ and a MDP class $\cM$, we say that $\alg$ has sample complexity $\non^{\alg}( \Pi, \cM; \eps, \delta)$ (resp.~$\ngen^{\alg}(\Pi, \cM; \eps, \delta)$) if for every MDP $M \in \cM$, $\alg$ is an \((\epsilon, \delta)\)-PAC RL algorithm and collects at most $\non^\alg(\Pi, \cM; \eps, \delta)$ trajectories in the online interaction model (resp.~generative model) in order to return \(\widehat{\pi}\).

We define the \emph{minimax sample complexity} for agnostically learning $\Pi$ over $\cM$ as the minimum sample complexity of any $(\eps, \delta)$-PAC RL algorithm, i.e.~
\begin{align*}
    \non(\Pi, \cM; \eps, \delta) \coloneqq \min_{\alg} \non^\alg(\Pi, \cM; \eps, \delta), \quad\text{and}\quad \ngen(\Pi, \cM; \eps, \delta) \coloneqq \min_{\alg} \ngen^{\alg}( \Pi, \cM; \eps, \delta).
\end{align*}
\gledit{For brevity, when $\cM = \cMsto$, we will drop the dependence on $\cM$ in our notation, e.g., we will write $\non(\Pi; \eps, \delta)$ and $\ngen(\Pi; \eps, \delta)$ to denote $\non(\Pi, \cM; \eps, \delta)$ and $\ngen(\Pi, \cM; \eps, \delta)$ respectively.} 

\paragraph{Known Results in Agnostic RL.} We first note the following classical result which shows that agnostic PAC RL is statistically intractable in the worst case. 
\begin{proposition}[No Free Lunch Theorem for RL; \citet{krishnamurthy2016pac}]\label{prop:nfl-rl} There exists a policy class $\Pi$ for which the minimax sample complexity under a generative model is at least $\ngen(\Pi; \eps, \delta) = \Omega(\min\crl{A^H, \abs{\Pi}, SA} / \eps^2)$. 
\end{proposition} 

Since online RL is only harder than learning with a generative model, the lower bound in \pref{prop:nfl-rl} extends to online RL. \pref{prop:nfl-rl} is the analogue of the classical \emph{No Free Lunch} results in statistical learning theory \citep{shalev2014understanding}; %\ayush{Can remove the first footnote}\footnote{To see this, observe that standard binary classification corresponds to $H=1$, $A=2$, and 0/1-valued rewards in the RL setup.} 
it indicates that without placing further assumptions on the MDP or the policy class \(\Pi\) (e.g., by introducing additional structure or constraining the state/action space sizes, policy class size, or the horizon), sample efficient agnostic PAC RL is not possible. 

Indeed, an almost matching upper bound of $\non(\Pi; \eps, \delta) = \wt{\cO} \prn{ \min\crl{A^H, \abs{\Pi}, HSA} /\eps^2}$
% \begin{align*}
%     \non(\Pi; \eps, \delta) = \wt{\cO} \prn*{\tfrac{ \min\crl{A^H, \abs{\Pi}, HSA} }{\eps^2}}
% \end{align*} 
is quite easy to obtain. The $\abs{\Pi}/\eps^2$ guarantee can simply be obtained by iterating over $\pi \in \Pi$, collecting  $\wt{\cO}(1/\eps^2)$ trajectories per policy, and  then picking the policy with highest empirical value. The $HSA/\eps^2$ guarantee can be obtained by running known algorithms for tabular RL \citep{zhang2021reinforcement}. Finally, the $A^H/\eps^2$ guarantee is achieved by the classical importance sampling (IS) algorithm \citep{kearns1999approximate, agarwal2019reinforcement}. Since importance sampling will be an important technique that we repeatedly use and build upon in this paper, we give a formal description of the algorithm below:
\neurIPS{\vspace{-\topsep}
\begin{enumerate}[label=\(\bullet\), leftmargin=5mm, itemsep=0.2mm]
    \item  Collect $n = \cO(A^H \log \abs{\Pi}/\eps^2)$ trajectories by executing actions $(a_1, \dots, a_H)\sim \unif(\cA^H)$. % and running the policy that plays $a_h$ in step $h$.
    
    %\item  Collect $n = O(A^H \cdot \log \abs{\Pi}/\eps^2)$ trajectories by sampling a random action sequence $(a_1, \dots, a_H)\sim \unif(\cA^H)$ and running the policy that plays $a_h$ in step $h$.
    \item Return $\pihat = \argmax_{\pi \in \Pi} \wh{v}^\pi_{\mathrm{IS}}$, where \(\wh{v}^\pi_{\mathrm{IS}} \coloneqq \tfrac{A^H}{n} \sum_{i=1}^n \ind{\pi \cons \tau^{(i)}} \prn{\sum_{h=1}^H r_h^{(i)}}\).
    % \begin{align*}
    % \wh{v}^\pi_{\mathrm{IS}} :=\frac{A^H}{n} \sum_{i=1}^n \ind{\pi \cons \tau^{(i)}} \prn*{\sum_{h=1}^H r_h^{(i)}}. 
    % \end{align*} 
\end{enumerate}\vspace{-\topsep}}

\arxiv{
\begin{center}
\noindent\fbox{%
    \parbox{0.90\textwidth}{~%
$\mathsf{ImportanceSampling}$:
\begin{enumerate}[label=\(\bullet\)]
    \item  Collect $n = \cO(A^H \log \abs{\Pi}/\eps^2)$ trajectories by executing $(a_1, \dots, a_H)\sim \unif(\cA^H)$. % and running the policy that plays $a_h$ in step $h$.
    
    %\item  Collect $n = O(A^H \cdot \log \abs{\Pi}/\eps^2)$ trajectories by sampling a random action sequence $(a_1, \dots, a_H)\sim \unif(\cA^H)$ and running the policy that plays $a_h$ in step $h$.
    \item Return $\pihat = \argmax_{\pi \in \Pi} \wh{v}^\pi_{\mathrm{IS}}$, where \(\wh{v}^\pi_{\mathrm{IS}} \coloneqq \tfrac{A^H}{n} \sum_{i=1}^n \ind{\pi \cons \tau^{(i)}} \prn{\sum_{h=1}^H r_h^{(i)}}\).
    % \begin{align*}
    % \wh{v}^\pi_{\mathrm{IS}} :=\frac{A^H}{n} \sum_{i=1}^n \ind{\pi \cons \tau^{(i)}} \prn*{\sum_{h=1}^H r_h^{(i)}}. 
    % \end{align*} 
\end{enumerate}
    }%
}
\end{center}
}
For every $\pi \in \Pi$, the quantity $\wh{v}^\pi_{\mathrm{IS}}$ is an unbiased estimate of $V^\pi$ with variance $A^H$; the sample complexity result follows by standard concentration guarantees \citep[see, e.g.,][]{agarwal2019reinforcement}.

\paragraph{Towards Structural Assumptions for Statistically Efficient Agnostic PAC RL.} Of course, No Free Lunch results do not necessarily spell doom---for example, in supervised learning, various structural assumptions have been studied that enable statistically efficient learning. Furthermore, there has been a substantial effort in developing complexity measures like VC dimension, fat-shattering dimension, covering numbers, etc.~ that characterize agnostic PAC learnability under different scenarios \citep{shalev2014understanding}. %These assumptions enable statistically efficient learning and often hold in practice. In RL theory, an extensive line of work \cite{russo2013eluder, jiang2017contextual, wang2020reinforcement, du2021bilinear, jin2021bellman, foster2021statistical} attempts to circumvent the No Free Lunch by placing additional \emph{realizability} assumptions on the MDP by constraining either the MDP or value function to lie inside class with restricted statistical complexity, and we often have nearly tight characterizations of learnability under realizability \cite{foster2021statistical}. 
In this paper, we consider the agnostic reinforcement learning setting, and explore whether there exists a complexity measure that characterizes learnability for every policy class $\Pi$. 
Formally, can we establish a complexity measure  $\dimRL$ (a function that maps policy classes to real numbers), such that for any \(\Pi\), the minimax sample complexity  satisfies
%Formally, can we establish a complexity measure  $\dimRL(\Pi)$  such that for any \(\Pi\), the minimax sample complexity  satisfies 
\begin{align*}
    \non(\Pi; \eps, \delta) = \wt{\Theta}\prn*{\poly\prn*{\dimRL(\Pi), H, \eps^{-1}, \log \delta^{-1} } }, 
\end{align*} 
where \(\dimRL(\Pi)\) denotes the complexity of \(\Pi\). \gledit{We mainly focus on finite (but large) policy classes and assume that the \(\log \abs{\Pi}\) factors in our upper bounds are mild. In \pref{app:infinite_policy_classes}, we discuss how our results can be extended to infinite policy classes.}

\paragraph{Is \pref{prop:nfl-rl} Tight for Every $\Pi$?} % \gene{I think this is a more succinct title.} 
In light of \pref{prop:nfl-rl}, one obvious candidate is $\overline{ \dimRL}(\Pi) = \min\crl{A^H, \abs{\Pi}, SA}$. While \(\overline{ \dimRL}(\Pi)\) is definitely sufficient to upper bound the minimax sample complexity for any policy class \(\Pi\) up to log factors, a priori it is not clear if it is also necessary for every policy class \(\Pi\). In fact, our next proposition implies that  $\overline{ \dimRL}(\Pi)$ is indeed not the right measure of complexity by giving an example of a policy class for which $\overline{ \dimRL}(\Pi) \coloneqq \min\crl{A^H, \abs{\Pi}, SA}$ is exponentially larger than the minimax sample complexity for agnostic learning for that policy class, even when \(\epsilon\) is constant. 

% the right measure of complexity i.e.~if it is necessary for sample efficient learning. In the following, we provide an example of a policy class for which such a choice of \(\dimRL(\Pi)\) is clearly wrong, as it is exponentially larger than 

% A priori, it is not obvious that this quantity cannot characterize learnability for every $\Pi$! Certainly, it characterizes learnability for tabular RL as well as contextual bandits. We show that the quantity $\min\crl{A^H, \abs{\Pi}, SA}$ \emph{cannot} characterize learnability for any $\Pi$ by exhibiting a policy class for which $\min\crl{A^H, \abs{\Pi}, SA}$ is exponentially larger than $\non(\Pi; \eps, \delta)$.

\begin{proposition}\label{prop:singleton-minimax}
Let $H \in \bbN$, \(K \in \bbN\), $\cS_h = \crl*{s_{(i,h)} : i \in [K]}$ for all $h\in [H]$, and $\cA = \crl{0,1}$. Consider the singleton policy class:
%\begin{align*} 
    $\Pisingleton \coloneqq \crl*{\pi_{(i', h')}: i' \in [K], h' \in [H]}$, 
%\end{align*}
where \(\pi_{(i', h')}\) takes the action \(1\) on state \(s_{(i', h')}\), and \(0\) everywhere else. Then $\min\crl{A^H, \abs{\Pisingleton}, SA} = 2^H$ but $\non(\Pisingleton; \eps,\delta) \le \widetilde{\cO}(H^3 \cdot \log(1/\delta)/\eps^2)$.
\end{proposition} 

The above upper bound on minimax sample complexity holds arbitrarily large values of \(K\), and can be obtained as a corollary of our more general upper bound in \pref{sec:upper-bound}. The key intuition for why $\Pisingleton$ can be learned in $\poly(H)$ samples is that even though the policy class and the state space are large when \(K\) is large, the set of possible trajectories obtained by running any $\pi \in \Pisingleton$ has low complexity. In particular, every trajectory $\tau$ has at most one $a_h = 1$. This observation enables us to employ the straightforward modification of the classical IS algorithm: draw $\poly(H) \cdot \log(1/\delta)/\eps^2$ samples from the uniform distribution over $\Picore = \crl*{\pi_h : h \in [H]}$ where the policy \(\pi_h\) takes the action 1 on every state at layer \(h\) and \(0\) everywhere else. %; every possible trajectory $\tau$ is consistent with some $\pi_e \in \Picore$. 
The variance of the resulting estimator $\wh{v}^\pi_{\mathrm{IS}}$ is $1/H$, so the sample complexity of this modified variant of IS has only $\poly(H)$ dependence by standard concentration bounds. 

In the sequel, we present a new complexity measure that formalizes this intuition that a policy class \(\Pi\) is efficiently learnable if the set of trajectories induced by policies in \(\Pi\) is small. % and investigate whether it characterizes agnostic PAC RL.

\section{Spanning Capacity}   \label{sec:C_pi}  
% In this section we define the \Compname{} and illustrate it with several examples. 

The \Compname{} precisely captures the intuition that trajectories obtained by running any $\pi \in \Pi$ have ``low complexity.'' We first define a notion of reachability: in deterministic MDP $M \in \cMdet$, we say $(s,a)$ is \emph{reachable} by $\pi \in \Pi$  if $(s,a)$ lies on the trajectory obtained by running $\pi$ on $M$. Roughly speaking, the \Compname{} measures ``complexity'' of $\Pi$ as the maximum number of state-action pairs which are reachable by some $\pi \in \Pi$ in any \emph{deterministic} MDP. 

\begin{definition}[\Compname{}]\label{def:dimension}
Fix a deterministic MDP $M \in \cMdet$. We define the \emph{cumulative reachability} at layer $h\in[H]$, denoted by $C^\mathsf{reach}_h(\Pi; M)$, as
\begin{align*}
    C^\mathsf{reach}_h(\Pi; M) \coloneqq \abs{ \crl{(s,a) : (s,a)~\text{is  reachable by } \Pi ~\text{at layer \(h\)} } }.
\end{align*}
We define the \emph{\Compname{}} of $\Pi$ as 
\begin{align*}
    \dimRL(\Pi) \coloneqq \max_{h \in [H]} \max_{M \in \cMdet} C^\mathsf{reach}_h(\Pi; M).
\end{align*} 
\end{definition} 

To build intuition, we first look at some simple examples with small \Compname{}: 
\neurIPS{\vspace{-\topsep}
\begin{enumerate}[label=\(\bullet\), leftmargin=5mm, itemsep=0.2mm]}
\arxiv{\begin{enumerate}[label=\(\bullet\)]}
\item \textbf{Contextual Bandits:} Consider the standard formulation of contextual bandits (i.e., RL with $H = 1$). For any policy class \(\Picb\), since $H=1$, the largest deterministic MDP we can construct has a single state $s_1$ and at most $A$ actions available on $s_1$, so $\dimRL(\Picb) \le A$.
     % \item H-layer Contextual Bandit, (\(\PicbH\)):  
    \item \textbf{Tabular MDPs:} Consider tabular RL with the policy class \(\Pitab = \cA^\cS\) consisting of all deterministic policies on the underlying state space. Depending on the relationship between $S,A$ and $H$, we have two possible bounds on $\dimRL(\Pitab) \le \min\crl{A^H, SA}$. If the state space is exponentially large in $H$, then it is possible to construct a full $A$-ary ``tree'' such that every $(s,a)$ pair at layer $H$ is visited, giving us the $A^H$ bound. However, if the state space is small, then the number of $(s,a)$ pairs available at any layer $H$ is trivially bounded by $SA$.
    \item \textbf{Bounded Cardinality Policy Classes:} For any policy class \(\Pi\), we always have that $\dimRL(\Pi) \le \abs{\Pi}$, since in any deterministic MDP, in any layer \(h \in [H]\), each $\pi \in \Pi$ can visit at most one new $(s,a)$ pair. Thus, for policy classes  \(\abs{\Pismall}\) with small cardinality 
 (e.g. \(\abs{\Pismall} = O(\poly(H, A))\),  the spanning capacity is also small; Note that in this case, we allow our sample complexity bounds to depend on \(\abs{\Pismall}\). 
    
    \item \textbf{Singletons:} For the singleton class we have $\dimRL(\Pisingleton) = H+1$, since once we fix a deterministic MDP, there are at most $H$ states where we can split from the trajectory taken by the policy which always plays $a=0$, so the maximum number of $(s,a)$ pairs reachable at layer $h \in [H]$ is $h+1$. Observe that in light of \pref{prop:singleton-minimax}, the \Compname{} for $\Pisingleton$ is ``on the right order'' for characterizing the minimax sample complexity for agnostic PAC RL. 
\end{enumerate}\neurIPS{\vspace{-\topsep}}
Before proceeding, we note that for any policy class \(\Pi\), the \Compname{} is always bounded. 
%Based on our discussion, we have shown the following bound.
\begin{proposition}\label{prop:dimension-ub}
For any policy class $\Pi$, we have $\dimRL(\Pi) \le \min\crl{A^H, \abs{\Pi}, SA}$.  
\end{proposition}

\pref{prop:dimension-ub} recovers the worst-case upper and lower bound from \pref{sec:agnostic-pac-rl}. However, for many policy classes, \Compname{} is substantially smaller than upper bound of \pref{prop:dimension-ub}. In addition to the examples we provided above, we list several additional policy classes with small \Compname{}. For these policy classes we set the state/action spaces to be $\cS_h = \crl*{s_{(i,h)} : i \in [K]}$ for all $h\in [H]$ and $\cA = \crl{0,1}$, respectively. All proofs are deferred to \pref{app:examples-policy-classes}.

\neurIPS{\vspace{-\topsep}
\begin{enumerate}[label=\(\bullet\), leftmargin=5mm, itemsep=0.2mm]}
\arxiv{\begin{enumerate}[label=\(\bullet\)]}
    \item \textbf{$\boldsymbol{\ell}$-tons}: A natural generalization of singletons. We define 
$\PiLton \coloneqq \crl{\pi_I : I \subset \cS, \abs{I} \le \ell}$, where the policy \(\pi_I\) is defined s.t.~\(\pi_I(s) = \ind{s \in I}\) for any \(s \in \cS\). Here, $\dimRL(\PiLton) = \Theta(H^\ell)$.
    \item \textbf{1-Active Policies}: We define \(\Pioneactive\) to be the class of policies which can take both possible actions on a single state $s_{(1,h)}$ in each layer $h$, but on other states $s_{(i, h)}$ for $i \ne 1$ must take action 0. Formally, \(\Pioneactive \ldef{} \crl{\pi_b \mid b \in \crl{0, 1}^H}\), where for any \(b \in \crl{0, 1}^H\) the policy \(\pi_b\) is defined such that  \( \pi_b(s) = b[h] \) if \(s = s_{(1, h)} \), and $\pi_b(s) = 0$ otherwise. 
    \item \textbf{All-Active Policies}: We define $\PiJactive \coloneqq \crl{\pi_b \mid b \in \crl{0, 1}^H}$, where for any \(b \in \crl{0, 1}^H\) the policy \(\pi_b\) is defined such that  \( \pi_b(s) = b[h] \) if \(s = s_{(j, h)} \), and $\pi_b(s) = 0$ otherwise. We let $\Piactive \coloneqq \bigcup_{j = 1}^K \PiJactive$. Here, $\dimRL(\Piactive) = \Theta(H^2)$. 
\end{enumerate}\neurIPS{\vspace{-\topsep}}
%For the proofs of these bounds on the \Compname{} and further discussion of these policy classes, we refer the reader to \pref{app:examples-policy-classes}.
A natural interpretation of the \Compname{} is that it represents the largest ``needle in a haystack'' that can be embedded in a deterministic MDP using the policy class $\Pi$. To see this, let $(M^\star, h^\star)$ be the MDP and layer which witnesses $\dimRL(\Pi)$, and let $\crl{(s_i,a_i)}_{i=1}^{\dimRL(\Pi)}$ be the set of  state-action pairs reachable by \(\Pi\) in $M^\star$ at layer $h^\star$. Then one can hide a reward of 1 on one of these state-action pairs; since every trajectory visits a single $(s_i, a_i)$ at layer $h^\star$, we need at least $\dimRL(\Pi)$ samples in order to discover which state-action pair has the hidden reward. Note that in this agnostic learning setup, we only need to care about the states that are reachable using \(\Pi\), even though the \(h^\star\) layer may have other non-reachable states and actions. 

%\sasha{somehow this discussion needs to highlight that we are competing with $\Pi$, and this means that it's sufficient to focus only on these states at layer $h$? not sure what can be said precisely about this }

% To summarize, the \Compname{} recovers existing worst-case guarantees and behaves nicely for several classes where we know how to do sample efficient learning. Thus, it is a natural candidate for measuring the minimax sample complexity of agnostic PAC RL.

\subsection{Connection to Coverability} The \Compname{} has another interpretation as the worst-case \emph{coverability}, a structural parameter defined in a recent work by \cite{xie2022role}. 

% We state the definition of coverability below. 

\begin{definition}[Coverability, \citet{xie2022role}]\label{def:coverability}
For any MDP $M$ and policy class $\Pi$, the coverability coefficient $C^\mathsf{cov}$ is denoted
\neurIPS{\begin{align*}
   C^\mathsf{cov}(\Pi; M) \coloneqq~ \inf_{\mu_1, \dots \mu_H \in \Delta(\cS \times \cA)} \sup_{\pi \in \Pi, h \in [H]} ~ \bigg\lVert \tfrac{d^\pi_h}{\mu_h} \bigg\rVert_\infty =~ \max_{h \in [H]}~ \sum_{s, a} \sup_{\pi \in \Pi} d^\pi_h(s,a).
\end{align*}}
\arxiv{\begin{align*}
   C^\mathsf{cov}(\Pi; M) \coloneqq~ \inf_{\mu_1, \dots \mu_H \in \Delta(\cS \times \cA)} \sup_{\pi \in \Pi, h \in [H]} ~ \bigg\lVert \frac{d^\pi_h}{\mu_h} \bigg\rVert_\infty =~ \max_{h \in [H]}~ \sum_{s, a} \sup_{\pi \in \Pi} d^\pi_h(s,a). \numberthis \label{eq:cov_defn}
\end{align*}} 
\end{definition}  

The last equality is shown in Lemma 3 of \cite{xie2022role}, and it says that the coverability coefficient is equivalent to a notion of cumulative reachability (one can check that their definition coincides with ours for deterministic MDPs). %, their definition ).} \gene{I uncommented this out, since a reviewer asked about it. Ayush please check.} 

Coverage conditions date back to the analysis of the classic Fitted Q-Iteration (FQI) algorithm \citep{munos2007performance, munos2008finite}, and have extensively been studied in offline RL. Various models like tabular MDPs, linear MDPs, low-rank MDPs, and exogenous MDPs satisfy the above coverage condition \citep{antos2008learning, chen2019information, jin2021pessimism, rashidinejad2021bridging, zhan2022offline, xie2022role}, and recently, \citeauthor{xie2022role} showed that \emph{coverability} can be used to prove regret guarantees for online RL, albeit under the additional assumption of value function realizability. % and Bellman completeness. %\footnote{Bellman completeness is the assumption that the value function class is closed under Bellman backups; it is stronger than assuming realizability of the value function class.} % In some sense, the notion of coverability captures the complexity of exploration in an unknown MDP $M$.
 
% Coverage conditions have extensively been studied in offline RL---a setting where the learner has access to logged transitions and rewards instead of the ability to interact with the MDP---for the past two decades, dating back to the classic Fitted Q-Iteration (FQI) algorithm \citep{munos2007performance, munos2008finite}. Offline RL under coverage conditions has been extensively studied for model-based and model-free settings, and various models like tabular MDPs, linear MDPs, low-rank MDPs, and exogenous MDPs satisfy coverage conditions \citep{antos2008learning, chen2019information, jin2021pessimism, rashidinejad2021bridging, zhan2022offline, xie2022role}. Recently, \citeauthor{xie2022role} show that \emph{coverability} (merely the \emph{existence} of an offline distribution $\mu \in \Delta(\cS \times \cA)$ with good coverage) in conjunction with Bellman completeness of the value function class can be used to prove regret guarantees for online RL.\footnote{Bellman completeness is the assumption that the value function class is closed under Bellman backups; it is stronger than assuming realizability of the value function class.} In some sense, the notion of coverability captures the complexity of exploration in an unknown MDP $M$.

It is straightforward from \pref{def:coverability} that our notion of \Compname{} is worst-case coverability when we maximize over deterministic MDPs, since for any deterministic MDP, $\sup_{\pi \in \Pi} d^\pi_h(s,a) = \ind{(s,a)~\text{is  reachable by } \Pi ~\text{at layer \(h\)}}$. The next lemma shows that our notion of \Compname{} is \emph{exactly} worst-case coverability even when we maximize over the larger class of stochastic MDPs. As a consequence, there always exists a deterministic MDP that witnesses worst-case coverability.

\begin{lemma}\label{lem:coverability} 
For any policy class $\Pi$, we have $\sup_{M \in \cMsto} C^\mathsf{cov}(\Pi; M) =   \dimRL(\Pi)$. 
\end{lemma}

\ascomment{Take another pass here!} 
While \Compname{} is equal to the worst-case coverability, we remark that the two definitions have different origins. The notion of coverability bridges offline and online RL, and was introduced in \cite{xie2022role} to characterize when sample efficient learning is possible in value-based RL, where the learner has access to a realizable value function class. On the other hand, \Compname{} is developed for the much weaker agnostic RL setting, where the learner only has access to a policy class (and does not have access to a realizable value function class). Note that a realizable value function class can be used to construct a policy class that contains the optimal policy, but the converse is not true. Furthermore, note that the above equivalence only holds in a worst-case sense (over MDPs). In fact, as we show in \pref{app:cpi-section-proofs}, coverability alone is not sufficient for sample efficient agnostic PAC RL in the online interaction model.

\section{Generative Model: 
Spanning Capacity is Necessary and Sufficient} \label{sec:generative} 

In this section, we show that for any policy class, the \Compname{} characterizes the minimax sample complexity for agnostic PAC RL under generative model. 

\begin{theorem}[Upper Bound for Generative Model] 
\label{thm:generative_upper_bound}
For any $\Pi$, the minimax sample complexity $(\eps,\delta)$-PAC learning $\Pi$ is at most
$ \ngen(\Pi; \eps, \delta) \le  \cO \prn*{\tfrac{H \cdot \dimRL(\Pi)}{\eps^2} \cdot \log \tfrac{\abs{\Pi}}{\delta}  }.$
% \begin{align*}
%     \ngen(\Pi; \eps, \delta) \le  \cO \prn*{\tfrac{H \cdot \dimRL(\Pi)}{\eps^2} \cdot \log \tfrac{\abs{\Pi}}{\delta}  }.
% \end{align*}
\end{theorem} 
The proof can be found in \pref{app:proof_generative_upper_bound}, and is a straightforward modification of the classic \emph{trajectory tree method} from \cite{kearns1999approximate}: using generative access, sample $\cO(\log \abs{\Pi}/\eps^2)$ deterministic trajectory trees from the MDP to get unbiased evaluations for every $\pi \in \Pi$; the number of generative queries made is bounded since the size of the maximum deterministic tree is at most $H \cdot \dimRL(\Pi)$. %, we have a bound on the number of queries used.

\begin{theorem}[Lower Bound for Generative Model] \label{thm:generative_lower_bound}
For any $\Pi$, the minimax sample complexity $(\eps,\delta)$-PAC learning $\Pi$ is at least
$ \ngen(\Pi; \eps, \delta) \ge  \Omega \prn*{\tfrac{\dimRL(\Pi)}{\eps^2} \cdot \log \tfrac{1}{\delta}  }.$
\end{theorem}

The proof can be found in \pref{app:proof_generative_lower_bound}. Intuitively, given an MDP $M^\star$ which witnesses $\dimRL(\Pi)$, one can embed a bandit instance on the relevant $(s,a)$ pairs spanned by \(\Pi\) in $M^\star$. The lower bound follows by a reduction to the lower bound for $(\eps, \delta)$-PAC learning in multi-armed bandits.

Together, \pref{thm:generative_upper_bound} and \pref{thm:generative_lower_bound} paint a relatively complete picture for the minimax sample complexity of learning any policy class $\Pi$, in the generative model, up to an $H\cdot \log \abs{\Pi}$ factor. 

\paragraph{Deterministic MDPs.} A similar guarantee holds for online RL over deterministic MDPs. %with deterministic transitions $\cMdetP$.

\begin{corollary}\label{corr:deterministic-mdp} 
Over the class of MDPs with deterministic transitions, the minimax sample complexity of $(\eps,\delta)$-PAC learning any $\Pi$ is
\neurIPS{\begin{align*}
  \Omega \prn*{\tfrac{\dimRL(\Pi)}{\eps^2} \cdot \log \tfrac{1}{\delta}  } \le \non(\Pi; \eps, \delta) \le \cO \prn*{\tfrac{H \cdot \dimRL(\Pi)}{\eps^2} \cdot \log \tfrac{\abs{\Pi}}{\delta} }.
\end{align*}}
\arxiv{\begin{align*}
  \Omega \prn*{\frac{\dimRL(\Pi)}{\eps^2} \cdot \log \frac{1}{\delta}  } \le \non(\Pi, \cMdetP; \eps, \delta) \le \cO \prn*{\frac{H \cdot \dimRL(\Pi)}{\eps^2} \cdot \log \frac{\abs{\Pi}}{\delta} }.
\end{align*}}
\end{corollary}
The upper bound follows because the trajectory tree algorithm for deterministic transitions samples the same tree over and over again (with different stochastic rewards). The lower bound trivially extends because the lower bound of \pref{thm:generative_lower_bound} actually uses an MDP $M \in \cMdetP$ (in fact, the transitions of \(M\) are also known to the learner).

\section{Online RL: Spanning Capacity is Not Sufficient}\label{sec:online-lower-bound} 
Given that fact that \Compname{} characterizes the minimax sample complexity of agnostic PAC RL in the generative model, one might be tempted to 
% In light of \pref{thm:generative_upper_bound} and \pref{thm:generative_lower_bound}, one might also 
conjecture that \Compname{} is also the right characterization in online RL. The lower bound is clear since online RL is at least as hard as learning with a generative model, so \pref{thm:generative_lower_bound} already shows that \Compname{} is \emph{necessary}. But is it also sufficient?

In this section, we prove a surprising negative result showing that bounded \Compname{} by itself is not sufficient to characterize the minimax sample complexity in online RL. In particular, we provide an example for which we have a \emph{superpolynomial} (in \(H\)) lower bound on the number of trajectories needed for learning in online RL, that is not captured by any polynomial function of \Compname{}. This implies that, contrary to RL with a generative model, one can not hope for \(\non(\Pi; \eps, \delta) = \wt{\Theta}\prn{\poly\prn{\dimRL(\Pi), H, \eps^{-1}, \log \delta^{-1}}}\) in online RL. 

\neurIPS{\footnote{In the lower bound construction, the optimal policy \(\pistar\) for the underlying MDP belongs to the set \(\Pi\). This shows that realizability of the optimal policy in the policy class also does not help.}} 

%stronger lower bound on the samples required in the online RL setting. Our main result is \pref{thm:lower-bound-online}, which demonstrates a \emph{superpolynomial} lower bound on the sample complexity required for a specific policy class $\Pi$ with bounded \Compname{}.
\begin{theorem}[Lower Bound for Online RL] 
\label{thm:lower-bound-online} 
Fix any sufficiently large $H$. Let $\eps \in \prn{1/2^{\cO(H)},\cO(1/H)}$ and $\ell \in \crl{2, \dots, H}$ such that $1/\eps^\ell \le 2^H$. There exists a policy class $\Piell$ of size $\cO(1/\eps^\ell)$ with $\dimRL(\Piell) \le \cO(H^{4\ell+2})$ and a family of MDPs $\cM$ with state space $\cS$ of size $2^{\cO(H)}$, binary action space, and horizon $H$ such that: for any $(\eps, 1/8)$-PAC algorithm, %that returns an $\eps/16$-optimal policy with probability at least $1/8$, 
there exists an MDP $M \in \cM$ for which the algorithm must collect at least 
\neurIPS{$\Omega\prn{\min \crl{\tfrac{1}{\eps^\ell}, 2^{H/3} } }$ online trajectories in expectation.}
\arxiv{\begin{align*}
    \Omega\prn*{\min \crl*{\frac{1}{\eps^\ell}, 2^{H/3} } } \quad \text{online trajectories in expectation.}
\end{align*}}
\end{theorem} 
Informally speaking, the lower bound shows that there exists a policy class \(\Pi\) for which \(\non(\Pi; \eps, \delta) = 1/\eps^{\Omega(\log_H \dimRL(\Pi))} \). In order to interpret this theorem, we can instantiate choices of $\epsilon = 1/2^{\sqrt{H}}$ and $\ell = \sqrt{H}$ to show an explicit separation.
\begin{corollary} For any sufficiently large $H$, there exists a policy class $\Pi$ with $\dimRL(\Pi) = 2^{\cO(\sqrt{H} \log H)}$ such that for any $(1/2^{\sqrt{H}}, 1/8)$-PAC algorithm, there exists an MDP for which the algorithm must collect at least $2^{\Omega(H)}$ online trajectories in expectation. 
\end{corollary}

%\paragraph{Discussion.} Before describing the proof of \pref{thm:lower-bound-online}, a few remarks are in order. 

%\ayush{Come back here later!} The lower bound shows that the number of samples required to learn the class $\Piell$ is superpolynomial in $\dimRL(\Pi)$, as the sample complexity scales as $1/\eps^{\Omega(\log_H \dimRL(\Pi))}$. Since the policy class and state space are moderately sized with $\log \abs{\Pi} \le \cO(\log H \log(1/\eps))$ and $\log S \le \cO(H)$, this rules out a $\poly(\dimRL(\Pi), H, \log \abs{\Pi}, \log S, \eps^{-1}, \log \delta^{-1})$ upper bound in general. Notice that there is no contradiction to the (weaker) $\cO(\min\crl{2^H \log \abs{\Pi}, \abs{\Pi}, HS}/\eps^2)$ upper bound. 

In conjunction with the results of \pref{sec:generative}, \pref{thm:lower-bound-online} shows that (1) online RL is \emph{strictly harder} than RL with generative access, and (2) online RL for stochastic MDPs is \emph{strictly harder} than online RL for MDPs with deterministic transitions. \neurIPS{We defer the proof of \pref{thm:lower-bound-online} to \pref{app:proof-lower-bound-online}. Our lower bound introduces several technical novelties: the family $\cM$ utilizes a \emph{contextual} variant of the combination lock, and the policy class $\Pi$ is constructed via a careful probabilistic argument such that it is hard to explore despite having small \Compname{}.}

\arxiv{
\begin{figure}[!t]
    \centering
\includegraphics[scale=0.25, trim={2cm 1.5cm 5cm 0.5cm}, clip]{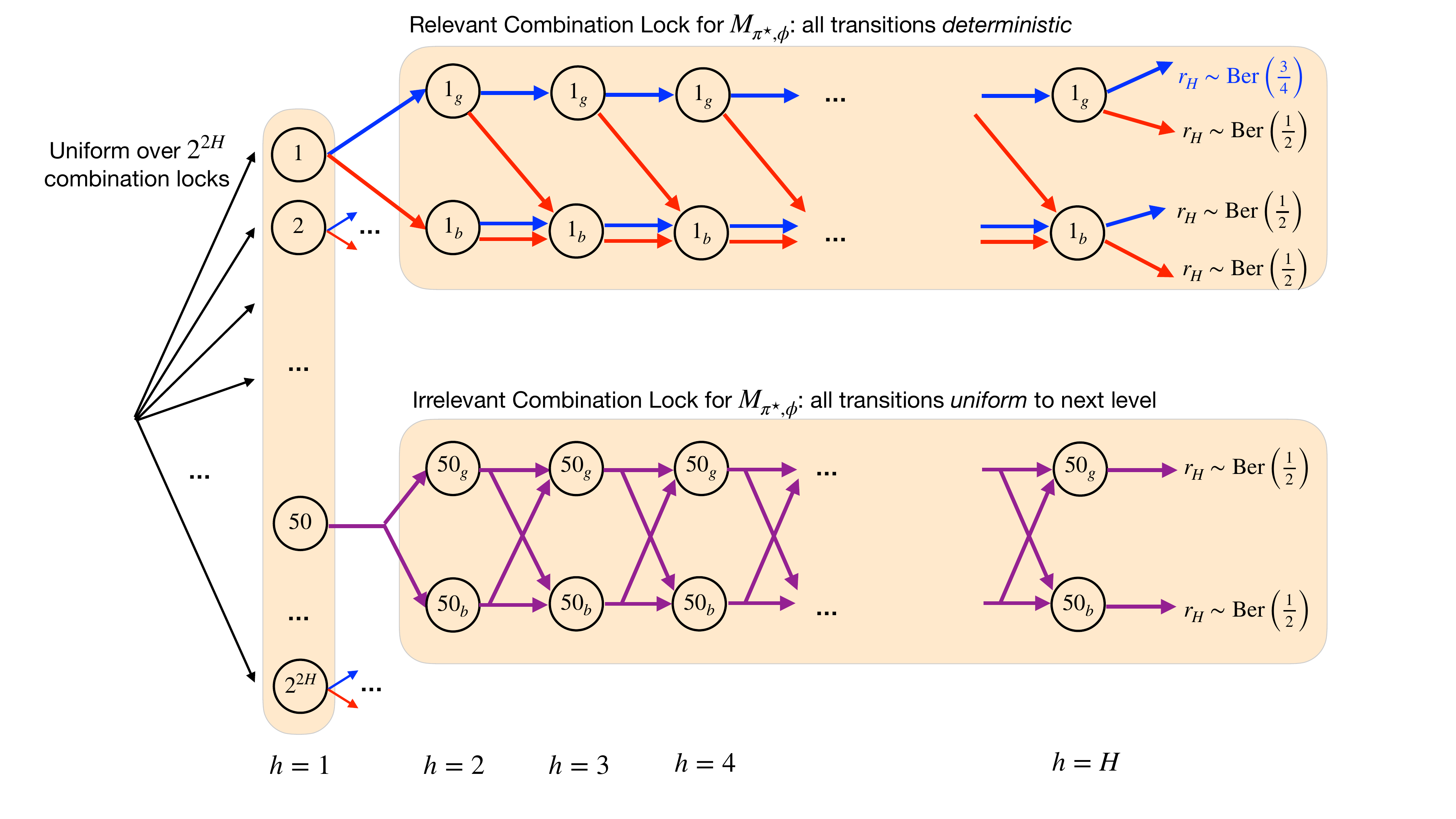}
    \caption{Illustration of the lower bound from \pref{thm:lower-bound-online}. \textcolor{blue}{Blue} arrows represent taking the action $\pistar(s)$, while \textcolor{red}{red} arrows represent taking the action $1- \pistar(s)$. \textcolor{purple}{Purple} arrows denote uniform transition to the states in the next layer, regardless of action. The MDP $M_{\pistar, \phi}$ is a uniform distribution of $2^{2H}$ combination locks of two types. In the \emph{relevant} combination locks (such as Lock 1 in the figure), following $\pistar$ keeps one in the ``good'' chain and gives reward of $\mathrm{Ber}(3/4)$ in the last layer, while deviating from $\pistar$ leads one to the ``bad'' chain and gives reward of $\mathrm{Ber}(1/2)$. In \emph{irrelevant} combination locks (such as Lock 50 in the figure), the next state is uniform regardless of action, and all rewards at the last layer are $\mathrm{Ber}(1/2)$.} 
    \label{fig:lower-bound-idea}
\end{figure}

\paragraph{Proof Sketch for \pref{thm:lower-bound-online}.} We defer the full proof of \pref{thm:lower-bound-online} to \pref{app:proof-lower-bound-online} and sketch the main ideas here.  An illustration of an MDP in the family $\cM$ can be found in \pref{fig:lower-bound-idea}. 

The basic building block for our lower bound is the combination lock, a prototypical construction used in prior works \citep{krishnamurthy2016pac,du2019provably, sekhari2021agnostic}. The MDPs we construct are essentially a uniform distribution over $2^{2H}$  different combination locks. In order to receive positive feedback, the learner must play a sequence of $H$ correct actions in a particular combination lock, figuring which out intuitively requires many revisits to the same lock. However, under online access, it is  unlikely that the learner will get to see the same lock multiple times unless they use an exponential number of samples. Note that under generative access, this is not an issue, since the learner can ``reset'' to any state they like.  

In more detail, each hard MDP $M_{\pistar, \phi} \in \cM$ is parameterized by a policy $\pistar \in \Piell$ (which is optimal for that MDP) and a decoder $\phi: \cS \mapsto \crl{\textsc{Good}, \textsc{Bad}}$. In the MDP $M_{\pistar, \phi}$ there will be a \emph{planted set} of ``relevant'' combination locks on which running $\pistar$ achieves $\ber(3/4)$ reward; on the rest of the combination locks any policy $\pi \in \cA^\cS$ achieves $\ber(1/2)$ reward. Since the planted set is an $\eps$-fraction of the total, the learner must solve an \(\Omega(\epsilon)\)-fraction of the relevant locks in order to find an $O(\eps)$-optimal policy. However, as we have established, the learner will never see the same combination lock multiple times, so their only hope is to try to identify $\pistar$ through alternative means (e.g.~via elimination). 

In the vanilla combination lock, it becomes easy to learn $\pistar$ via trajectory data, since once the learner observes a jump to the ``bad'' chain, they can immediately eliminate many candidate policies. Our construction utilizes a \emph{contextual} variant of the combination lock which minimizes information leakage about $\pistar$ from transition data. This is formalized by the decoder $\phi$, which randomly assigns states to be in the ``good'' chain and the ``bad'' chain. In this way, the learner, upon a single visit to a certain lock, cannot know when they have jumped to the ``bad'' chain (or even whether they are in a planted combination lock or not!) unless they can identify whether the reward at level $H$ in that combination lock is $\ber(1/2)$ instead of $\ber(3/4)$.

The last key to the puzzle is to prove that there exists such a policy class $\Piell$ which satisfies these properties yet still has bounded \Compname{}. We reduce this problem to showing the existence of certain ``block-free'' binary matrices, whose existence is shown using the probabilistic method.} Explicit construction of such ``block-free" binary matrices is left open as a direction for future research.
\section{Statistically Efficient  Agnostic Learning in Online RL} \label{sec:upper-bound} 

The lower bound in \pref{thm:lower-bound-online} suggests that further structural assumptions on \(\Pi\)  are needed for statistically efficient agnostic RL under the online interaction model. Essentially, the lower bound example provided in \pref{thm:lower-bound-online} is hard to agnostically learn because any two distinct policies \(\pi, \pi' \in \Pi\) can differ substantially on a large subset of states (of size at least \(\epsilon \cdot 2^{2H}\)). Thus, we cannot hope to learn ``in parallel'' via a low variance IS strategy that utilizes extrapolation to evaluate all policies $\pi \in \Pi$, as we did for singletons. 

In the sequel, we consider the following \coreset{} property to rule out such problematic scenarios, and show how bounded \Compname{} along with the \coreset{} property enable sample-efficient agnostic RL in the online interaction model. The \coreset{} property only depends on the state space, action space, and policy class, and is independent of the transition dynamics and rewards of the underlying MDP. We first define a petal, a key ingredient of a sunflower.  

\begin{definition}[Petal] 
\label{def:petal_policy}
For a policy set \(\bar{\Pi}\), and states \(\bar{\cS} \subseteq \cS\), a policy \(\pi\) is said to be a \(\bar{\cS}\)-\textit{petal} on \(\bar{\Pi}\) if for all \(h \leq h' \leq H\), and partial trajectories $\tau = (s_h, a_h, \cdots, s_{h'}, a_{h'})$ that are consistent with $\pi$: either \(\tau\) is also consistent with some \(\pi' \in \bar{\Pi}\), or there exists \(i \in (h, h']\) s.t.~$s_i\in \bar{\mS}$. 
\end{definition}  

Informally, \(\pi\) is a \(\bar{\cS}\)-petal on \(\bar{\Pi}\) if any trajectory that can be obtained using \(\pi\) can either also be obtained using a policy in \(\bar{\Pi}\)  or must pass through \(\bar{\cS}\). Thus, any policy that is a \(\bar{\cS}\)-petal on \(\bar{\Pi}\) can only differentiate from \(\bar{\Pi}\) in a structured way. A policy class is said to be a sunflower if it is a union of petals as defined below: 

\begin{definition}[Sunflower]  
\label{def:core_policy} 
    A policy class $\Pi$ is said to be a \((K, D)\)-\coreset{} if there exists a set \(\Picore\) of Markovian policies with $|\Picore|\le K$ such that for every policy $\pi\in \Pi$ there exists a set $\mS_\pi \subseteq \mS$, of size at most \(D\), so that \(\pi\) is an \(S_\pi\)-petal on \(\Picore\). 
\end{definition}

% \begin{remark}
% 	Notice that $(K, D)$-\coreset{} is only a property of state space, action space and policy class, and is independent to the underlying transitions and rewards. Hence whether $\Pi$ is a $(K, D)$-\coreset{} can be verified and the set $\Picore$ and $\mS_\pi$ can be computed before start of the algorithm.
% \end{remark}

% Note that the above \coreset{} property only depends on the state space, action space, and policy class, and is independent of the transition dynamics and rewards of the underlying MDP. In fact, one can enumerate and verify whether $\Pi$ is a $(K, D)$-\coreset{}, and compute the sets $\Picore$ and $\crl{\cS_\pi}_{\pi \in \Pi}$ without any \ayush{replace \(\Picore\) by \(\Pi_{\mathrm{core}}\)} interaction with the MDP. 

Our next theorem provides a sample complexity bound for Agnostic PAC RL for~policy classes that have \((K, D)\)-sunflower structure. This bound is obtained via a new exploration algorithm called $\mathsf{POPLER}$ that takes as input the set \(\Picore\) and corresponding petals \(\crl{\cS_\pi}_{\pi \in \Pi}\) and leverages importance sampling as well as reachable state identification techniques to simultaneously estimate the value of every policy in \(\Pi\). Algorithm details are deferred to \pref{sec:algorithm_description}.

\begin{theorem} \label{thm:sunflower} 
Let \(\epsilon, \delta > 0\). Suppose the policy class \(\Pi\) satisfies \pref{def:dimension} with \Compname{} \(\dimRL(\Pi)\), and is a \((K, D)\)-\sunflower. Then, for any MDP \(M\), with probability at least \(1 - \delta\), $\mathsf{POPLER}$ (\pref{alg:main}) succeeds in returning a policy \(\wh \pi\) that satisfies  \(V^{\wh \pi} \geq \max_{\pi \in \Pi} V^\pi - \epsilon\), after collecting 
\neurIPS{\begin{align*}
\widetilde{\cO}\prn*{\prn*{\tfrac{1}{\epsilon^2} + \tfrac{HD^6 \dimRL(\Pi)}{\epsilon^4}} \cdot K^2 \log\tfrac{|\Pi|}{\delta}} \quad \text{online trajectories in \(M\).} 
\end{align*}}
\arxiv{\begin{align*}
\widetilde{\cO}\prn*{\prn*{\frac{1}{\epsilon^2} + \frac{HD^6 \dimRL(\Pi)}{\epsilon^4}} \cdot K^2 \log\frac{|\Pi|}{\delta}} \quad \text{online trajectories in \(M\).} 
\end{align*}}
\end{theorem}
 The proof of \pref{thm:sunflower}, and the corresponding hyperparameters in $\mathsf{POPLER}$ needed to obtain the above bound, can be found in \pref{app:upper_bound_main}. Before diving into the algorithm and proof details, let us highlight several key aspects of the above sample complexity bound: 
\begin{itemize}[label=$\bullet$]
    \item Note that a class \(\Pi\) may be a \((K, D)\)-\coreset{} for many different choices of \(K\) and \(D\). Barring computational issues, one can enumerate over all choices of $\Picore$ and $\crl{\cS_\pi}_{\pi \in \Pi}$, and check if $\Pi$ is a \((K, D)\)-\coreset{} for that choice of $K = \abs{\Picore}$ and $D = \max_{\pi \in \Pi} \abs{\cS_\pi}$. Since our bound in \pref{thm:sunflower} scales with \(K\) and \(D\), we are free to choose \(K\) and \(D\) to minimize the corresponding sample complexity bound. 
    \item In order to get a polynomial sample complexity in \pref{thm:sunflower}, both \(\dimRL(\Pi)\) and \((K, D)\) are required to be \(\poly(H, \log\abs{\Pi})\). All of the policy classes considered in \pref{sec:C_pi} have the sunflower property, with both \(K, D = \poly(H)\), and thus our sample complexity bound extends for all these classes. See \pref{app:examples-policy-classes} for details. 
    % ; moreover for various examples \ayush{List the examples here which satisfy this property!} we have $K = \poly(H)$ and $D=0$, so we also obtain the optimal \(\widetilde{\cO}({\nicefrac{1}{\epsilon^2}})\) dependence on \(\epsilon\). See \pref{app:examples-policy-classes} for details.
    % \gene{I deleted this sentence, since actually we only have a few policy classes with $D=0$ (at least the way we prove it in the appendix).}
    \item Notice that for \pref{thm:sunflower} to hold, we need both bounded \Compname{} as well as the sunflower structure on the policy class with bounded \((K, D)\). Thus, one may wonder if we can obtain a similar polynomial sample complexity guarantee in online RL  under weaker assumptions. In \pref{thm:lower-bound-online}, we already showed that bounded \(\dimRL(\Pi)\) alone is not sufficient to obtain polynomial sample complexity in online RL. Likewise, as we show in \pref{app:upper_bound_main}, sunflower property with bounded \((K, D)\) alone is also not sufficient for polynomial sample complexity, and hence both assumptions cannot be individually removed. However, it is an interesting question if there is some other structural assumption that combines both spanning capacity and the sunflower property, and is both sufficient and necessary for agnostic PAC learning in online RL. See \pref{sec:conclusion} for further discussions on this. 
\end{itemize} 

\paragraph{Why Does the Sunflower Property Enable Sample-Efficient Learning?}
Intuitively, the \coreset{} property captures the intuition of simultaneous estimation of all policies \(\pi \in \Pi\) via Importance Sampling (IS), and allows control of both bias and variance. Let \(\pi\) be a \(\cS_\pi\)-petal on \(\Picore\). Any trajectory $\tau \cons \pi$ that avoids \(\cS_\pi\) must be consistent with some policy in \(\Picore\), and will thus be covered by the data collected using \(\pi' \sim \unif(\Picore)\). Thus, using IS with variance scaling with \(K\), one can create a biased estimator for \(V^\pi\), where the bias is \emph{only due} to trajectories that pass through \(\cS_\pi\). There are two cases: if every state in \(\cS_\pi\) has small reachability under \(\pi\) i.e.~$d^\pi(s) \ll \eps$ for every \(s \in \cS_\pi\), then the IS estimate will have a low bias (linear in $\abs{\cS_\pi}$) and thus we can compute \(V^\pi\) up to error at most \(\epsilon \abs{\cS_\pi}\). On the other hand, if $d^\pi(s)$ is large for some $s \in \cS_\pi$, it is possible to explicitly control the bias that arises from trajectories passing through them since there are at most $D$ of them.

\subsection{Algorithm and Proof Ideas}  \label{sec:algorithm_description}

$\mathsf{POPLER}$, described in \pref{alg:main}, takes as input a policy class $\Pi$, as well as sets \(\Picore\) and \(\crl{\cS_{\pi}}_{\pi \in \Pi}\), which can be computed beforehand by enumeration. $\mathsf{POPLER}$ has two phases: a \emph{state identification phase}, where it finds ``petal'' states $s \in \bigcup_{\pi \in \Pi} \cS_\pi$ that are reachable with sufficiently large probability; and an \emph{evaluation phase} where it computes estimates $\wh{V}^\pi$ for every $\pi \in \Pi$. It uses three subroutines $\datacollector$, $\estreach$, and $\evaluate$, whose pseudocodes are stated in \pref{app:algorithm_details}.

The structure of the algorithm is reminiscent of reward-free exploration algorithms in tabular RL e.g.~\citet{jin2020reward}, where we first identify (petal) states that are reachable with probability at least \(\Omega(\epsilon/D)\) and build a policy cover for these states, and then use dynamic programming to estimate the values. However, contrary to the classical tabular RL setting, because the state space can be large, our setting is much more challenging and necessitates technical innovations. In particular, we can no longer enumerate over all petal states and check if they are sufficiently reachable by some policy $\pi \in \Pi$ (since the total number of petal states \(\sum_{\pi \in \Pi} \abs{\cS_\pi}\) could scale linearly in \(\abs{\Pi}\), a factor that we do not want to appear in our sample complexity). Instead, the key observation that we rely on is that if \Compname{} is bounded, then by the equivalence of \Compname{} and worst-case coverability (\pref{lem:coverability}) and due to \pref{eq:cov_defn}, the number of highly reachable (and thus relevant) petal states is also bounded. Thus, we only need to build a policy cover to reach these relevant petal states. Our algorithm does this in a sample-efficient \emph{sequential} manner. For both state identification as well as evaluation, we interleave importance sampling estimates with the construction of a \emph{policy-specific} Markov Reward Processes (MRPs), which are defined for every $\pi \in \Pi$. The challenge is doing all of this ``in parallel'' for every $\pi \in \Pi$ through extensive sample reuse to avoid a blowup of $\abs{\Pi}$ or $S$ in the sample complexity. 

\paragraph{Key Tool: Policy-Specific Markov Reward Process.} We elaborate on the MRP construction, which is the key technical tool used in both the identification and evaluation phases of the algorithm. To build intuition, let us consider a fixed policy $\pi \in \Pi$ with petal states $\cS_\pi$ and define a population version of policy-specific MRP. In particular, let \(\cS_\pi^+ = \cS \cup \crl{s_\bot, s_\top}\), and associated with $\pi$ define $\MRPsign^\pi = \mathrm{MRP}(\cS^+_{\pi}, P^\pi, R^\pi, H, s_\top, s_\bot)$  which essentially compresses the transition and reward information in the original MDP relevant to the policy \(\pi\). % In order to define \(P^\pi\) and \(R^\pi\), we first define additional notation. Let $s, s' \in \cS_\pi$ denote two petal states which reside in different layers $h < h'$ in the underlying MDP. We also use $\tau_{h:h'}$ to denote a partial trajectory from layer $h$ to $h'$, and $R(\tau_{h:h'})$ to denote the cumulative rewards from layer $h$ to $h'$ along the partial trajectory $\tau_{h:h'}$.  
For any states \(s \in \cS_\pi \cup \crl{s_\top}\) and $s' \in \cS_\pi \cup \crl{s_\bot}$ residing in different layers $h < h'$ in the underlying MDP \footnote{For the simplicity of analysis, we slightly abuse the notation and assume that all trajectories in the underlying MDP start at \(s_\top\) at time step \(0\) and terminate at \(s_\bot\) at time step \(H+1\), and do not observe \(s_\bot\) and \(s_\top\) in between from time steps \(h = 1, \dots, H\). However, recall that \(s_\bot\) and \(s_\top\) are not part of the original layered state space \(\cS\) for the MDP. }, we define: 
\begin{itemize}[label=$\bullet$]
    \item \textbf{Transition $\bm{P_{s\to s'}^\pi}$} as:
    \begin{align*} 
        P_{s\to s'}^\pi \coloneqq \bbP^\pi \brk*{\substack{ \text{$\tau_{h:h'}$ goes from \(s\) to \(s'\)} \\ \text{without passing through any other $s'' \in \cS_\pi$} } \mid \text{$\tau_h = s$}}.
    \end{align*} 

    \item \textbf{Rewards $\bm{R_{s\to s'}^\pi}$} as:
       \begin{align*}
        R_{s\to s'}^\pi \coloneqq \En^\pi \brk*{ R(\tau_{h:h'})\ind{\substack{ \text{$\tau_{h,h'}$ goes from \(s\) to \(s'\)} \\ \text{without passing through any other $s'' \in \cS_\pi$ } }} \mid \text{$\tau_h = s$}}.
    \end{align*}
\end{itemize}
where $\tau_{h:h'}$ denotes a partial trajectory from layer $h$ to $h'$, and $R(\tau_{h:h'})$ denotes the cumulative rewards from layer $h$ to $h'$ along the partial trajectory $\tau_{h:h'}$.  Furthermore, 
\(P_{s_\bot \to s_\bot}^\pi = 1\) and \(R_{s_\bot \to s_\bot}^\pi = 0\).

%Furthermore, \(P_{s\to s'}^\pi = 0\) and \(R_{s\to s'}^\pi = 0\) whenever \(s' \neq s_\bot\) and \(h \geq h'\). Finally, \(P_{s_\bot \to s_\bot}^\pi = 1\) and \(R_{s_\bot \to s_\bot}^\pi = 0\). % Furthermore, \({P}^\pi_{s \rightarrow s_\bot} = 1\) for all \(s \in \cS_{\pi}\), and \(R_{s\to s'}^\pi = 0\) whenever \(s = s_\top\) or \(s'= s_\bot\).  \ayush{@zeyu, please verify the last line} 

The key technical benefit of using policy-specific MRPs is that the value $V^\pi$ for the policy \(\pi\) in the original MDP is identical to the value of policy-specific MRP \(\MRPsign^\pi\) (starting from \(s_\bot\)). Thus, if one knew the transitions and rewards in $\MRPsign^\pi$, one could calculate the value of the policy \(\pi\) via dynamic programming on $\MRPsign^\pi$. Of course, we do not know these quantities, so we must estimate them by interacting with the original MDP. A naive approach is to simply run $\pi$ many times to get estimates for each transition and reward---but since we want to estimate $V^\pi$ for every $\pi \in \Pi$ simultaneously, this approach would incur an $\abs{\Pi}$ dependency in the sample complexity. Instead, our algorithm uses importance sampling to estimate transitions and rewards in the corresponding $\MRPsign^\pi$ for many policies simultaneously. 

We next describe the key algorithmic ideas, as well as the empirical estimation of policy-specific MRPs, in the two phases of \pref{alg:main}. 

\begin{algorithm}[!t] 
    \caption{{\bf{P}}olicy {\bf{OP}}timization by {\bf{L}}earning $\boldsymbol\eps$-{\bf{R}}eachable States ($\mathsf{POPLER}$)}  \label{alg:main} 
    \begin{algorithmic}[1]  
            \Require Policy class \(\Pi\), Sets \(\Picore\) and \(\crl{\cS_\pi}_{\pi \in \Pi}\),  Parameters \(K, D, n_1, n_2, \epsilon, \delta\). 
           % \State Set $n_1 = \widetilde{\cO}\prn*{\tfrac{(D^4+1)K^2}{\epsilon^2} \log(|\Pi|/\delta)}, n_2 = \widetilde{\cO}\prn*{\tfrac{C_2D^5K^2}{\epsilon^3} \log(|\Pi|/\delta },$ 
%            \State Find $\Picore$, and the set $\mS_\pi$ for every $\pi\in\Pi$, corresponding to the \((K, D)\)-sunflower of \(\Pi\). 
%           \State  
            \State Define an additional start state \(s_\top\) (at \(h = 0\)) and end state \(s_\bot\) (at \(h = H+1\)). 
           \State Initialize $\reachablestates = \crl{s_\top}$, $\cT \leftarrow \crl{(s_\top, \mathrm{Null})}$, and for every \(\pi \in \Pi\), define \(\cS_\pi^+ \ldef{} \cS_\pi \cup \crl{s_\top, s_\bot}\). 
           \label{line:initialization}
%       \State $U, \crl{\cD_s}_{s \in U} \leftarrow \text{Identify}(\{\Picore\}, n_1, n_2)$. \\ 
        \State $\mD_\top \leftarrow \datacollector(s_\top, \mathrm{Null}, \Picore, n_1)$
\\
\algcommentbig{Identification of Petal States that are Reachable with \(\Omega(\epsilon/D)\) Probability}  
             \While{$\mathrm{Terminate}=\mathrm{False}$}  \label{line:while_loop}     \hfill % \algcomment{} 
             \State Set \(\mathrm{Terminate}=\mathrm{True}\).
     \For{$\pi\in \Pi$} \label{line:while_loop_start}            
                \State Compute the set of already explored reachable states \( \SHPalg{\pi} = \cS^+_\pi \cap \reachablestates \), and the remaining states \(\SRemalg{\pi} = \cS_{\pi} \setminus (\SHPalg{\pi}\cup \crl{s_\bot})\). \label{line:S_sets_comp} 
                 \State Estimate the policy-specific MRP $\widehat \MRPsign^\pi_{\reachablestates}$ according to \eqref{eq:emp-transitions} and \eqref{eq:emp-rewards}. \label{line:est-policy-specific-mrp}
            %  \State Estimate transition probability $\widehat{P}^{\pi} = \crl{\widehat{P}^\pi_{s\to s'} \mid s, s' \in \mS^+_\pi}$ using \eqref{eq:empirical}.  \label{line:prob_estimate}
            % \ascomment{Account for 0 transitions!} 
            \For{$\bar{s}\in\SRemalg{\pi}$} \label{line:state_identification} 
                \State Estimate probability of reaching \(\bar{s}\) under \(\pi\) as $\widehat{d}^\pi (\bar{s})\leftarrow \estreach(\mS_\pi^+, \widehat \MRPsign^\pi_{\reachablestates}, \bar{s})$. \label{line:DP_solver_search}  \label{line:DP}
                \If{$\widehat{d}^\pi(\bar{s})\ge \nicefrac{\epsilon}{6D}$}   \label{line:test_and_add}
                    \State Update $\reachablestates \gets \reachablestates \cup \crl{\bar{s}}$, $\cT\leftarrow \cT \cup \crl{(\bar{s}, \pi)}$ and set  \(\mathrm{Terminate}=\mathrm{False}\). 
                    \label{line:add_s} %, and update $\SHP{\pi'}$ for all $\pi'\in\Pi$. 
                    \State Collect dataset $\mD_{\bar{s}}\leftarrow \datacollector(\bar{s}, \pi, \Picore, n_2)$. \label{line:fresh_dataset}
%                   \State Set \(\mathrm{Terminate}=\mathrm{False}\), and repeat from \pref{line:while_loop}. 
                \EndIf 
            \EndFor 
        \EndFor \label{line:while_loop_end}
  \EndWhile
    \\
%       \State break the while loop.
%       \EndWhile \\ 
        \algcommentbig{Policy Evaluation and Optimization} 
        \For{$\pi\in \Pi$} 
            \State $\hV^\pi\leftarrow \evaluate(\Picore, \reachablestates, \{\mD_s\}_{s \in \reachablestates}, \pi)$. \label{line:evaluate}
        \EndFor
        \State \textbf{Return} $\widehat{\pi} \in \argmax_\pi \hV^\pi$. \label{line:return}
    \end{algorithmic}
\end{algorithm} 

\subsubsection*{State Identification Phase} 
The goal of the state identification phase is to discover all such petal states that are reachable with probability  $\Omega(\eps/D)$. The algorithm proceeds in a loop and sequentially grows the set $\cT$, which contains tuples of the form \((s, \pi_s)\), where $s \in \bigcup_{\pi \in \Pi} \cS_\pi$ is a sufficiently reachable petal state (for some policy) and \(\pi_s\) denotes a policy that reaches \(s\) with probability \(\Omega(\eps/D)\). We also denote $\reachablestates \coloneqq \crl*{s: (s, \pi_s) \in \cT}$ to denote the set of reachable states in \(\cT\). Initially, $\cT$ only contains a dummy start state $s_\top$ and a null policy. We will collect data using the $\datacollector$ subroutine that: for a given $(s, \pi_s) \in \cT$, first run $\pi_s$ to reach state $s$, and if we succeed in reaching $s$, restart exploration by sampling a policy from $\unif(\Pi_\mathrm{exp})$. Note that $\datacollector$ will be sample-efficient for any \((s, \pi_s)\) since \(\Omega(\epsilon/D)\) fraction of the trajectories obtained via \(\pi_s\) are guaranteed to reach $s$ (by definition of \(\pi_s\) and construction of set $\cT$). Initially, we run $\datacollector$ using $\unif(\Pi_\mathrm{exp})$ from the start, where we slightly abuse the notation and assume that all trajectories in the MDP start at the dummy state $s_\top$ at time step \(h = 0\). 

In every loop, the algorithm attempts to find a new petal state $\bar{s}$ for some $\pi \in \Pi$ that is guaranteed to be $\Omega(\eps/D)$-reachable by $\pi$. This is accomplished by constructing a (estimated and partial) version of the policy-specific MRP using the datasets collected up until that loop (\pref{line:est-policy-specific-mrp}). In particular given a policy \(\pi\) and a set \( \SHPalg{\pi} = \cS^+_\pi \cap \reachablestates\),  we construct $\widehat \MRPsign^\pi_{\reachablestates} = \mathrm{MRP}(\cS^+_{\pi}, \wh P^\pi, \wh R^\pi, H, s_\top, s_\bot)$ which essentially compresses our empirical knowledge of the original MDP relevant to the policy \(\pi\). In particular, for any states \(s \in \cS_\pi \cup \crl{s_\top}\) and $s' \in \cS_\pi \cup \crl{s_\bot}$ residing in different layers $h < h'$ in the underlying MDP, we define: 
\begin{itemize}[label=$\bullet$]
    \item \textbf{Transition $\bm{\wh P_{s \to s'}^\pi}$} as:
    \begin{align*} 
        \widehat{P}_{s\to s'}^\pi = 
	    \frac{1}{|\mD_s|}\sum_{\tau\in \mD_s}\frac{\ind{\pi\cons\tau_{h:h'}}}{\tfrac{1}{|\Picore|} \sum_{\pi'\in \Picore}\ind{\pi' \cons \tau_{h:h'}}}\ind{ \substack{ \text{$\tau_{h:h'}$ goes from \(s\) to \(s'\)} \\ \text{without passing through any other \(s'' \in \cS_\pi\)}}}.
	\numberthis\label{eq:emp-transitions}
    \end{align*}
    \item \textbf{Transition $\bm{\wh R_{s \to s'}^\pi}$} as: 
    \begin{align*} 
        \widehat{R}_{s\to s'}^\pi = \frac{1}{|\mD_s|}\sum_{\tau\in \mD_s}\frac{R(\tau_{h:h'}) \cdot \ind{\pi\cons\tau_{h:h'}}}{\tfrac{1}{|\Picore|}\sum_{\pi'\in \Picore}\ind{\pi' \cons \tau_{h:h'}}}\ind{ \substack{ \text{$\tau_{h:h'}$ goes from \(s\) to \(s'\)} \\ \text{without passing through any other \(s'' \in \cS_\pi\)}}}. \numberthis\label{eq:emp-rewards}
    \end{align*}   
\end{itemize}

Clearly, the above definition implies that $\widehat{P}_{s\to s'}^\pi = 0$ and  $\widehat{R}_{s\to s'}^\pi = 0$ for any $s \notin \SHPalg{\pi}$ since \(\cD_s\) would be empty corresponding to these unexplored states. Furthermore, %\(P_{s\to s'}^\pi = 0\) and \(R_{s\to s'}^\pi = 0\) whenever \(s' \neq s_\bot\) and \(h \geq h'\), and, 
\(P_{s_\bot \to s_\bot}^\pi = 1\) and \(R_{s_\bot \to s_\bot}^\pi = 0\). 

%In the above, if $s \notin \SHPalg{\pi}$, then $\cD_s$ is empty, and by convention, we will set the transitions $\widehat{P}_{s\to s'}^\pi$ and rewards $\widehat{R}_{s\to s'}^\pi$ to 0. 

Note that since the set $\reachablestates$ is changing in each iteration of the loop as the algorithm collects more data, the policy-specific MRP $\widehat \MRPsign^\pi_{\reachablestates}$ also changes in every iteration of the loop --- in particular, more and more transitions/rewards are assigned nonzero values due to new states being added to $\reachablestates$. More details on policy-specific MRPs is given in \pref{app:algorithm_details}. 

The key advantage of constructing the empirical versions of policy-specific MRPs is that they allow us to explore and find new leaf states in \(\cS_\pi\) which are reachable with probability at least \(\Omega(\nicefrac{\epsilon}{6D})\). In particular, using standard dynamic programming (subroutine $\estreach$), we can check whether a candidate petal $\bar{s}$ is reachable with decent probability by $\pi$ (lines \ref{line:DP}-\ref{line:test_and_add}); If it is, then we add $(\bar{s}, \pi)$ to the set $\cT$ and collect a fresh dataset using $\datacollector$ (lines \ref{line:add_s}-\ref{line:fresh_dataset}). Crucially, the importance sampling technique enables us to be sample efficient, since the same dataset $\cD_s$ can be used to evaluate transitions/rewards in Eqs.~\eqref{eq:emp-transitions} and \eqref{eq:emp-rewards} for multiple $\pi \in \Pi$ for which $s$ is a petal state. Furthermore, the number of such datasets we collect must be bounded---each $(s, \pi_s) \in \cT$ contributes $\Omega(\eps/D)$  to cumulative reachability, but since cumulative reachability is bounded from above by $\dimRL(\Pi)$ (\pref{lem:coverability}), we know that $\abs{\cT} \le \cO(D \cdot \dimRL(\Pi)/\eps)$.

% We show that $\widehat{P}_{s\to s'}^\pi$ is an accurate estimate of the transition probability for some (imaginary) MRP. \zeyu{Do we need to mention for $s\in\SRem{\pi}$ this MRP has nearly the same occupancy as the original MDP?} Therefore, $\mathsf{POPLER}$ can identify whether a new petal state $\bar{s} \notin \cT$ is $\Omega(\eps/D)$-reachable by $\pi$, and then add the tuple $(\bar{s}, \pi)$ to $\cT$, after which the next iteration starts

\subsubsection*{Evaluation Phase} 
Next, $\mathsf{POPLER}$ moves to the evaluation phase. Using the collected data, it executes the $\evaluate$ subroutine for every $\pi \in \Pi$ to get estimates $\wh{V}^\pi$ (\pref{line:evaluate}) corresponding to \(V^\pi\). For a given \(\pi \in \Pi\), the $\evaluate$ subroutine also constructs an empirical policy-specific MRP $\widehat \MRPsign^\pi_{\reachablestates}$ for every \(\pi \in \Pi\) and computes the value of \(\pi\) via dynamic programming on $\widehat\MRPsign^\pi_{\reachablestates}$. While the returned estimate $\wh{V}^\pi$ is biased, in the complete proof, we will show that the bias is negligible since it is now only due to the states in the petal $\cS_\pi$ which are \emph{not} $\Omega(\eps/D)$-reachable. Thus, we can guarantee that $\wh{V}^\pi$ closely estimates \(V^\pi\) for every $\pi \in \Pi$, and therefore $\mathsf{POPLER}$ returns a near-optimal policy.

\section{Conclusion and Discussion}\label{sec:conclusion}  
In this paper, we investigated when agnostic RL is statistically tractable in large state and action spaces, and introduced \Compname{} as a natural measure of complexity that only depends on the policy class, and is independent of the MDP rewards and transitions. We first showed that bounded \Compname{} is both necessary and sufficient for agnostic PAC RL under the generative access model. However, we also provided a negative result showing that bounded \Compname{} does not suffice for online RL, thus showing a surprising separation between agnostic RL with a generative model and online interaction. We then provided an additional structural assumption, called the sunflower property, that allows for statistically efficient learning in online RL. Our sample complexity bound for online RL is obtained using a novel exploration algorithm called $\mathsf{POPLER}$ that relies on certain policy-specific Markov Reward Processes to guide exploration, and takes inspiration from the classical importance sampling method and reward-free exploration algorithms for Tabular MDPs. Our results pave the way for several future lines of inquiry, discussed below. 
\begin{enumerate}[label=\(\bullet\)] 
\item \textit{Tight Characterization of Agnostic Online RL:} Perhaps the most interesting direction is exploring complexity measures to tightly characterize the minimax sample complexity for online RL (c.f.~the fundamental theorem of statistical learning). On the upper bound side, \pref{thm:sunflower} shows that bounded \Compname{} along with an additional sunflower property is sufficient for online RL. On the lower bound side, while we know that bounded \Compname{} is necessary (due to  \pref{thm:generative_lower_bound}), we do not know if the sunflower property is also necessary or whether it can be relaxed. Resolving this is an exciting next step. %However, we believe that any complexity measure that will characterize learnability in agnostic RL should have some flavor of the sunflower property; this is because (1) all the explicit examples we considered that are agnostic PAC learnable in online RL satisfy the sunflower property, (2) the only policy class we know that is not agnostically PAC learnable in online RL despite having bounded spanning capacity violates the sunflower property. Exploring the necessary and sufficient conditions on $\Pi$ for agnostic learnability in online RL is an exciting direction for future work. 
\item \textit{Instance-Dependent Complexity Measures:} Our primary focus in this paper was to understand the minimax sample complexity for agnostic PAC RL, i.e.~the worst-case bound on the number of samples needed to learn any stochastic MDP. However, real-life MDPs are not necessarily worst-case; thus, \gledit{proving instance-dependent bounds which depend on structural properties of the underlying (and unknown) MDP instance and the policy class is a fascinating future research direction.}
%going beyond the minimax bounds towards more instance-dependent bounds is a fascinating future research direction. 
% In particular, it would be interesting to understand how many samples are necessary and sufficient for agnostic learning as a function of the underlying (and unknown) MDP instance and the policy class, and study what kind of exploration strategies would achieve these bounds. 
The right instance-dependent sample complexity is not known even for the generative model. Since \Compname{} characterizes the minimax complexity under the generative model (Theorems \ref{thm:generative_upper_bound} and \ref{thm:generative_lower_bound}), and \Compname{} is the worst-case coverability coefficient over the set of all stochastic MDPs (\pref{lem:coverability}), one might conjecture that coverability coefficient (of the underlying MDP and policy class) characterizes the instance-dependent complexity under the generative model. We leave this as open question for future research. 

\item \textit{Agnostic RL with Stronger Feedback:}  In order to overcome the limitation of reward based RL, prior works for large state and action spaces have explored other forms of feedback, including learning via noisy comparison signals \citep{pacchiano2021dueling, sekhari2023contextual}, active queries to an expert \citep{ross2014reinforcement, sekhari2023selective}, noisy \(Q^\star\) feedback \citep{golowich2022can}, etc. Understanding when stronger feedback models can improve the statistical complexity of policy-based RL is largely open. In  \pref{app:expert_feedback}, we ask whether observing the optimal value function \(\crl{Q^\star(s, a)}_{a \in \cA}\) on visited states can help. Surprisingly, the answer depends on realizability of an optimal policy \(\pistar \in \Pi\). In particular,  
\begin{enumerate} 
\item \textit{Realizable Setting:} When \(\pistar \in \Pi\), then \(Q^\star\) feedback can be utilized to achieve an \(O\prn*{\poly(\log\abs{\Pi}, H, \nicefrac{1}{\epsilon})}\) sample complexity bound. Note that this strictly improves the sample complexity bounds in this paper for reward-based RL that have additional dependence on complexity measures like \Compname{}. 
\item \textit{Non-Realizable Setting:} When \(\pistar \notin \Pi\), as we show in \pref{app:expert_feedback}, one can not hope to learn with less than \(\Omega\prn*{\dimRL(\Pi)} \) samples in the worst case, even if the learner observes \(\crl{Q^\star(s, a)}_{a \in \cA}\) on the visited states. Furthermore, this lower bound holds even if \(\Pi\) contains a policy that obtains the same value as \(\pistar\) (but may not be optimal on all states). 
\end{enumerate}
Understanding the role of \(\pistar\)-realizability and exploring the benefits of other feedback models in agnostic RL are interesting future research directions. 

%\item Our analysis in the paper is restricted to deterministic policy classes; Characterizing agnostic learnability and extending \(\mathrm{POPLER}\) to stochastic policy classes is an important next step. It is also interesting to explore if we can further reduce, or completely remove, the dependence on \Compname{} in the sample complexity upper bound under additional structural properties like realizability of \(\pistar\) in \(\Pi\), policy completeness, etc., or under stronger feedback models. 
%
%
%could further reduce the dependence on \Compname{} in the sample complexity upper bound. For the role of realizability of \(\pistar\) in \(\Pi\), we note that our lower bound in 
%
%Furthermore, it would also be interesting to understand what other conditions on the policy class could improve the sample complexity. For example, in 

\item \textit{Other Directions:} Other future research directions include sharpening the sample complexity bound in \pref{thm:sunflower}, extending \(\mathrm{POPLER}\) for regret minimization, and developing computationally efficient algorithms. On the computational side, $\mathsf{POPLER}$ runs in time that scales polynomially with $\abs{\Pi}$ as well as $S,A,H$, which can be prohibitive for large-scale RL problems, and thus exploring end-to-end computationally efficient or oracle efficient algorithms would be interesting. 
\end{enumerate} 

%\ascomment{
%\begin{enumerate}
%\item Not needing additional assumptions like Policy Completeness, Bellman Completeness, etc in our proofs. 
%\item Do we expect any improvements if \(\pistar\) is in the class? 
%\item Discussion on proper vs improper learning? 
%\end{enumerate}} 

\arxiv{
\subsection*{Acknowledgements}
We thank Pritish Kamath, Jason D.~Lee, Wen Sun, and Cong Ma for helpful discussions. GL and NS are partially supported by National Science Foundation. Part of this work was completed while GL was visiting Princeton University. AR acknowledges support from  the ONR through award N00014-20-1-2336, 
ARO through award W911NF-21-1-0328, and from the DOE through award DE-SC0022199.
\ascomment{Nati, please add acknowledgements / grant support.} 
}

\newpage 
\neurIPS{\bibliographystyle{plainnat}} 
\bibliography{paper.bbl} 

\clearpage 

\appendix

\renewcommand{\contentsname}{Contents of Appendix}
\tableofcontents 
\addtocontents{toc}{\protect\setcounter{tocdepth}{3}} 

\clearpage 

\section{Detailed Comparison to Related Works}\label{apdx:related-works}
Reinforcement Learning (RL) has seen substantial progress over the past few years, with several different directions of work being pursued for efficiently solving RL problems that occur in practice. The classical approach to solving an RL problem is to model it as a tabular MDP. A long line of work \citep{sutton2018reinforcement, agarwal2019reinforcement, kearns2002near, brafman2002r, auer2008near, azar2017minimax, gheshlaghi2013minimax, jin2018q} has studied provably sample-efficient learning algorithms that can find the optimal policy in tabular RL. Unfortunately, the sample complexity of such algorithms unavoidably scales with the size of the state/action spaces, so they fail to be efficient in practical RL problems with large state/action spaces. In order to develop algorithms for the more practical large state/action RL settings, various assumptions have been considered in the prior works. In the following, we provide a detailed comparison of our setup and assumptions with the existing literature.

%These include: Various prior works have 
%
%On the other hand, the key focus of our work is to develop algorithms for MDPs with large state/action spaces. Towards that end, we take an agnostic viewpoint of RL: we assume that the learner is given a policy class $\Pi$ (which the learner believes contains a good policy for the underlying MDP), and the goal of the learner is to find a policy that perform as well as the best policy in the given class. 

%\paragraph{Learning in Tabular MDP} \citep{sutton2018reinforcement, agarwal2019reinforcement, kearns2002near, brafman2002r, auer2008near, azar2017minimax, gheshlaghi2013minimax, jin2018q} have studied provably sample-efficient learning algorithms in the classical tabular RL literature. Tabular algorithms' sample complexity scales with state and action space sizes, causing them to fail in "rich observation" settings with large state/action spaces. This contrasts with our approach, which targets rich observation settings and competes against the best policy in a given policy class.

\paragraph{RL with Function Approximation.} A popular paradigm for  developing algorithms for MDPs with large state/action spaces is to use function approximation to either model the MDP dynamics or optimal value functions. Over the last decade, there has been a long line of work \citep{jiang2017contextual, dann2018oracle, sun2019model, du2019provably, wang2020reinforcement, du2021bilinear, foster2021statistical, jin2021bellman, zhong2022gec, foster2023tight} in understanding structural conditions on the function class and the underlying MDP that enable statistically efficient RL. However, all of these works rely crucially on the realizability assumption, namely that the true model / value function belong to the chosen class. %Furthermore, the prior works using function approximation make strong additional assumptions, like Bellman Completeness, strong modeling assumptions, etc. 
%Both of these assumptions are too strong to hold in practice. The former is difficult to verify for the underlying problem, and the latter is non-monotone, in the sense that adding more functions may break Bellman Completeness. 
Unfortunately, such an assumption is too strong to hold in practice.  Furthermore, the prior works using function approximation make additional assumptions like Bellman Completeness that are difficult to verify for the underlying task. 

In our work, we study the problem of agnostic RL to sidestep these challenges. In particular, instead of modeling the value/dynamics, the learner now models ``good policies" for the underlying task, and the learning objective is to find a policy that can perform as well as the best in the chosen policy class. We note that while a realizable value class/dynamics class \(\cF\) can be converted into a realizable policy class \(\Pi_{\cF}\) by choosing the greedy policies for each value function/dynamics, the converse is not true. Thus, our agnostic RL objective relies on a strictly weaker modeling assumption. 

\paragraph{Connections to Decision-Estimation Coefficient (DEC).} The seminal work of  \citet{foster2021statistical} provides a unified complexity measure called Decision-Estimation Coefficient (DEC) that characterizes the complexity of model-based RL. Given the generality of the E2D algorithm of \citet{foster2021statistical}, one may be wondering if our results can be recovered using their framework via a model-based approach. In particular, can we recover the sample complexity bound in Theorems \ref{thm:generative_upper_bound} or \ref{thm:sunflower} by considering the model class \(\cMdet\) or $\cMsto$ along with the decision set $\Pi$. To the best of our knowledge, the framework of \citet{foster2021statistical} do not directly recover our results, however, the complexity measures are closely related. Note that one can upper-bound the DEC by the coverability coefficient \cite{xie2022role}; furthermore we show that $\dimRL(\Pi)$ is worst-case coverability (\pref{lem:coverability}) so it follows that DEC is upper-bounded by $\dimRL(\Pi)$. However, the algorithm in \cite{foster2021statistical} achieves regret bounds which scale with $\mathrm{DEC} \cdot \log \abs{\cMsto}$, which can be vacuous in the large state-space setting since $\log \abs{\cMsto} \propto \abs{\cS}$; In contrast, our upper bounds have no explicit dependence on $\abs{\cS}$. 

\paragraph{RL with Rich Observations.} Various settings have been studied where the dynamics are determined by a simple latent state space, but instead of observing the latent states directly, the learner receives rich observations corresponding to the underlying latent states. These include the Block MDP \citep{krishnamurthy2016pac,du2019provably, misra2020kinematic, mhammedi2023representation}, Low-Rank MDPs \citep{uehara2021representation, huang2023reinforcement}, Exogenous MDPs \citep{efroni2021provably, xie2022role, efroni2022sample}, etc. However, all of these prior works assume that the learner is given a realizable decoder class (consisting of functions that map observations to latent states) that contains the true decoder for the underlying MDP. Additionally, they require strong assumptions on the underlying latent state space dynamics, e.g.~it is tabular or low-rank, in order to make learning tractable. Thus, their guarantees are not agnostic. In fact, given a realizable decoder class and additional structure on the latent state dynamics, one can construct a policy class that contains the optimal policy for the MDP, but the converse is not true. Thus, our agnostic RL setting is strictly more general. 

\paragraph{Relation to Exponential Lower Bounds for RL with Function Approximation.}
Recently, many statistical lower bounds have been developed in RL with function approximation under only realizability. A line of work including \citep{ruosong2020statisticallimits, zanette2021exponential, weisz2021exponential, foster2021offline} showed that the sample complexity scales exponentially in the horizon \(H\) for learning the optimal policy for RL problems where only the optimal value function \(Q^\star\) is linear w.r.t.~the given features. Similarly, \citet{du2019good} showed that one may need exponentially in \(H\) even if the optimal policy is linear w.r.t.~the true features. These lower bounds can be extended to our agnostic RL setting, giving similar exponential in \(H\) lower bounds for agnostic RL, thus supplementing the well-known lower bounds \citep{krishnamurthy2016pac} which show that agnostic RL is intractable without additional structural assumptions on the policy class. However, to recall, the focus of this paper is to propose assumptions, like \pref{def:dimension} or \ref{def:core_policy}, that circumvent these lower bounds and allow for sample efficient agnostic RL. 

\paragraph{Importance Sampling for RL.} 
Various importance sampling based estimators 
\citep{xie2019towards, jiang2016doubly, gottesman2019combining, yin2020asymptotically, thomas2016data, nachum2019dualdice} have been developed in RL literature to provide reliable off-policy evaluation in offline RL. However, these methods also require realizable value function approximation and  rely on additional assumptions on the off-policy/offline data, in particular, that the offline data covers the state/action space that is explored by the comparator policy. We note that this line of work does not directly overlap with our current approach but provides a valuable tool for dealing with off-policy data.

\paragraph{Agnostic RL in Low-Rank MDPs.} \citet{sekhari2021agnostic} explored agnostic PAC RL in low-rank MDPs, and showed that one can perform agnostic learning w.r.t.~any policy class for MDPs that have a small rank. While their guarantees are similar to ours, i.e., they compete with the best policy in the given class and do not assume access to a realizable dynamics / value-function class, the key objectives of the two works are complementary. \citet{sekhari2021agnostic}  explore assumptions on the underlying MDP dynamics which suffice for agnostic learning for any given policy class, whereas we ask what assumptions on the given policy class suffice for agnostic learning for any underlying dynamics. Exploring the benefits of structure in both the policy class and the underlying MDP in agnostic RL is an interesting direction for future research. 

\paragraph{Policy Gradient Methods.}  
A significant body of work in RL, in both theory  \citep{agarwal2021theory, abbasi2019politex, bhandari2019global, liu2020improved, agarwal2020pc, zhan2021policy, xiao2022convergence} and practice \citep{kakade2001natural, kakade2002approximately, levine2013guided, schulman2015trust, schulman2017proximal}, studies policy-gradient based methods that directly search for the best policy in a given policy class. These approaches often leverage mirror-descent style analysis,  and can deliver guarantees that are similar to ours, i.e.~the returned policy can compete with any policy in the given class, which is an agnostic guarantee in some sense. However, these works primarily study smooth and parametric  policy classes, e.g.~tabular and linear policy classes, which limits their applicability for a broader range of problem instances. Furthermore, they require strong additional assumptions to work: for instance, that the learner is given a good reset distribution that can cover the occupancy measure of the policy that we wish to compare to, and that the policy class satisfies a certain ``policy completeness assumption"; both of which are difficult to verify in practice. In contrast, our work makes no such assumptions but instead studies what kind of policy classes are learnable for any MDP.

\paragraph{CPI, PSDP, and Other Reductions to Supervised Learning.} Various RL methods have been developed that return a policy that performs as well as the best policy in the given policy class, by reducing the RL problem from supervised learning. The key difference from policy-gradient based methods (discussed previously) is that these approaches do not require a smoothly parameterized policy class, but instead rely on access to a supervised learning oracle w.r.t.~the given policy class. Popular approaches include Conservative Policy Iteration (CPI) \citep{kakade2002approximately, kakade2003sample, brukhim2022boosting, agarwal2023variance}, PSDP \citep{bagnell2003policy}, Behavior Cloning \citep{ross2010efficient, torabi2018behavioral}, etc. We note that these algorithms rely on additional assumptions, including ``policy completeness assumption" and a good sampling / reset  distribution that covers the policies that we wish to compare to; in comparison, we do not make any such assumptions in our work. 

Efficient RL via reductions to online regression oracles w.r.t.~the given policy class has also been studied, see, e.g., DAgger \citep{ross2011reduction}, AggreVaTe \citep{ross2014reinforcement}, etc. However, these algorithms rely on much stronger feedback. In particular the learner, on the states which it visits, can query an expert policy (that we wish to complete with) for its actions or the value function. On the other hand, in this paper, we restrict ourselves to the standard RL setting where the learner only gets instantaneous reward signal. In \pref{app:expert_feedback} we investigate whether such stronger feedback can be used for agnostic RL.

\paragraph{Reward-Free RL.} From a technical viewpoint, our algorithm (\pref{alg:main}) share similaries to algorithms developed in the reward-free RL literature \citep{jin2020reward}. In reward-free RL, the goal of the learner is to output a dataset or a set of policies, after interacting with the underlying MDP, that can be later used for planning (with no further interaction with the MDP) for downstream reward functions. The key ideas in our \pref{alg:main}, in particular, that the learner first finds states \(\cI\) that are \(\Omega(\epsilon)\)-reachable and corresponding policies that can reach them, and then outputs datasets \(\crl{\cD_s}_{s \in \cI}\) that can be later used for evaluating any policy \(\pi \in \Pi\), share similarities to algorithmic ideas used in reward-free RL. However, our algorithm strictly generalizes prior works in reward-free RL, and in particular can work with large state-action spaces where the notion of reachability as well as the offline RL objective, is defined w.r.t.~the given policy class. In comparison, prior reward-free RL works compete with the best policy for the underlying MDP, and make structure assumptions on the dynamics, e.g. tabular structure \citep{jin2020reward, menard2021fast, li2023minimax} or linear dynamics \citep{wang2020reward, zanette2020provably, zhang2021reward, wagenmaker2022reward}, to make the problem tractable.

\begin{comment} 
\cite{wang2020reward, wagenmaker2022reward, menard2021fast, zhang2020nearly, zanette2020provably, zhang2021reward, chen2022statistical, li2023minimax}
\begin{enumerate}
\item \citep{jin2020reward} - First work on reward free exploration in online RL - " We achieve this by findingexploratory policies that visit each “significant” state with probability proportional to its maximum vis-itation probability underanypossible policy."

\item \citep{wang2020reward} - linear MDPs only. same for \citep{zanette2020provably, zhang2021reward} 

\item \citep{wagenmaker2022reward} - same as above. 

\item \citep{menard2021fast} - Tabular MDPs. 

\item \citep{zhang2021reward} - Non-linear RL but still under function approximation setting. 

\item \citep{li2023minimax} - Tabular MDPs only. Same for \citep{zhang2020nearly} 
\end{enumerate}
\end{comment} 

\paragraph{Instance Optimal Measures.} Several recent works including \citet{wagenmaker2022instance, tirinzoni2023optimistic, JMLR:v14:bottou13a, al2023active} have explored instance-dependent complexity measures for PAC RL. At a high level,  these instance-dependent bounds are obtained via similar algorithmic ideas to ours that combine reward-free exploration with policy elimination. However, there are major differences. Firstly, these prior works in instance-dependent PAC RL  operate under additional modeling assumptions on the MDP dynamics, e.g., that it is a tabular or linear MDP. Secondly, they require additional reachability assumptions on the state space, which is restrictive for MDPs with large states/actions; in fact, their sample complexity bounds typically have a dependence on the number of states/actions in the lower order terms. Finally, they implicitly assume that the optimal policy \(\pistar \in \Pi\), and thus the provided algorithms do not transfer cleanly to the agnostic PAC RL setting considered in our paper.

\paragraph{Other Complexity Measures for RL.} 
A recent work by \citet{mou2020sample} proposed a new notion of eluder dimension for the policy class, and provide upper bounds for policy-based RL when the class $\Pi$ has bounded eluder dimension. However, they require various additional assumptions: that the policy class contains the optimal policy, the learner has access to a generative model, and that the optimal value function has a gap. On the other hand, we do not make any such assumptions and characterize learnability in terms of \Compname{} or the size of the minimal sunflower in \(\Pi\). We discuss connections to the eluder dimension, as well as other classical complexity measures in learning theory in \pref{app:complexity-measures}.

\section{Examples of Policy Classes}\label{app:examples-policy-classes}
\par In this section, we will prove that examples in \pref{sec:C_pi} have both bounded \Compname{} and the \coreset{} property with small \(K\) and \(D\). To facilitate our discussion, we define the following notation: for any policy class $\Pi$ we let 
\begin{align*}
\dimRL_h(\Pi)\coloneqq \max_{M \in \cMdet} C^\mathsf{reach}_h(\Pi; M),
\end{align*}
where $C^\mathsf{reach}_h(\Pi; M)$ is defined in \pref{def:dimension}. That is, $\dimRL_h(\Pi)$ is the per-layer spanning capacity of $\Pi$.
Then as defined in \pref{def:dimension}, we have
$$\dimRL(\Pi) = \max_{h\in [H]}\dimRL_h(\Pi).$$

\paragraph{Tabular MDP.}
Since there are at most $\abs{\cS_h}$ states in layer $h$, it is obvious that $\dimRL_h(\Pi)\le \abs{\cS_h}A$, so therefore $\dimRL(\Pi) \le SA$. Additionally, if we choose $\Picore = \{\pi_a: \pi_a(s) = a, a\in \mA\}$ to be the set of policies which play the constant $a$ for each $a \in \cA$  and $\mS_{\pi} = \mS$ for every $\pi\in \Pi$, then any partial trajectory which satisfies the condition in \pref{def:core_policy} is of the form $(s_h, a_h)$, which is consistent with $\pi_{a_h}\in \Picore$. Hence $\Pi$ is a $(A, S)$-\sunflower.

\paragraph{Contextual Bandit.}
Since there is only one layer, any deterministic MDP has a single state with at most $A$ actions possible, so $\dimRL(\Pi)\le A$. Additionally, if we choose $\Picore = \{\pi_a: \pi_a(s)\equiv a, a\in\mA\}$, and $\mS_{\pi} = \emptyset$ for every $\pi\in\Pi$, then any partial trajectory which satisfies the condition in \pref{def:core_policy} is in the form $(s, a)$, which is consistent with $\pi_a\in \Picore$. Hence $\Pi$ is a $(A, 0)$-\sunflower.

\paragraph{$H$-Layer Contextual Bandit.} 
By induction, it is easy to see that any deterministic MDP spans at most $A^{h-1}$ states in layer $h$, each of which has at most $A$ actions. Hence $\dimRL(\Pi)\le A^{H}$.
Additionally, if we choose 
$$\Picore = \{\pi_{a_1, \cdots, a_H}: \pi_{a_1, \cdots, a_H}(s_h)\equiv a_h, a_1, \cdots, a_H\in\mA\}$$
and $\mS_{\pi} = \emptyset$ for every $\pi\in\Pi$, then any partial trajectory which satisfies the condition in \pref{def:core_policy} is in the form $(s_1, a_1, \cdots, s_H, a_H)$, which is consistent with $\pi_{a_1, a_2, \cdots, a_H}\in \Picore$. Hence $\Pi$ is a $(A^H, 0)$-\sunflower.

\paragraph{$\ell$-tons.}
In the following, we will denote $\Pi_\ell \coloneqq \PiLton$. We will first prove that $\dimRL(\Pi_\ell)\le 2H^{\ell}$. To show this, we will prove that $\dimRL_h(\Pi_\ell)\le 2h^\ell$ by induction on $H$. When $H = 1$, the class is a subclass of the above contextual bandit class, hence we have $\dimRL_1(\Pi_{\ell})\le 2$. Next, suppose $\dimRL_{h-1}(\Pi_\ell)\le 2(h-1)^\ell$. Fix any deterministic MDP and call the first state $s_1$. Policies taking $a=1$ at $s_1$ can only take $a=1$ on at most $\ell-1$ states in the following layers. Such policies reach at most $\dimRL_{h-1}(\Pi_{\ell-1})$ states in layer $h$. Policies taking $a=0$ at $s_1$ can only take $a=1$ on at most $\ell$ states in the following layers. Such policies reach at most $\dimRL_{h-1}(\Pi_{\ell})$ states in layer $h$.  Hence we get
$$\dimRL_{h}(\Pi_{\ell})\le \dimRL_{h-1}(\Pi_{\ell-1}) + \dimRL_{h-1}(\Pi_{\ell})\le 2(h-1)^{\ell-1} + 2(h-1)^{\ell}\le 2h^{\ell}.$$
This finishes the proof of the induction hypothesis. Based on the induction argument, we get
$$\dimRL(\Pi_\ell) = \max_{h\in [H]}\dimRL_h(\Pi_\ell)\le 2H^\ell.$$
Additionally, choose 
$$\Picore = \{\pi_0\}\cup\{\pi_h:1\le h\le H\},$$
where $\pi_0(s)\equiv 0$, and \(\pi_h\) chooses the action \(1\) on all the states at layer \(h\), i.e., $\pi_h(s)\coloneqq \ind{s\in\mS_h}$. For every $\pi\in \Pi_\ell$, we choose $\mS_\pi$ to be the states for which $\pi(s) = 1$ (there are at most $\ell$ such states). Fix any partial trajectory $\tau = (s_{h}, a_h \cdots, s_{h'}, a_{h'})$ which satisfies $\pi\cons \tau$. Suppose that for all $i \in (h, h']$, $s_i\not\in\mS_\pi$. Then we must have $a_i = 0$ for all $i \in (h, h']$. Hence $\pi_h \cons \tau$ (if $a_h = 1$) or $\pi_0 \cons \tau$ (if $a_h = 0$), and $\tau$ is consistent with some policy in $\Picore$.  Therefore, $\Pi_\ell$ is an $(H+1, \ell)$-\sunflower.

\paragraph{$1$-Active Policies.} 
We will first prove that $\dimRL(\Pioneactive)\le 2H$. For any deterministic MDP, we use $\bar{\mS}_h$ to denote the set of states reachable by $\Pioneactive$ at layer $h$. We will show that $\bar{\mS}_h \le 2h$ by induction on $h$. For $h = 1$, this holds since any deterministic MDP has only one state in the first layer. Suppose it holds at layer $h$. Then, we have
$$|\bar{\cS}_{h+1}|\le |\{(s, \pi(s)):s\in\bar{\cS}_h, \pi\in \Pi\}|.$$
Note that policies in $\Pioneactive$ must take $a=0$ on every $s \notin \{s_{(1,1)}, s_{(1,2)}, \cdots, s_{(1, H)}\}$. Hence $|\{(s, \pi(s)) ~|~s\in\bar{\cS}_h, \pi\in \Pi\}|\le |\bar{\mS}_h| + 1\le h+1$. Thus, the induction argument is complete. As a consequence we have $\dimRL_{h}(\Pi)\le 2h$ for all $h$, so 
$$\dimRL(\Pioneactive) = \max_{h\in [H]} \dimRL_h(\Pioneactive)\le 2H.$$
Additionally, if we choose $\mS_\pi = \{s_{(1,1)}, s_{(1,2)}, \cdots, s_{(1, H)}\}$ for all $\pi\in \Pi$ as well as 
$$\Picore = \{\pi_0\}\cup\{\pi_h:1\le h\le H\},$$
where $\pi_0(s)\equiv 0$ and $\pi_h(s)\coloneqq \ind{s\in\mS_h}$. Now fix any partial trajectory $\tau = (s_{h}, a_h \cdots, s_{h'}, a_{h'})$ which satisfies $\pi\cons \tau$. If we have $i \in (h, h']$, $s_i\not\in\mS_\pi$, then we must have $a_i = 0$. Thus, $\pi_h\cons \tau$ (if $a_h = 1$) or $\pi_0\cons \tau$ (if $a_h = 0$), so $\tau$ is consistent with some policy in $\Picore$. Therefore, $\Pioneactive$ is a $(H+1, H)$-\sunflower.

\paragraph{All-Active Policies.} For any deterministic MDP, there is a single state $s_{(j,1)}$ in the first layer. Any policy which takes $a=1$ at state $s_{(j,1)}$ must belong to $\PiJactive$. Hence such policies can reach at most $\dimRL_{h-1}(\PiJactive)$ states in layer $h$. For polices which take action $0$ at state $h$, all these policies will transit to a fixed state in layer $2$. Hence such policies can reach at most $\dimRL_{h-1}(\Piactive)$ states at layer $h$. Therefore, we get
$$\dimRL_h(\Piactive) \le \dimRL_{h-1}(\Piactive) + \max_j \dimRL_{h-1}(\PiJactive)\le \dimRL_{h-1}(\Piactive) + 2(h-1).$$
By telescoping, we get
$$\dimRL_h(\Piactive)\le h(h-1),$$
which indicates that
$$\dimRL(\Piactive) = \max_{h\in [H]} \dimRL_h(\Piactive)\le H(H-1).$$
Additionally, if we choose $\mS_\pi = \{s_{(j,1)}, \cdots, s_{(j,H)}\}$ for all $\pi\in \PiJactive$, as well as 
$$\Picore = \{\pi_0\}\cup\{\pi_h:1\le h\le H\},$$ 
where $\pi_0(s)\coloneqq 0$ and $\pi_h(s)\coloneqq \ind{s\in\mS_h}$. Now fix any partial trajectory $\tau = (s_{h}, a_h \cdots, s_{h'}, a_{h'})$ which satisfies $\pi\cons \tau$. If we have $i \in (h, h']$, $s_i\not\in\mS_\pi$, then we must have $a_i = 0$. Thus, $\pi_h\cons \tau$ (if $a_h = 1$) or $\pi_0\cons \tau$ (if $a_h = 0$), so $\tau$ is consistent with some policy in $\Picore$. Therefore, $\Piactive$ is a $(H+1, H)$-\sunflower.

\paragraph{Policy Classes for Continuous State Spaces.} In some cases, it is possible to construct policy classes over continuous state spaces that have bounded \Compname{}. For example, consider $\Pisingleton$, which is defined over a discrete (but large) state space. We can extend this to continuous state space by defining new state spaces $\cS_h = \crl*{s_{(x, h)} : x \in \bbR}$ for all $h \in [H]$, action space $\cA = \crl{0,1}$, and policy class 
\begin{align*}
    \wt{\Pisingleton} \coloneqq \crl*{\pi_{(i, h')} : \pi_{(i, h')} (s_{(x, h)}) = \ind{x \in [i, i+1) \text{ and } h = h'}, i\in \bbN, h' \in [H]}.
\end{align*}
Essentially, we have expanded each state to be an interval on the real line. Using the same reasoning, we have the bound $\dimRL(\wt{\Pisingleton}) = H+1$. One can also generalize this construction to the policy class $\wt{\PiLton}$ and preserve the same value of $\dimRL$.\footnote{\gledit{To compute the $(K,D)$ values of $\wt{\PiLton}$, the previous arguments do not go through, since the sets $\cS_\pi$ are infinite. With a suitable extension of \pref{def:core_policy} to allow for non-Markovian $\Picore$, it is possible to show that $\wt{\PiLton}$ is an $(\cO(H^\ell), 0)$-\sunflower; Furthermore, the proof of \pref{thm:sunflower} can be easily adapted to work under this extension. }} %This extension to non-Markovian $\Picore$ does not affect the statement of \pref{thm:sunflower}, so our results show polynomial sample complexity for learning $\wt{\PiLton}$.}}\gene{see this footnote.} 

However, in general, this expansion to continuous state spaces may blow up the \Compname{}. Consider a similar modification to $\Pioneactive$ (again, with the same new state space and action space $\cA = \crl{0,1}$):
\begin{align*}
    \wt{\Pioneactive} \coloneqq \crl{\pi :  \pi(s_{(x,h)}) = 0 \text{ if } x \notin [0,1)]}.
\end{align*}
While $\dimRL(\Pioneactive) = \Theta(H)$, it is easy to see that $\dimRL(\wt{\Pioneactive}) = 2^H$ since one can construct a $H$-layer deterministic tree using states in $[0,1)$ as every $(s,a)$ pair at layer $H$ will be reachable by $\wt{\Pioneactive}$.

\section{Proofs for \pref{sec:C_pi}}\label{app:cpi-section-proofs}

\subsection{Proof of \pref{lem:coverability}}
Fix any $M \in \cMsto$, as well as $h \in [H]$. We claim that  
\begin{align*} 
    \Gamma_h \coloneqq \sum_{s_h \in \cS_h, a_h \in \cA_h} \sup_{\pi \in \Pi} d^\pi_h(s_h,a_h; M) \le \max_{M' \in \cMdet} C^\mathsf{reach}_h(\Pi; M').\numberthis \label{eq:lem-coverability-eq}
\end{align*} 
Here, $d^\pi_h(s_h,a_h; M)$ is the state-action visitation distribution of the policy \(\pi\) on MDP $M$.

We first set up additional notation. Let us define a \emph{prefix} as any tuple of pairs of the form 
\begin{align*}
    (s_1, a_1, s_2, a_2, \dots, s_k, a_k) \quad \text{or} \quad (s_1, a_1, s_2, a_2, \dots, s_{k}, a_{k}, s_{k+1}).
\end{align*}
We will denote prefix sequences as $(s_{1:k}, a_{1:k})$ or $(s_{1:k+1}, a_{1:k})$ respectively. For any prefix $(s_{1:k}, a_{1:k})$ (similarly prefixes of the type $(s_{1:k+1}, a_{1:k})$) we let $d^\pi_h(s_h, a_h \mid  (s_{1:k}, a_{1:k}) ; M)$ denote the conditional probability of reaching $(s_h, a_h)$ under policy $\pi$ given one observed the prefix $(s_{1:k}, a_{1:k})$ in MDP $M$, with $d^\pi_h(s_h, a_h \mid  (s_{1:k}, a_{1:k}) ; M) = 0$ if $\pi \not\cons (s_{1:k}, a_{1:k})$ or $ \pi \not\cons (s_h, a_h)$. 

In the following proof, we assume that the start state $s_1$ is fixed, but this is without any loss of generality, and the proof can easily be adapted to hold for stochastic start states. 

Our strategy will be to explicitly compute the quantity $\Gamma_h$ in terms of the dynamics of $M$ and show that we can upper bound it by a ``derandomized'' MDP $M'$ which maximizes reachability at layer $h$. Let us unroll one step of the dynamics: 
\begin{align*}
    \Gamma_h
    &\coloneqq \sum_{s_h \in \cS_h, a_h \in \cA} \sup_{\pi \in \Pi} d^\pi_h(s_h, a_h; M) \\
    &\overset{(i)}{=} \sum_{s_h \in \cS_h, a_h \in \cA} \sup_{\pi \in \Pi} d^\pi_h(s_h, a_h \mid  s_1 ; M) , \\
    &\overset{(ii)}{=} \sum_{s_h \in \cS_h, a_h \in \cA} \sup_{\pi \in \Pi} \crl*{ \sum_{a_1 \in \cA} d^\pi_h(s_h, a_h\mid s_1, a_1; M) } \\
    &\overset{(iii)}{\leq} \sum_{a_1 \in \cA} \sum_{s_h \in \cS_h, a_h \in \cA} \sup_{\pi \in \Pi}  d^\pi_h(s_h, a_h \mid  s_1, a_1; M).
\end{align*}
The equality $(i)$ follows from the fact that $M$ always starts at $s_1$. The equality $(ii)$ follows from the fact that $\pi$ is deterministic, so there exists exactly one $a' = \pi(s_1)$ for which $d^\pi_h(s_h, a_h\mid s_1, a'; M) = d^\pi_h(s_h, a_h\mid s_1 ; M)$, with all other $a'' \ne a'$ satisfying $d^\pi_h(s_h, a_h|s_1, a''; M) = 0$. The inequality $(iii)$ follows by swapping the supremum and the sum. 

Continuing in this way, we can show that
\begin{align*}
\Gamma_h &= \sum_{a_1 \in \cA} \sum_{s_h \in \cS_h, a_h \in \cA} \sup_{\pi \in \Pi} \crl*{ \sum_{s_2 \in \cS_2} P(s_2 | s_1, a_1) \sum_{a_2 \in \cA} d^\pi_h(s_h, a_h \mid  (s_{1:2}, a_{1:2}); M) } \\ 
&\le \sum_{a_1 \in \cA} \sum_{s_2 \in \cS_2} P(s_2 | s_1, a_1) \sum_{a_2 \in \cA} \sum_{s_h \in \cS_h, a_h \in \cA} \sup_{\pi \in \Pi}  d^\pi_h(s_h, a_h\mid (s_{1:2}, a_{1:2}); M)  \\
&\hspace{0.5in} \vdots \\ 
&\le \sum_{a_1 \in \cA} \sum_{s_2 \in \cS_2} P(s_2 | s_1, a_1) \sum_{a_2 \in \cA} \dots \sum_{s_{h-1} \in \cS_{h-1}} P(s_{h-1}| s_{h-2}, a_{h-2})  \\ 
&\hspace{2.0in} \times \sum_{a_{h-1} \in \cA} \sum_{s_h \in \cS_h, a_h \in \cA} \sup_{\pi \in \Pi}  d^\pi_h(s_h, a_h\mid (s_{1:h-1}, a_{1:h-1}); M). 
\end{align*} 
Now we examine the conditional visitation $d^\pi_h(s_h, a_h\mid (s_{1:h-1}, a_{1:h-1}); M)$. Observe that it can be rewritten as 
\begin{align*}
    d^\pi_h(s_h, a_h\mid (s_{1:h-1}, a_{1:h-1}); M) = P(s_h|s_{h-1}, a_{h-1}) \cdot \ind{\pi \cons (s_{1:h}, a_{1:h})}.
\end{align*}
Plugging this back into the previous display, and again swapping the supremum and the sum, we get that 
\begin{align*}
\Gamma_h &\le \sum_{a_1 \in \cA} \dots \sum_{s_{h} \in \cS_{h}} \P(s_{h}| s_{h-1}, a_{h-1}) \sum_{a_{h} \in \cA} \sup_{\pi \in \Pi}   \ind{\pi \cons (s_{1:h}, a_{1:h})} \\ 
&= \sum_{a_1 \in \cA} \dots \sum_{s_{h} \in \cS_{h}} P(s_{h}| s_{h-1}, a_{h-1}) \sum_{a_{h} \in \cA}   \ind{\exists \pi \in \Pi: \pi \cons (s_{1:h}, a_{1:h})}
\end{align*} 
Our last step is to derandomize the stochastic transitions in the above stochastic MDP, simply by taking the sup over the transition probabilities: 
\begin{align*} 
&\Gamma_h \le \sum_{a_1 \in \cA} \sup_{s_2 \in \cS_2} \sum_{a_2 \in \cA} \dots \sup_{s_{h} \in \cS_{h}} \sum_{a_{h} \in \cA}   \ind{\exists \pi \in \Pi: \pi \cons (s_{1:h}, a_{1:h})} = \max_{M' \in \cMdet} C^\mathsf{reach}_h(\Pi; M').
\end{align*}
The right hand side of the inequality is exactly the definition of $\max_{M' \in \cMdet} C^\mathsf{reach}_h(\Pi; M')$, thus proving Eq.~\pref{eq:lem-coverability-eq}. In particular, the above process defines the deterministic MDP which maximizes the reachability at level $h$. Taking the maximum over $h$ as well as supremum over $M$, we see that $\sup_{M \in \cMsto} C^\mathsf{cov}(\Pi; M) \le \dimRL(\Pi)$. Furthermore, from the definitions we have 
\begin{align*}
    \dimRL(\Pi) = \sup_{M \in \cMdet} C^\mathsf{cov}(\Pi; M) \le \sup_{M \in \cMsto} C^\mathsf{cov}(\Pi; M).
\end{align*}
This concludes the proof of \pref{lem:coverability}.\qed 

\subsection{Coverability is Not Sufficient for Online RL} 
In this section, we observe that bounded coverability by itself is not sufficient to ensure sample efficient agnostic PAC RL in the online interactive model. First note that \pref{thm:lower-bound-online} already shows this indirectly. In particular, in \pref{thm:lower-bound-online}, we show that there exists a policy class with bounded \Compname{} that is hard to learn in online RL. However, recall \pref{lem:coverability} which implies that any policy class with bounded  \Compname{} must also have bounded coverability; and thus the lower bound in \pref{thm:lower-bound-online} can be trivially extended to argue that bounded coverability by itself does not suffice for statistically efficient agnostic online RL. 

However, we can also show the insufficiency of coverability through a much simpler route by directly invoking the lower bound construction in \cite{sekhari2021agnostic}. In particular, \citet{sekhari2021agnostic} provides a construction for a low rank MDP with rich observations which satisfies $C^\mathsf{cov}(\Pi; M) = \cO(1)$ for every $M \in \cM$, but still needs $2^{\Omega(H)}$ many samples for any $(\Theta(1), \Theta(1))$-PAC learner (see Theorem 2 in their paper for more details; we simply set $d= \Theta(H)$ to get our lower bound). 

\gledit{We do not know if coverability is also insufficient for the generative model setting; we conjecture that one may be able to show, using a similar construction, that coverability is insufficient when $\dimRL(\Pi)$ is large, showing that one cannot adapt to benign problem instances.}
\section{Proofs for \pref{sec:generative}} \label{app:generative}  

\subsection{Proof of \pref{thm:generative_upper_bound}}\label{app:proof_generative_upper_bound}

\begin{algorithm}[!htp] 
\caption{$\mathsf{TrajectoryTree}$ \citep{kearns1999approximate}}\label{alg:trajectory-tree}
    \begin{algorithmic}[1]
            \Require Policy class $\Pi$, generative access to the underlying MDP $M$, number of samples $n$ 
            \State Initialize dataset of trajectory trees $\cD = \emptyset$.
            \For{$i=1, \dots, n$}
            \State Initialize trajectory tree $\wh{T}_i = \emptyset$.
            \State Sample initial state $s_1^{(i)} \sim \mu$.
            \While{$\mathrm{True}$}   \hfill
            \algcomment{Sample transitions and rewards for a trajectory tree}
            \State{Find any unsampled $(s,a)$ s.t.~$(s,a)$ is reachable in $\wh{T}_i$ by some $\pi \in \Pi$.}\label{line:trajectory-tree-sample} 
            \If{no such $(s,a)$ exists} \textbf{break}
            \EndIf
            \State Sample $s' \sim P(\cdot|s,a)$ and $r \sim R(s,a)$ \label{line:generative-query}
            \State Add transition $(s,a,r,s')$ to $\wh{T}_i$.
            \EndWhile
            \State $\cD \gets \cD \cup \wh{T}_i$.
            \EndFor
            \For{$\pi \in \Pi$}
            \hfill
            \algcomment{Policy evaluation}
            \State Set $\wh{V}^\pi \gets \frac{1}{n} \sum_{i=1}^n \wh{v}^\pi_i$, where $\wh{v}^\pi_i$ is the cumulative reward of $\pi$ on $\wh{T}_i$. 
            \EndFor
            \State \textbf{Return} $\wh{\pi} \gets \argmax_{\pi \in \Pi} \wh{V}^\pi$.
    \end{algorithmic}
\end{algorithm} 

We show that, with minor changes, the $\mathsf{TrajectoryTree}$ algorithm of \cite{kearns1999approximate} attains the guarantee in \pref{thm:generative_upper_bound}. The pseudocode can be found in \pref{alg:trajectory-tree}. The key modification is \pref{line:trajectory-tree-sample}: we simply observe that only $(s,a)$ pairs which are reachable by some $\pi \in \Pi$ in the current tree $\wh{T}_i$ need to be sampled (in contrast, in the original algorithm of \cite{kearns1999approximate}, they sample all $A^H$ transitions). 

Fix any $\pi \in \Pi$. For every trajectory tree $i \in [n]$, the algorithm has collected enough transitions so that $\wh{v}_i^\pi$ is well-defined, by \pref{line:trajectory-tree-sample} of the algorithm. \gledit{By the sampling process, it is clear that the values $\crl{\wh{v}_i^\pi}_{i \in [n]}$ are i.i.d.~generated. We claim that they are unbiased estimates of $V^\pi$. Observe that one way of defining $V^\pi$ is the expected value of the following process: 
\begin{enumerate}[label=\((\arabic*)\)]
    \item 
For every $(s, a) \in \cS \times \cA$, independently sample a next state $s' \sim P(\cdot|s,a)$ and a reward $r \sim R(s,a)$ to define a deterministic MDP $\wh{M}$ 
\item Return the value $\wh{v}^\pi$ to be the value of $\pi$ run on $\wh{M}$. 
\end{enumerate}
Define the law of this process as $\overline{\cQ}$. The sampling process of $\mathsf{TrajectoryTree}$ (call the law of this process $\cQ$) can be viewed as sampling the subset of $\wh{M}^\mathrm{det}$ which is reachable by some $\pi \in \Pi$. Thus, we have 
\begin{align*}
V^\pi = \En_{\wh{M}\sim\overline{\cQ}}\brk*{ \wh{v}^\pi } = \En_{\wh{T}_\sim \cQ} \brk*{ \En \brk*{\wh{v}^\pi \mid{} \wh{T}} } = \En_{\wh{T}_\sim \cQ} \brk*{ \wh{v}^\pi } ,
\end{align*}
where the second equality is due to the law of total probability, and the third equality is due to the fact that $\wh{v}^\pi$ is measurable with respect to the trajectory tree $\wh{T}$. Thus, $\crl{\wh{v}_i^\pi}_{i \in [n]}$ are unbiased estimates of $V^\pi$.} 

% \glcomment{The cumulative reward $\wh{v}^\pi_i$ is an unbiased estimate of $V^\pi$. One can consider an alternative process for the construction of $\bar{M}_t$ as first constructing the path that $\pi$ takes and then filling out the rest of the tree. The only difference between this process and the actual one is the \emph{order} in which the transitions are sampled, so all of the transitions and rewards are still independently sampled from the correct distributions.}

Therefore, by Hoeffding's inequality (\pref{lem:hoeffding}) we see that $\abs{V^\pi - \wh{V}^\pi} \le \sqrt{ \tfrac{\log (2/\delta)}{2n} }$. Applying union bound we see that when the number of trajectory trees exceeds $n \gtrsim \tfrac{\log (\abs{\Pi}/\delta)}{\eps^2}$, with probability at least $1-\delta$, for all $\pi \in \Pi$, the estimates satisfy $\abs{V^\pi - \wh{V}^\pi} \le \eps/2$. Thus the $\mathsf{TrajectoryTree}$ algorithm returns an $\eps$-optimal policy. Since each trajectory tree uses at most $H \cdot \dimRL(\Pi)$ queries to the generative model, we have the claimed sample complexity bound. \hfill \qedsymbol

\subsection{Proof of \pref{thm:generative_lower_bound}}\label{app:proof_generative_lower_bound}
% \ayush{The reduction needs to be made more clear here as it feels a bit information.  You need to define that arms now correspond to running policies, each policy gets mapped to a unique \((s, a)\) and that Agnostic RL can be solved by best-arm identification!}
\gledit{Fix any worst-case deterministic MDP $M^\star$ which witnesses $\dimRL(\Pi)$ at layer $h^\star$. Since $\dimRL(\Pi)$ is a property depending on the dynamics of $M^\star$, we can assume that $M^\star$ has zero rewards. We can also assume that the algorithm knows $M^\star$ and $h^\star$ (this only makes the lower bound stronger). We construct a family of instances $\cM^\star$ where all the MDPs in $\cM^\star$ have the same dynamics as $M^\star$ but different nonzero rewards at the reachable $(s,a)$ pairs at layer $h^\star$.} 

\gledit{Observe that we can embed a multi-armed bandit instance with $\dimRL(\Pi)$ arms using the class $\cM^\star$. The value of any policy $\pi \in \Pi$ is exactly the reward that it receives at the \emph{unique} $(s,a)$ pair in layer $h^\star$ that it reaches. Any $(\eps,\delta)$-PAC algorithm that works over the family of instances $\cM^\star$ must return a policy $\hat{\pi}$ that reaches an $(s,a)$ pair in layer $h^\star$ with near-optimal reward. Furthermore, in the generative model setting, the algorithm can only receive information about a single $(s,a)$ pair. Thus, such a PAC algorithm must also be able to PAC learn the best arm for multi-armed bandits with $\dimRL(\Pi)$ arms. Therefore, we can directly apply existing PAC lower bounds which show that the sample complexity of $(\eps, \delta)$-PAC learning the best arm for $K$-armed multi-armed bandits is at least $\Omega(\tfrac{K}{\eps^2} \cdot \log \tfrac{1}{\delta})$ \citep[see, e.g.,][]{mannor2004sample}.\qed}

\subsection{Proof of \pref{corr:deterministic-mdp}}

The upper bound is obtained by a simple modification of the argument in the proof of \pref{thm:generative_upper_bound}. In terms of data collection, the trajectory tree collected every time is the same fixed deterministic MDP (with different rewards); furthermore, one can always execute \pref{line:trajectory-tree-sample} and \pref{line:generative-query} for a deterministic MDP since the algorithm can execute a sequence of actions to get to any new $(s,a)$ pair required by \pref{line:generative-query}. Thus in every episode of online interaction we are guaranteed to add the new $(s,a)$ pair to the trajectory tree.
% We can improve $\abs{\Pi}$ to $H\cdot \dimRL(\Pi)$ since we observe that there are at most $H \cdot \dimRL(\Pi)$
% \gene{Actually, I'm not sure we can improve the $\abs{\Pi}$ in the log to $H\cdot \dimRL$, since this might cause us to need to estimate each reward up to $\eps/H$ accuracy.}

The lower bound trivially extends because the proof of \pref{app:proof_generative_lower_bound} uses a family of MDPs with deterministic transitions (that are even known to the algorithm beforehand).
 
\clearpage
\section{Proofs for \pref{sec:online-lower-bound}}\label{app:proof-lower-bound-online} 

\gledit{In this section, we prove \pref{thm:lower-bound-online}, which shows a superpolynomial lower bound on the sample complexity required to learn bounded \Compname{} classes, ruling out $\poly(\dimRL(\Pi), H, \log \abs{\Pi})$ sample complexity for online RL.} We restate the theorem below with the precise constants: 

\begin{theorem}[Lower bound for online RL]\label{thm:apdx-lower-bound}
Let $h_0 \in \bbN$ and $c \in (0,1)$ be universal constants. Fix any $H \ge h_0$. Let $\eps \in (1/2^{cH},1/(100H))$ and $\ell \in \crl{2, \dots, H}$ such that $1/\eps^{\ell} \le 2^H$. There exists a policy class $\Piell$ of size $1/(6\eps^\ell)$ with $\dimRL(\Piell) \le O(H^{4\ell+2})$ and a family of MDPs $\cM$ with state space $\cS$ of size $H\cdot 2^{2H+1}$, binary action space, horizon $H$ such that: for any $(\eps/16, 1/8)$-PAC algorithm, there exists an $M \in \cM$ in which the algorithm has to collect at least
\begin{align*}
    \min\crl*{ \frac{1}{120 \eps^\ell}, 2^{H/3 - 3} } \quad \text{online trajectories in expectation.}
\end{align*} 
\end{theorem}

\subsection{Construction of State Space, Action Space, and Policy Class}

\paragraph{State and Action Spaces.} We define the state space $\cS$. In every layer $h \in [H]$, there will be $2^{2H+1}$ states. The states will be paired up, and each state will be denoted by either $\jof{h}$ or $\jpof{h}$, so $\cS_h = \crl{\jof{h}: j \in [2^{2H}]} \cup \crl{\jpof{h}: j \in [2^{2H}]}$. For any state $s \in \cS$, we define the \emph{index} of $s$, denoted $\idx(s)$ as the unique $j\in [2^{2H}]$ such that $s \in \crl{\jof{h}}_{h\in[H]} \cup \crl{\jpof{h}}_{h\in[H]}$. In total there are $H \cdot 2^{2H+1}$ states. The action space is $\cA = \crl{0,1}$.

\paragraph{Policy Class.} For the given $\eps$ and $\ell \in \crl{2, \dots, H}$, we show via a probabilistic argument the existence of a large policy class $\Piell$ which has bounded \Compname{} but is hard to explore. We state several properties in \pref{lem:piell-properties} which will be exploited in the lower bound.

We introduce some additional notation. For any $j \in [2^{2H}]$ we denote
\begin{align*}
    \Piell_j \coloneqq \crl{\pi \in \Piell: \exists h \in [H], \pi(\jof{h}) = 1},
\end{align*}
that is, $\Piell_j$ are the policies which take an action $a=1$ on at least one state with index $j$.

We also define the set of \emph{relevant state indices} for a given policy $\pi \in \Piell$ as 
\begin{align*}
    \Jrel^\pi \coloneqq \crl{j \in [2^{2H}]: \pi \in \Piell_j}.
\end{align*}
For any policy $\pi$ we denote $\pi(j_{1:H}) \coloneqq (\pi(\jof{1}), \dots, \pi(\jof{H})) \in \crl{0,1}^H$ to be the vector that represents the actions that $\pi$ takes on the states  in index $j$. The vector $\pi(j'_{1:H})$ is defined similarly.

\begin{lemma}\label{lem:piell-properties} Let $H$, $\eps$, and $\ell$ satisfy the assumptions of \pref{thm:apdx-lower-bound}. There exists a policy class $\Piell$ of size $N = 1/(6 \eps^\ell)$ which satisfies the following properties. 
\neurIPS{
\vspace{-\topsep}
\begin{itemize}[leftmargin=7mm, itemsep=0.2mm]}
\arxiv{\begin{itemize}}
    \item[(1)] For every $j \in [2^{2H}]$ we have $\abs{\Piell_j} \in  [\eps N/2, 2\eps N]$. %\gene{it seems like we don't need the lower bound property actually! We only need it in an ``averaged'' sense, which is implied by the fact that $\Piell$ is large. Of course, we get it for free by concentration.} \ayush{We can let it be as it is; does not hurt!} 
    \item[(2)] For every $\pi \in \Pi$ we have $\abs{ \Jrel^\pi } \ge \eps/2 \cdot 2^{2H}$.
    \item[(3)] For every $\pi \in \Piell_j$, the vector $\pi(j_{1:H})$ is unique and always equal to $\pi(j'_{1:H})$.
    \item[(4)] Bounded \Compname{}: $\dimRL(\Piell) \le c \cdot H^{4\ell+2}$ for some universal constant $c > 0$.
\end{itemize}
\end{lemma} 

\subsection{Construction of MDP Family}
The family $\cM = \crl{M_{\pistar,\phi}}_{\pistar \in \Piell, \phi \in \Phi}$ will be a family of MDPs which are indexed by a policy $\pistar$ as well as a \emph{decoder} function $\phi: \cS \mapsto \crl{\textsc{good}, \textsc{bad}}$, which assigns each state to be ``good'' or ``bad'' in a sense that will be described later on. An example construction of an MDP $M_{\pistar, \phi}$ is illustrated in \pref{fig:lower-bound-idea}. For brevity, the bracket notation used to denote the layer that each state lies in has been omitted in the figure.

\paragraph{Decoder Function Class.} The decoder function class $\Phi$ will be the set of all possible mappings which for every $j\in [2^{2H}]$ and $h \ge 2$ assign exactly one of $\jof{h}$ or $ \jpof{h}$ to the label $\textsc{Good}$ (where the other is assigned to the label $\textsc{Bad}$). There are $(2^{H-1})^{2^{2H}}$ such functions. The label of a state will be used to describe the transition dynamics. Intuitively, a learner who does not know the decoder function $\phi$ will not be able to tell if a certain state has the label  $\textsc{Good}$ or $\textsc{Bad}$ when visiting that state for the first time.

\paragraph{Transition Dynamics.} The MDP $M_{\pistar, \phi}$ will be a uniform distribution over $2^{2H}$ combination locks $\crl{\mathsf{CL}_j}_{j \in [2^{2H}]}$ with disjoint states. More formally, $s_1 \sim \unif(\{\jof{1}\}_{j\in[2^{2H}]})$. From each start state $\jof{1}$, only the $2H-2$ states corresponding to index $j$ at layers $h\ge 2$ will be reachable in the combination lock $\mathsf{CL}_j$.

In the following, we will describe each combination lock $\mathsf{CL}_j$, which forms the basic building block of the MDP construction.
\neurIPS{
\vspace{-\topsep}
\begin{itemize}[label=\(\bullet\), leftmargin=5mm, itemsep=0.2mm]}
\arxiv{\begin{itemize}[label=\(\bullet\)]}
    \item \textbf{Good/Bad Set.} At every layer $h\in [H]$, for each $\jof{h}$ and $\jpof{h}$, the decoder function $\phi$ assigns one of them to be $\textsc{Good}$ and one of them to be $\textsc{Bad}$. We will henceforth denote $\jgof{h}$ to be the good state and $\jbof{h}$ to be the bad state. Observe that by construction in Eq.~\eqref{eq:pi-construction}, for every $\pi \in \Piell$ and $h\in [H]$ we have $\pi(\jgof{h}) = \pi(\jbof{h})$. 
    \item \textbf{Dynamics of $\mathsf{CL}_j$, if $j\in \Jrel^\pistar$.} Here, the transition dynamics of the combination locks are deterministic. For every $h\in [H]$,
    \begin{itemize}
        \item On good states $\jgof{h}$ we transit to the next good state iff the action is $\pistar$: \begin{align*}
            P(s'~|~\jgof{h}, a) &= \begin{cases}
            \ind{s' = \jgof{h+1}}, &\text{if } a=\pistar(\jgof{h}) \\
            \ind{s' = \jbof{h+1}}, &\text{if } a\ne\pistar(\jgof{h}).
            \end{cases}
        \end{align*}
        \item On bad states $\jbof{h}$ we always transit to the next bad state: 
        \begin{align*}
            P(s'~|~\jbof{h}, a) = \ind{s' = \jbof{h+1}}, \quad \text{for all } a \in \cA.
        \end{align*}
    \end{itemize}
    \item \textbf{Dynamics of $\mathsf{CL}_j$, if $j\notin \Jrel^\pistar$.} If $j$ is not a relevant index for $\pistar$, then the transitions are uniformly random regardless of the current state/action. For every $h\in [H]$,
    \begin{align*}
    P(\cdot ~|~ \jgof{h},a) = P(\cdot ~|~ \jbof{h}, a) = \unif \prn*{ \crl{\jgof{h+1}, \jbof{h+1}} }, \quad \text{for all } a \in \cA.
\end{align*}
    \item \textbf{Reward Structure.} The reward function is nonzero only at layer $H$, and is defined as 
    \begin{align*}
        R(s,a) = \mathrm{Ber}\prn*{ \frac{1}{2} + \frac{1}{4} \cdot \mathbbm{1}\crl{\pistar \in \Piell_j} \cdot \mathbbm{1}\crl{s = \jgof{H}, a = \pistar(\jgof{H})} }
    \end{align*}
    That is, we get $3/4$ whenever we reach the $H$-th good state for an index $j$ which is relevant for $\pistar$, and $1/2$ reward otherwise.
\end{itemize}

\paragraph{Reference MDPs.} We define several reference MDPs.
\begin{itemize}
    \item In the reference MDP $M_0$, the initial start state is again taken to be the uniform distribution, i.e., $s_1 \sim \unif(\{\jof{1}\}_{j\in[2^{2H}]})$, and all the combination locks behave the same and have uniform transitions to the next state along the chain: for every $h\in [H]$ and $j \in [2^{2H}]$,
    \begin{align*}
    P(\cdot ~|~ \jof{h},a) = P(\cdot ~|~ \jpof{h}, a) = \unif \prn*{ \crl{\jof{h+1}, \jpof{h+1}} }, \quad \text{for all } a \in \cA.
\end{align*}
    The rewards for $M_0$ are $\mathrm{Ber}(1/2)$ for every $(s,a) \in \cS_H\times \cA$. 
    \item For any decoder $\phi \in \Phi$, the reference MDP $M_{0, \pistar, \phi}$ has the same transitions as $M_{\pistar, \phi}$ but the rewards are $\mathrm{Ber}(1/2)$ for every $(s,a) \in \cS_H\times \cA$.
\end{itemize}

\subsection{Proof of \pref{thm:apdx-lower-bound}}
We are now ready to prove the lower bound using the construction of the MDP family $\cM$. 

\paragraph{Value Calculation.} Consider any $M_{\pistar,\phi} \in \cM$. For any policy $\pi \in \cA^{\cS}$ we use $V_{\pistar,\phi}(\pi)$ to denote the value of running $\pi$ in MDP $M_{\pistar,\phi}$. By construction we can see that 
\begin{align*}
    V_{\pistar,\phi}(\pi) = \frac{1}{2} + \frac{1}{4} \cdot \Pr_{\pistar,\phi}\brk*{ \idx(s_1) \in \Jrel^\pistar \text{ and } \pi(\idx(s_1)_{1:H}) = \pistar(\idx(s_1)_{1:H}) },  \numberthis\label{eq:value-of-pi} 
\end{align*}
where in the above, we defined for any \(s_1\), \(\pi(\idx(s_1)_{1:H})\ = \pi(j_{1:H}) = (\pi(j[1]), \dots, \pi(j[H]))\), where \(j\) denotes \(\idx(s_1)\). Informally speaking, the second term counts the additional reward that $\pi$ gets for solving a combination lock rooted at a relevant state index $\idx(s_1) \in \Jrel^\pistar$. By Property (2) and (3) of \pref{lem:piell-properties}, we additionally have $V_{\pistar,\phi}(\pistar) \ge 1/2 + \eps/8$, as well as $V_{\pistar,\phi}(\pi) = 1/2$ for all other $\pi\ne \pistar\in \Piell$. %\gene{not sure we need the second statement.}

By Eq.~\eqref{eq:value-of-pi}, if $\pi$ is an $\eps/16$-optimal policy on $M_{\pistar,\phi}$ it must satisfy
\begin{align*}
    \Pr_{\pistar,\phi} \brk*{ \idx(s_1) \in \Jrel^\pistar \text{ and } \pi(\idx(s_1)_{1:H}) = \pistar(\idx(s_1)_{1:H}) } \ge \frac{\eps}{4}.
\end{align*}

\paragraph{Averaged Measures.} We define the following measures which will be used in the analysis. 
\begin{itemize}
    \item Define $\Pr_\pistar[\cdot] = \frac{1}{\abs{\Phi}} \sum_{\phi \in \Phi} \Pr_{\pistar, \phi}[\cdot]$ to be the averaged measure where we first pick $\phi$ uniformly among all decoders and then consider the distribution induced by $M_{\pistar, \phi}$.
    \item Define the averaged measure $\Pr_{0, \pistar}[\cdot] = \frac{1}{\abs{\Phi}} \sum_{\phi \in \Phi} \Pr_{0, \pistar, \phi}[\cdot]$ where we pick $\phi$ uniformly and then consider the distribution induced by $M_{0, \pistar, \phi}$.
\end{itemize}
For both averaged measures the expectations $\En_{\pistar}$ and $\En_{0, \pistar}$ are defined analogously. 

\paragraph{Algorithm and Stopping Time.}
Recall that an algorithm $\alg$ is comprised of two phases. In the first phase, it collects some number of trajectories by interacting with the MDP in episodes. We use $\st$ to denote the (random) number of episodes after which $\alg$ terminates. We also use $\alg_t$ to denote the intermediate policy that the algorithm runs in round $t$ for $t \in [\st]$. In the second phase, $\alg$ outputs\footnote{We present the lower bound for the class of deterministic algorithms that output a deterministic policy. However, all the arguments could be extended to stochastic algorithms.} a policy $\pihat$. We use the notation $\alg_f: \crl{\tau^{(t)}}_{t \in [\st]} \mapsto \cA^{\cS}$  to denote the second phase of $\alg$ which outputs $\pihat$ as a measurable function of collected data. 

For any policy $\pistar$, decoder $\phi$, and dataset $\cD$ we define the event
\begin{align*}
    \cE(\pistar, \phi, \alg_f(\cD)) := \crl*{ \Pr_{\pistar,\phi}\brk*{ \idx(s_1) \in \Jrel^\pistar~ \text{and}~ \alg_f(\cD)(\idx(s_1)_{1:H}) = \pistar(\idx(s_1)_{1:H}) } \ge \frac{\eps}{4} }.
\end{align*}
The event $\cE(\pistar, \phi, \alg_f(\cD))$ is measurable with respect to the random variable $\cD$, which denotes the collected data.
% The randomness in $\cE(\pistar, \phi, \alg_f(\cD))$ is due to randomness in $\cD$, which is the data collection process of $\alg$. Note that the event $\cE$ is well defined for $\cD$ that is collected on \emph{any} MDP, not just $M_{\pistar,\phi}$.\ayush{This sentence is not clear!} 

Under this notation, the PAC learning guarantee on $\alg$ implies that for every $\pistar \in \Piell$, $\phi \in \Phi$ we have
\begin{align*}
    \Pr_{\pistar,\phi} \brk*{ \cE(\pistar, \phi, \alg_f(\cD)) } \ge 7/8.
\end{align*}
Moreover via an averaging argument we also have
\begin{align*}
    \Pr_{\pistar} \brk*{ \cE(\pistar, \phi, \alg_f(\cD)) } \ge 7/8. \numberthis \label{eq:pac-guarantee-exp} 
\end{align*}
\paragraph{Lower Bound Argument.} We apply a truncation to the stopping time $\st$. Define $\Tmax := 2^{H/3}$. Observe that if $\Pr_\pistar \brk{\st > \Tmax} > 1/8$ for some $\pistar \in \Piell$ then the lower bound immediately follows, since
\begin{align*}
    \max_{\phi \in \Phi} \En_{\pistar, \phi}[\st] ~>~ \En_{\pistar}[\st] ~\ge~ \Pr_{\pistar}[\st > \Tmax] \cdot \Tmax ~\ge~ \Tmax/8,
\end{align*}
so there must exist an MDP $M_{\pistar, \phi}$ for which $\alg$ collects at least $\Tmax/8 = 2^{H/3-3}$ samples in expectation.

Otherwise we have $\Pr_\pistar \brk{\st > \Tmax} \le 1/8$ for all $\pistar \in \Piell$. This further implies that for all $\pistar \in \Piell$,
\begin{align*}
    &\Pr_\pistar \brk*{\st < \Tmax \text{ and } \cE(\pistar, \phi, \alg_f(\cD))} \\
    &= \Pr_\pistar \brk*{\cE(\pistar, \phi, \alg_f(\cD))} - \Pr_\pistar \brk*{\st > \Tmax \text{ and } \cE(\pistar, \phi, \alg_f(\cD))} \ge 3/4. \numberthis \label{eq:lower_bound_property}
\end{align*}

%In this second case, we will show that $\alg$ requires a lot of samples on $M_0$. This is formalized in the following lemma.
However, in the following, we will argue that if Eq.~\eqref{eq:lower_bound_property} holds then \(\alg\) must query a significant number of samples in \(M_0\).

\begin{lemma}[Stopping Time Lemma]\label{lem:stopping-time} 
Let $\delta \in (0, 1/8]$. Let $\alg$ be an $(\eps/16,\delta)$-PAC  algorithm. Let $\Tmax \in \bbN$.  Suppose that $\Pr_\pistar \brk*{\st < \Tmax \text{ and } \cE(\pistar, \phi, \alg_f(\cD))} \ge 1-2\delta$ for all $\pistar \in \Piell$. The expected stopping time for $\alg$ on $M_0$ is at least
\begin{align*}
    \En_0 \brk*{\st} \ge \prn*{ \frac{\abs{\Piell}}{2} - \frac{4}{\eps}} \cdot \frac{1}{7} \log \prn*{\frac{1}{ 2 \delta }} - \abs{\Piell} \cdot  \frac{\Tmax^2}{2^{H+3}} \prn*{\Tmax + \frac{1}{7}\log \prn*{\frac{1}{ 2 \delta }}}.
\end{align*}
\end{lemma}

Using \pref{lem:stopping-time} with $\delta = 1/8$ and plugging in the value of $\abs{\Piell}$ and $\Tmax$, we see that
\begin{align*}
    \En_0[\st] &\ge  \prn*{ \frac{\abs{\Piell}}{2} - \frac{4}{\eps}} \cdot \frac{1}{7} \log \prn*{\frac{1}{ 2 \delta }} - \abs{\Piell} \cdot \frac{\Tmax^2}{2^{H+3}}  \prn*{\Tmax + \frac{1}{7}\log \prn*{\frac{1}{ 2 \delta }}} \ge \frac{\abs{\Piell}}{20}.
\end{align*}
For the second inequality, we used the fact that $\ell \ge 2$, $H \ge 10^5$, and $\eps < 1/10^7$. 

We have shown that either there exists some MDP $M_{\pistar, \phi}$ for which $\alg$ collects at least $\Tmax/8 = 2^{H/3-3}$ samples in expectation, or $\alg$ must query at least $\abs{\Piell}/20 = 1/(120 \eps^\ell)$ trajectories in expectation in \(M_0\). Putting it all together, the lower bound on the sample complexity is at least
\begin{align*}
    \min\crl*{ \frac{1}{120 \eps^\ell}, 2^{H/3 - 3} }.
\end{align*}
This concludes the proof of \pref{thm:apdx-lower-bound}.\qed

\subsection{Proof of \pref{lem:piell-properties}}\label{app:construction-policy-class}

To prove \pref{lem:piell-properties}, we first use a probabilistic argument to construct a certain binary matrix $B$ which satisfies several properties, and then construct $\Piell$ using $B$ and verify it satisfies Properties (1)-(4). 

\paragraph{Binary Matrix Construction.} 
First, we define a block-free property of binary matrices.
\begin{definition}[Block-free Matrices]\label{def:intersection-property} 
Fix parameters $k, \ell, N, d \in \bbN$ where \(k \leq N\) and \(l \leq d\). We say a binary matrix $B \in \crl{0,1}^{N \times d}$ is $(k, \ell)$-block-free if the following holds: for every $I \subseteq [N]$ with $\abs{I} = k$, and $J \subseteq [d]$ with $\abs{J} = \ell$ there exists some $(i,j)\in I\times J$ with $B_{ij} = 0$.
\end{definition}

In words, matrices which are $(k,\ell)$-block-free do not contain a $k\times \ell$ block of all 1s.

\begin{lemma}\label{lem:matrix-construction}
Fix any $\eps \in (0,1/10)$ and $\ell \in \bbN$. For any 
\begin{align*}
    d \in \Big[ \frac{16 \ell \cdot \log(1/\eps)}{\eps}, \frac{1}{20} \cdot \exp\Big( \frac{1}{48 \eps^{\ell-1}} \Big) \Big] ,
\end{align*}
there exists a binary matrix $B \in \crl{0,1}^{N\times d}$ with $N = 1/(6 \cdot \eps^\ell)$ such that:
\neurIPS{
\vspace{-\topsep}
\begin{enumerate}[leftmargin=7mm, itemsep=0.2mm]}
\arxiv{\begin{enumerate}[label=\((\arabic*)\)]} 
    \item (Row sum): for every row $i \in [N]$, we have $\sum_{j} B_{ij} \ge \eps d /2$.
    \item (Column sum): for every column $j \in [d]$, we have $\sum_{i} B_{ij} \in [\eps N /2, 2\eps N]$.
    \item The matrix $B$ is $(\ell \log d,\ell)$-block-free.
\end{enumerate}
\end{lemma}

\begin{proof}[Proof of \pref{lem:matrix-construction}.]
The existence of $B$ is proven using the probabilistic method. Let $\wt{B} \in \crl{0,1}^{N\times d}$ be a random matrix where each entry is i.i.d.~chosen to be 1 with probability $\eps$. 

By Chernoff bounds (\pref{lem:chernoff}), for every row $i \in [N]$, we have $\Pr \brk{\sum_{j} B_{ij} \le \tfrac{\eps d}{2}} \le \exp\prn{-\eps d /8}$; likewise for every column $j \in [d]$, we have $\Pr \brk{\sum_{j} B_{ij} \notin \brk{\tfrac{\eps N}{2}, 2\eps N}} \le 2\exp\prn{-\eps N /8}$. By union bound, the matrix $\wt{B}$ satisfies the first two properties with probability at least $0.8$ as long as 
\begin{align*}
    d \ge (8 \log 10N)/\eps, \quad \text{and} \quad N \ge (8 \log 20d)/\eps.
\end{align*}
One can check that under the choice of $N = 1/(6 \cdot \eps^\ell)$ and the assumption on $d$, both constraints are met.

Now we examine the probability of $\wt{B}$ satisfies the block-free property with parameters $(k \coloneqq \ell \log d, \ell)$. Let $X$ be the random variable which denotes the number of submatrices which violate the block-free property in $\wt{B}$, i.e., 
\begin{align*}
    X = \abs{ \crl{I \times J: I \subset [N], \abs{I} = k, J \subset [d], \abs{J} = \ell, \wt{B}_{ij} = 1 \ \forall \ (i,j)\in I \times J} }. 
\end{align*}
By linearity of expectation, we have
\begin{align*}
    \bbE \brk{X} \le N^{k} d^\ell \eps^{k \ell}.
\end{align*}
We now plug in the choice $k = \ell \log d$ and observe that as long as $N \le 1/(2e \cdot \eps^\ell)$ we have $\bbE[X] \le 1/2$. By Markov's inequality, $\Pr [X = 0] \ge 1/2$.

Therefore with positive probability, $\wt{B}$ satisfies all 3 properties (otherwise we would have a contradiction via inclusion-exlusion principle). Thus, there exists a matrix $B$ which satisfies all of the above three properties, proving the result of \pref{lem:matrix-construction}.
\end{proof}

\paragraph{Policy Class Construction.}
For the given $\eps$ and $\ell \in \crl{2, \dots, H}$ we will use \pref{lem:matrix-construction} to construct a policy class $\Piell$ which has bounded \Compname{} but is hard to explore. We instantiate \pref{lem:matrix-construction} with the given $\ell$ and $d = 2^{2H}$, and use the resulting matrix $B$ to construct $\Piell = \crl{\pi_i}_{i\in [N]}$ with $\abs{\Piell} = N = 1/(6\eps^\ell)$. 

Recall that we assume that
\begin{align*}
    H \ge h_0, \quad \text{and} \quad \eps \in \brk*{ \frac{1}{2^{c H}} , \frac{1}{100H} }.
\end{align*}
We claim that under these assumptions, the requirement of \pref{lem:matrix-construction} is met:
\begin{align*}
    d = 2^{2H} \in \brk*{ \frac{16 \ell \cdot \log(1/\eps)}{\eps}, \frac{1}{20} \cdot \exp\prn*{ \frac{1}{48 \eps^{\ell-1}} } }.
\end{align*}
For the lower bound, we can check that:
\begin{align*}
    \frac{16 \ell \cdot \log(1/\eps)}{\eps} \le 16 H \cdot cH \cdot 2^{cH} \le 2^{2H},
\end{align*}
where we use the bound $\ell \le H$ and $\eps \ge 2^{-cH}$. The last inequality holds for sufficiently small universal constant $c \in (0, 1)$ and sufficiently large $H \ge h_0$.

For the upper bound, we can also check that 
\begin{align*}
    \frac{1}{20} \cdot \exp\prn*{ \frac{1}{48 \eps^{\ell-1}} } \ge \frac{1}{20} \cdot \exp\prn*{ \frac{100H}{48}} \ge 2^{2H},
\end{align*} 
where we use the bound $\ell \ge 2$ and $\eps \le 1/(100H)$. The last inequality holds for sufficiently large $H$. 
We define the policies as follows: for every $\pi_i \in \Piell$ we set\ayush{I think we should add a picture here converting a matrix to a policy class. I feel the following is a bit hard to parse in the first attempt. @Sasha, what do you think?}  
\begin{align*}
    \text{for every } j \in [2^{2H}]: \quad&\pi_i(\jof{h}) = \pi_i(\jpof{h}) = \begin{cases}
    \bit_h(\sum_{a \le i} B_{aj}) &\text{if} \ B_{ij} =1,\\
    0 & \text{if} \ B_{ij} = 0.
    \end{cases} \numberthis\label{eq:pi-construction}
\end{align*}
The function $\bit_h: [2^H-1] \mapsto \crl{0,1}$ selects the $h$-th bit in the binary representation of the input.

\paragraph{Verifying Properties $(1)-(4)$ of \pref{lem:piell-properties}.} 
Properties $(1)-(3)$ are straightforward from the construction of $B$ and $\Piell$, since $\pi_i \in \Piell_j$ if and only if $B_{ij} = 1$. The only detail which requires some care is that we require that $2\eps N < 2^{H}$ in order for Property (3) to hold, since otherwise we cannot assign the behaviors of the policies according to Eq.~\eqref{eq:pi-construction}. However, by assumption, this always holds, since $2\eps N = 1/(3\eps^{\ell-1}) \le 2^H.$ 

We now prove Property (4) that $\Piell$ has bounded \Compname{}. To prove this we will use the block-free property of the underlying binary matrix $B$.

Fix any deterministic MDP $M^\star$ which witnesses $\dimRL(\Piell)$ at layer $h^\star$. To bound $\dimRL(\Piell)$, we need to count the contribution to $C^\mathsf{reach}_{h^\star}(\Pi; M^\star)$ from trajectories $\tau$ which are produced by some $\pi \in \Piell$ on $M$. We first define a \emph{layer decomposition} for a trajectory $\tau = (s_1, a_1, s_2, a_2, \dots, s_H, a_H)$ as the unique tuple of indices $(h_1, h_2, \dots h_{m})$, where each $h_k \in [H]$, that satisfies the following properties: 
\neurIPS{
\vspace{-\topsep}
\begin{itemize}[label=\(\bullet\), leftmargin=5mm, itemsep=0.2mm]}
\arxiv{\begin{itemize}[label=\(\bullet\)]}
    \item The layers satisfy $h_1 < h_2 < \dots < h_m$.
    \item The layer $h_1$ represents the first layer where $a_{h_1} = 1$.
    \item The layer $h_2$ represents the first layer where $a_{h_2} = 1$ on some state $s_{h_2}$ such that
    \begin{align*}
        \idx(s_{h_2}) \notin \crl{ \idx(s_{h_1}) }.
    \end{align*}
    \item The layer $h_3$ represents the first layer where $a_{h_3} = 1$ on some state $s_{h_3}$ such that 
    \begin{align*}
        \idx(s_{h_3}) \notin \crl{ \idx(s_{h_1}), \idx(s_{h_2}) }.
    \end{align*}
    \item More generally the layer $h_k$, $k\in[m]$ represents the first layer where $a_{h_k} = 1$ on some state $s_{h_k}$ such that 
    \begin{align*}
         \idx(s_{h_k}) \notin \crl{ \idx(s_{h_1}), \dots, \idx(s_{h_{k-1}}) }.
    \end{align*}
    In other words, the layer $h_k$ represents the $k$-th layer for where action is $a=1$ on a new state index which $\tau$ has never played $a=1$ on before.
\end{itemize}
We will count the contribution to $C^\mathsf{reach}_{h^\star}(\Pi; M^\star)$ by doing casework on the length of the layer decomposition for any $\tau$. That is, for every length $m \in \crl{0, \dots, H}$, we will bound $C_{h^\star}(m)$, which is defined to be the total number of $(s,a)$ at layer $h^\star$ which, for some $\pi \in \Piell$, a trajectory $\pi \cons \tau$ that has a $m$-length layer decomposition visits. Then we apply the bound 
\begin{align*}
    C^\mathsf{reach}_{h^\star}(\Pi; M^\star) \le \sum_{m=0}^H C_{h^\star}(m). \numberthis\label{eq:contribution-decomp}
\end{align*}
Note that this will overcount, since the same $(s,a)$ pair can belong to multiple different trajectories with different length layer decompositions.

\begin{lemma}\label{lem:contributions}
The following bounds hold: 
\neurIPS{
\vspace{-\topsep}
\begin{itemize}[label=\(\bullet\), leftmargin=5mm, itemsep=0.2mm]}
\arxiv{\begin{itemize}[label=\(\bullet\)]}
    \item For any $m \le \ell$, $C_{h^\star}(m) \le H^m \cdot \prod_{k=1}^m (2kH) = \cO(H^{4m})$.
    \item We have $\sum_{m \ge \ell+1} C_{h^\star}(m) \le \cO(\ell \cdot H^{4\ell + 1})$.
\end{itemize}    
\end{lemma}
Therefore, applying \pref{lem:contributions} to Eq.~\eqref{eq:contribution-decomp}, we have the bound that
\begin{align*}
    \dimRL(\Piell) \le \prn*{\sum_{m \le \ell} O(H^{4m})} + O(\ell \cdot H^{4\ell+1}) \le O(H^{4\ell+2}).
\end{align*}
This concludes the proof of \pref{lem:piell-properties}.\qed

\begin{proof}[Proof of \pref{lem:contributions}]
All of our upper bounds will be monotone in the value of $h^\star$, so we will prove the bounds for $C_H(m)$. In the following, fix any deterministic MDP $M^\star$.

First we start with the case where $m=0$. The trajectory $\tau$ must play $a=0$ at all times; since there is only one such $\tau$, we have $C_{H}(0) = 1$.

Now we will bound $C_H(m)$, for any $m \in \crl{1, \dots, \ell}$. Observe that there are ${H\choose m} \le H^m$ ways to pick the tuple $(h_1, \dots, h_m)$. Now we will fix $(h_1, \dots, h_m)$ and count the contributions to $C_H(m)$ for trajectories $\tau$ which have this fixed layer decomposition, and then sum up over all possible choices of $(h_1, \dots, h_m)$.

In the MDP $M^\star$, there is a unique state $s_{h_1}$ which $\tau$ must visit. In the layers between $h_1$ and $h_2$, all trajectories are only allowed take $1$ on states with index $\idx(s_{h_1})$, but they are not required to. Thus we can compute that the contribution to $C_{h_2}(m)$ from trajectories with the fixed layer decomposition to be at most $2H$. The reasoning is as follows. At $h_1$, there is exactly one $(s,a)$ pair which is reachable by trajectories with this fixed layer decomposition, since any $\tau$ must take $a=1$ at $s_{h_1}$. Subsequently we can add at most two reachable pairs in every layer $h \in \{h_1+1, \dots, h_2-1\}$ due to encountering a state $\jof{h}$ or $\jpof{h}$ where $j = \idx(s_{h_1})$, and at layer $h_2$ we must play $a=1$, for a total of $1 + 2(h_2 - h_1 - 1) \le 2H$. Using similar reasoning the contribution to $C_{h_3}(m)$ from trajectories with this fixed layer decomposition is at most $(2H) \cdot (4H)$, and so on. Continuing in this way, we have the final bound of $\prod_{k=1}^m (2kH)$. Since this holds for a fixed choice of $(h_1, \dots, h_m)$ in total we have $C_H(m) \le H^m \cdot \prod_{k=1}^m (2kH) = \cO(H^{4m})$. 

When $m \ge \ell + 1$, observe that the block-free property on $B$ implies that for any $J \subseteq [2^H]$ with $\abs{J} = \ell$ we have $\abs{\cap_{j \in J} \Pi_j} \le \ell \log 2^{2H}$. So for any trajectory $\tau$ with layer decomposition such that $m \ge \ell$ we can redo the previous analysis and argue that there is at most $ \ell \log 2^{2H}$ multiplicative factor contribution to the value $C_H(m)$ due to \emph{all} trajectories which have layer decompositions longer than $\ell$. Thus we arrive at the bound $\sum_{m \ge \ell+1} C_{H}(m) \le \cO(H^{4\ell}) \cdot \ell \log 2^{2H} \le \cO(\ell \cdot H^{4\ell+1})$.
\end{proof}

\subsection{Proof of \pref{lem:stopping-time}}
The proof of this stopping time lemma follows standard machinery for PAC lower bounds \citep{garivier2019explore, domingues2021episodic, sekhari2021agnostic}. In the following we use $\KL{P}{Q}$ to denote the Kullback-Leibler divergence between two distributions $P$ and $Q$ and $\kl{p}{q} $ to denote the Kullback-Leibler divergence between two Bernoulli distributions with parameters $p,q \in [0,1]$.

For any $\pistar \in \Piell$ we denote the random variable
\begin{align*}
N^\pistar = \sum_{t=1}^{\st \wedge \Tmax} \ind{\alg_t(\idx(s_1)_{1:H}) = \pistar(\idx(s_1)_{1:H}) \text{ and } \idx(s_1) \in \Jrel^\pistar},
\end{align*}
the number of episodes for which the algorithm's policy at round $t \in [\st \wedge \Tmax]$ matches that of $\pistar$ on a certain relevant state of $\pistar$. 

In the sequel we will prove upper and lower bounds on the intermediate quantity $\sum_{\pistar \in \Pi} \En_0 \brk*{N^\pistar}$ and relate these quantities to $\En_0[\st]$.

\paragraph{Step 1: Upper Bound.} First we prove an upper bound. We can compute that
\begin{align*}
&\sum_{\pistar \in \Pi} \En_0 \brk*{N^\pistar} \\
&= \sum_{t=1}^\Tmax \sum_{\pistar \in \Pi} \En_0 \brk*{\ind{\st > t - 1} \ind{\alg_t(\idx(s_1)_{1:H}) = \pistar(\idx(s_1)_{1:H}) \text{ and } \idx(s_1) \in \Jrel^\pistar}} \\ 
&= \sum_{t=1}^\Tmax \En_0 \brk*{\ind{\st > t - 1} \sum_{\pistar \in \Pi} 
\ind{\alg_t(\idx(s_1)_{1:H}) = \pistar(\idx(s_1)_{1:H}) \text{ and } \idx(s_1) \in \Jrel^\pistar}}  \\
&\overleq{(i)} \sum_{t=1}^\Tmax \En_0 \brk*{\ind{\st > t - 1} } \leq \En_0 \brk*{\st \wedge \Tmax} \leq \En_0 \brk*{\st}. \numberthis \label{eq:sum-upper-bound}  
\end{align*}
Here, the first inequality follows because for every  index $j$ and every $\pistar \in \Piell_j$, each $\pistar$ admits a unique sequence of actions (by Property (3) of \pref{lem:piell-properties}), so any policy $\alg_t$ can completely match with at most one of the $\pistar$.

\paragraph{Step 2: Lower Bound.} Now we turn to the lower bound. We use a change of measure argument. 
\begin{align*} 
\En_0 \brk*{N^\pistar} &\overgeq{(i)} \En_{0, \pistar} \brk*{N^\pistar}  - \Tmax \Delta(\Tmax) \\ 
&= \frac{1}{\abs{\Phi}} \sum_{\phi \in \Phi} \En_{0, \pistar, \phi} \brk*{N^\pistar}  - \Tmax \Delta(\Tmax) \\
&\overgeq{(ii)} \frac{1}{7}\cdot \frac{1}{\abs{\Phi}} \sum_{\phi \in \Phi}  \KL{\Pr_{0,\pistar, \phi}^{\cF_{\st \wedge \Tmax}}}{\Pr_{\pistar, \phi}^{\cF_{\st \wedge \Tmax}}} - \Tmax \Delta(\Tmax) \\
&\overgeq{(iii)} \frac{1}{7} \cdot \KL{\Pr_{0,\pistar}^{\cF_{\st \wedge \Tmax}}}{\Pr_{\pistar}^{\cF_{\st \wedge \Tmax}}} - \Tmax \Delta(\Tmax)
\end{align*}
The inequality $(i)$ follows from a change of measure argument using \pref{lem:change_measure}, with $\Delta(\Tmax) \coloneqq \Tmax^2/2^{H+3}$. Here, $\cF_{\st\wedge\Tmax}$ denotes the natural filtration generated by the first $\st\wedge \Tmax$ episodes. The inequality $(ii)$ follows from \pref{lem:kl-calculation}, using the fact that $M_{0, \pistar, \phi}$ and $M_{\pistar, \phi}$ have identical transitions and only differ in rewards at layer $H$ for the trajectories which reach the end of a relevant combination lock. The number of times this occurs is exactly $N^\pistar$. The factor $1/7$ is a lower bound on $\kl{1/2}{3/4}$. The inequality $(iii)$ follows by the convexity of KL divergence.

Now we apply \pref{lem:KL-to-kl} to lower bound the expectation for any $\cF_{\st\wedge \Tmax}$-measurable random variable $Z \in [0,1]$ as 
\begin{align*} 
\En_0 \brk*{N^\pistar} &\geq \frac{1}{7} \cdot \kl{\En_{0,\pistar} \brk*{Z}}{\En_{\pistar} \brk*{Z}} - \Tmax \Delta(\Tmax)\\
&\geq \frac{1}{7} \cdot (1- \En_{0,\pistar} \brk*{Z}) \log \prn*{\frac{1}{1-\En_{\pistar} \brk*{Z}}} - \frac{\log(2)}{7} - \Tmax \Delta(\Tmax),
\end{align*}
where the second inequality follows from the bound $\kl{p}{q} \ge (1-p)\log(1/(1-q)) - \log(2)$ \citep[see, e.g.,][Lemma 15]{domingues2021episodic}.

Now we pick $Z = Z_\pistar \coloneqq \ind{ \st < \Tmax \text{ and } \cE(\pistar, \phi, \alg_f(\cD)) }$ and note that $\En_{\pistar}[Z_\pistar] \geq 1-2\delta$ by assumption. This implies that 
\begin{align*}
\En_0 \brk*{N^\pistar} &\geq  (1- \En_{0,\pistar} \brk*{Z_\pistar}) \cdot \frac{1}{7} \log \prn*{\frac{1}{ 2 \delta }} - \frac{\log(2)}{7} - \Tmax \Delta(\Tmax).
\end{align*}
Another application of \pref{lem:change_measure} gives
\begin{align*}
\En_0 \brk*{N^\pistar} &\geq (1- \En_0 \brk*{Z_\pistar}) \cdot \frac{1}{7}  \log \prn*{\frac{1}{ 2 \delta }} - \frac{\log(2)}{7} - \Delta(\Tmax) \prn*{\Tmax + \frac{1}{7}\log \prn*{\frac{1}{ 2 \delta }}}. 
\end{align*}
Summing the above over $\pistar\in\Piell$, we get
\begin{align*}
\sum_{\pistar} \En_0 \brk*{N^\pistar} &\geq \prn*{ \abs{\Piell} - \sum_{\pistar}\En_0 \brk*{Z_\pistar} } \cdot \frac{1}{7} \log \prn*{\frac{1}{ 2 \delta }} - \abs{\Piell} \cdot \frac{\log(2)}{7}  - \abs{\Piell} \cdot \Delta(\Tmax)  \prn*{\Tmax + \frac{1}{7}\log \prn*{\frac{1}{ 2 \delta }}}. \numberthis\label{eq:sum-lower-bound-1}
\end{align*}

It remains to prove an upper bound on $\sum_{\pistar}\En_0 \brk*{Z_\pistar}$. We calculate that 
\begin{align*}
\sum_{\pistar}\En_0 \brk*{Z_\pistar} &= \sum_{\pistar}\En_0 \brk*{ \ind{\st < \Tmax \text{ and } \cE(\pistar, \phi, \alg_f(\cD))  }} \\
&\leq \sum_{\pistar}\En_0 \brk*{ \ind{ \Pr_\pistar\brk*{ \idx(s_1) \in \Jrel^\pistar \text{ and } \alg_f(\cD)(\idx(s_1)_{1:H}) = \pistar(\idx(s_1)_{1:H}) } \ge \frac{\eps}{4} } }   \\ 
&\leq \frac{4}{\eps} \cdot \En_0 \brk*{\sum_{\pistar} \Pr_\pistar\brk*{\idx(s_1) \in \Jrel^\pistar \text{ and } \alg_f(\cD)(\idx(s_1)_{1:H}) = \pistar(\idx(s_1)_{1:H}) } } \numberthis \label{eq:e0-upper-bound-1}
\end{align*}
The last inequality is an application of Markov's inequality.

Now we carefully investigate the sum. For any $\phi \in \Phi$, the sum can be rewritten as
\begin{align*}
    & \hspace{-0.3in}  \sum_{\pistar} \Pr_{\pistar, \phi} \brk*{ \idx(s_1) \in \Jrel^\pistar \text{ and } \alg_f(\cD)(\idx(s_1)_{1:H}) = \pistar(\idx(s_1)_{1:H}) } \\
    = \quad& \sum_{\pistar} \sum_{s_1 \in \cS_1} \Pr_{\pistar, \phi} \brk*{s_1} \Pr_{\pistar, \phi}\brk*{ \idx(s_1)\in \Jrel^\pistar \text{ and } \alg_f(\cD)(\idx(s_1)_{1:H}) = \pistar(\idx(s_1)_{1:H}) ~\mid~ s_1} \\
    \overeq{(i)} \quad& \frac{1}{\abs{\cS_1}} \sum_{s_1 \in \cS_1} \sum_{\pistar} \Pr_{\pistar, \phi} \brk*{ \idx(s_1) \in \Jrel^\pistar \text{ and } \alg_f(\cD)(\idx(s_1)_{1:H}) = \pistar(\idx(s_1)_{1:H}) ~\mid~ s_1} \\
    \overeq{(ii)} \quad& \frac{1}{\abs{\cS_1}} \sum_{s_1 \in \cS_1} \sum_{\pistar} \ind{ \idx(s_1) \in \Jrel^\pistar \text{ and } \alg_f(\cD)(\idx(s_1)_{1:H}) = \pistar(\idx(s_1)_{1:H})}. \numberthis\label{eq:sum-of-indicators}
\end{align*}
The equality $(i)$ follows because regardless of which MDP $M_\pistar$ we are in, the first state is distributed uniformly over $\cS_1$. The equality $(ii)$ follows because once we condition on the first state $s_1$, the probability is either 0 or 1. 

Fix any start state $s_1$. We can write
\begin{align*}
\hspace{1in}&\hspace{-1in} \sum_{\pistar} \ind{ \idx(s_1) \in \Jrel^\pistar \text{ and } \alg_f(\cD)(\idx(s_1)_{1:H})
\pistar(\idx(s_1)_{1:H}) } \\
&= \sum_{\pistar \in \Piell_{\idx(s_1)}} \ind{\alg_f(\cD)(\idx(s_1)_{1:H}) = \pistar(\idx(s_1)_{1:H}) } = 1,
\end{align*}
where the second equality uses the fact that on any index $j$, each $\pistar \in  \Piell_{j}$ behaves differently (Property (3) of \pref{lem:piell-properties}), so $\alg_f(\cD)$ can match at most one of these behaviors. Plugging this back into Eq.~\eqref{eq:sum-of-indicators}, averaging over $\phi \in \Phi$, and combining with Eq.~\eqref{eq:e0-upper-bound-1}, we arrive at the bound
\begin{align*}
    \sum_{\pistar}\En_0 \brk*{Z_\pistar} \le \frac{4}{\eps}.
\end{align*}
We now use this in conjunction with Eq.~\eqref{eq:sum-lower-bound-1} to arrive at the final lower bound
\begin{align*}
\sum_{\pistar} \En_0 \brk*{N^\pistar} &\geq \prn*{ \abs{\Piell} - \frac{4}{\eps} } \cdot \frac{1}{7} \log \prn*{\frac{1}{ 2 \delta }} - \abs{\Piell} \cdot \frac{\log(2)}{7}  - \abs{\Piell} \cdot \Delta(\Tmax)  \prn*{\Tmax + \frac{1}{7}\log \prn*{\frac{1}{ 2 \delta }}}. \numberthis\label{eq:sum-lower-bound-2}
\end{align*}

\paragraph{Step 3: Putting it All Together.}
Combining Eqs.~\eqref{eq:sum-upper-bound} and \eqref{eq:sum-lower-bound-2}, plugging in our choice of $\Delta(\Tmax)$, and simplifying we get
\begin{align*}
    \En_0 \brk*{\st} &\geq \prn*{ \abs{\Piell} - \frac{4}{\eps} } \cdot \frac{1}{7} \log \prn*{\frac{1}{ 2 \delta }} - \abs{\Piell} \cdot \frac{\log(2)}{7}  - \abs{\Piell} \cdot \Delta(\Tmax)  \prn*{\Tmax + \frac{1}{7}\log \prn*{\frac{1}{ 2 \delta }}}. \\
    &\ge \prn*{ \frac{\abs{\Piell}}{2} - \frac{4}{\eps}} \cdot \frac{1}{7} \log \prn*{\frac{1}{ 2 \delta }} - \abs{\Piell} \cdot  \frac{\Tmax^2}{2^{H+3}} \prn*{\Tmax + \frac{1}{7}\log \prn*{\frac{1}{ 2 \delta }}}.
\end{align*}
The last inequality follows since $\delta \le 1/8$ implies $\log(1/(2\delta)) \geq 2 \log (2)$. 

This concludes the proof of \pref{lem:stopping-time}.\qed

\subsection{Change of Measure Lemma}

\begin{lemma} 
\label{lem:change_measure}
Let \(Z \in \brk{0, 1}\) be a  \(\cF_{\Tmax}\)-measurable random variable. Then, for every $\pistar \in \Piell$, 
\begin{align*}
\abs{\En_0 \brk*{Z}  - \En_{0,\pistar} \brk*{Z}} \leq  \Delta(\Tmax) :=  \frac{\Tmax^2}{2^{H+3}}
\end{align*}    
\end{lemma}
\begin{proof} 
First, we note that 
\begin{align*}
    \abs{ \En_0 \brk*{Z}  - \En_{0,\pistar}\brk*{Z} }\le \mathrm{TV}\prn*{\Pr_0^{\cF_\Tmax}, \Pr_{0, \pistar}^{\cF_\Tmax}} \le \sum_{t=1}^\Tmax \En_0 \brk*{ \mathrm{TV}\prn*{ \Pr_0[\cdot|\cF_{t-1}], \Pr_{0,\pistar}[\cdot|\cF_{t-1}] } }.  
\end{align*}
Here $\Pr_0[\cdot|\cF_t]$ denotes the conditional distribution of the $t$-th trajectory given the first $t-1$ trajectories. Similarly $\Pr_{0, \pistar}[\cdot|\cF_t]$ is the averaged over decoders condition distribution of the $t$-th trajectory given the first $t-1$ trajectories. The second inequality follows by chain rule of TV distance \citep[see, e.g.,][pg.~152]{polyanskiy2022information}.

Now we examine each term $ \mathrm{TV}\prn*{ \Pr_0[\cdot|\cF_{t-1}], \Pr_{0,\pistar}[\cdot|\cF_{t-1}] }$. Fix a history $\cF_{t-1}$ and sequence $s_{1:H}$ where all $s_i$ have the same index. We want to bound the quantity
\begin{align*}
    \abs*{\Pr_{0, \pistar}\brk*{S_{1:H}^{(t)} = s_{1:H} ~\mid~ \cF_{t-1}} - \Pr_{0} \brk*{S_{1:H}^{(t)} =  s_{1:H} ~\mid~ \cF_{t-1} }},
\end{align*}
where it is understood that the random variable $S_{1:H}^{(t)}$ is drawn according to the MDP dynamics and algorithm's policy $\alg_t$ (which is in turn a measurable function of $\cF_{t-1}$).

We observe that the second term is exactly 
\begin{align*}
\Pr_{0}\brk*{S_{1:H}^{(t)} =  s_{1:H} ~\mid~ \cF_{t-1} } = \frac{1}{\abs{\cS_1}} \cdot \frac{1}{2^{H-1}},
\end{align*}
since the state $s_1$ appears with probability $1/\abs{\cS_1}$ and the transitions in $M_0$ are uniform to the next state in the combination lock, so each sequence is equally as likely.

For the first term, again the state $s_1$ appears with probability $1/\abs{\cS_1}$. Suppose that $\idx(s_1) \notin \Jrel^\pistar$. Then the dynamics of $\Pr_{0,\pi^\star, \phi}$ for all $\phi \in \Phi$ are exactly the same as $M_0$, so again the probability in this case is $1/(\abs{\cS_1}2^{H-1})$. Now consider when $\idx(s_1) \in \Jrel^\pistar$. At some point $\wh{h} \in [H+1]$, the policy $\alg_t$ will deviate from $\pistar$ for the first time (if $\alg_t$ never deviates from $\pistar$ we set $\wh{h} = H+1)$. The layer $\wh{h}$ is only a function of $s_1$ and $\alg_t$ and does not depend on the MDP dynamics. The correct decoder must assign $\phi(s_{1:\wh{h}-1}) = \textsc{Good}$ and $\phi(s_{\wh{h}: H}) = \textsc{Bad}$, so therefore we have 
\begin{align*}
    \Pr_{0, \pistar} \brk*{S_{1:H}^{(t)} =  s_{1:H} ~\mid~ \cF_{t-1} } 
    &= \Pr_{0, \pistar}\brk*{\phi(s_{1:\wh{h}-1}) = \textsc{Good} \text{ and }\phi(s_{\wh{h}: H}) = \textsc{Bad} ~\mid~ \cF_{t-1} }
\end{align*}
If $s_1 \notin \cF_{t-1}$, i.e., we are seeing $s_1$ for the first time, then the conditional distribution over the labels given by $\phi$ is the same as the unconditioned distribution: % \gene{how to argue this more rigorously? Key idea: while it could be the case that the algorithm has identified $\pistar$ and only plays the good actions, we care about the distribution of \emph{states}, and for a new state index, this is still the uniform distribution.} \ayush{I think it is fine as it is!} 
\begin{align*}
    \Pr_{0, \pistar}\brk*{\phi(s_{1:\wh{h}-1}) = \textsc{Good} \text{ and }\phi(s_{\wh{h}: H}) = \textsc{Bad} ~\mid~ \cF_{t-1} } = \frac{1}{\abs{\cS_1}} \cdot \frac{1}{2^{H-1}}.
\end{align*}
Otherwise, if $s_1 \in \cF_{t-1}$ then we bound the conditional probability by 1.
\begin{align*}
    &\Pr_{0, \pistar} \brk*{S_{1:H}^{(t)} =  s_{1:H} ~\mid~ \cF_{t-1} }  \le \frac{1}{\abs{\cS_1}}.
\end{align*}
Putting this all together we can compute 
\begin{align*}
    \Pr_{0, \pistar} \brk*{S_{1:H}^{(t)} =  s_{1:H} ~\mid~ \cF_{t-1} }  \quad \begin{cases}
         = \frac{1}{\abs{\cS_1}} \cdot \frac{1}{2^{H-1}} &\text{if}\quad \idx(s_1) \notin \Jrel^\pistar, \\[0.5em]
        = \frac{1}{\abs{\cS_1}} \cdot \frac{1}{2^{H-1}} &\text{if}\quad \idx(s_1) \in \Jrel^\pistar \text{ and } s_1 \notin \cF_{t-1}, \\[0.5em]
        \le \frac{1}{\abs{\cS_1}} &\text{if}\quad \idx(s_1) \in \Jrel^\pistar \text{ and } s_1 \in \cF_{t-1}, \\[0.5em]
        = 0 &\text{otherwise.}
    \end{cases}
\end{align*}

Therefore we have the bound 
\begin{align*}
%\hspace{1in}&\hspace{-1in} 
\abs*{\Pr_{0, \pistar}\brk*{S_{1:H}^{(t)} = s_{1:H} ~\mid~ \cF_{t-1}} - \Pr_{0} \brk*{S_{1:H}^{(t)} =  s_{1:H} ~\mid~ \cF_{t-1} }} 
&\leq \frac{1}{\abs{\cS_1}}\ind{\idx(s_1) \in \Jrel^\pistar, s_1 \in \cF_{t-1}}. 
\end{align*}
Summing over all possible sequences $s_{1:H}$ we have
\begin{align*}
    \mathrm{TV}\prn*{ \Pr_0[\cdot|\cF_{t-1}], \Pr_{0,\pistar}[\cdot|\cF_{t-1}] } &\le \frac{1}{2} \cdot \frac{(t-1) \cdot 2^{H-1}}{\abs{\cS_1}},
\end{align*}
since the only sequences $s_{1:H}$ for which the difference in the two measures are nonzero are the ones for which $s_1 \in \cF_{t-1}$, of which there are $(t-1) \cdot 2^{H-1}$ of them.

Lastly, taking expectations and summing over $t=1$ to $\Tmax$ and plugging in the value of $\abs{\cS_1} = 2^{2H}$ we have the final bound.
\end{proof} 

The next lemma is a straightforward modification of \citep[Lemma 5]{domingues2021episodic}, with varying rewards instead of varying transitions.
\begin{lemma}\label{lem:kl-calculation}
Let $M$ and $M'$ be two MDPs that are identical in transition and differ in the reward distributions, denote $r_h(s,a)$ and $r'_h(s,a)$. Assume that for all $(s,a)$ we have $r_h(s,a) \ll r'_h(s,a)$. Then for any stopping time $\st$ with respect to $(\cF^t)_{t\ge 1}$ that satisfies $\Pr_{M}[\st < \infty] = 1$,
\begin{align*}
    \KL{\Pr_{M}^{I_{\st}}}{\Pr_{M'}^{I_{\st}}} = \sum_{s\in\cS, a \in \cA, h \in [H]} \En_{M}[N^\st_{s,a,h}] \cdot \KL{r_h(s,a)}{r'_h(s,a)}, 
\end{align*}
where $N^\st_{s,a,h} := \sum_{t=1}^\st \ind{(S_h^{(t)}, A_h^{(t)}) = (s,a) }$ and $I_\st: \Omega \mapsto \bigcup_{t\ge 1} \cI_t : \omega \mapsto I_{\st(\omega)}(\omega)$ is the random vector representing the history up to episode $\st$.
\end{lemma}

\begin{lemma}[Lemma 1, \cite{garivier2019explore}]\label{lem:KL-to-kl} Consider a measurable space $(\Omega, \cF)$ equipped with two distributions $\bbP_1$ and $\bbP_2$. For any $\cF$-measurable function $Z: \Omega \mapsto [0,1]$ we have 
\begin{align*}
    \KL{\bbP_1}{\bbP_2} \ge \kl{\En_1[Z]}{\En_2[Z]}.
\end{align*}
\end{lemma}

\clearpage 
 
\section{Proofs for \pref{sec:upper-bound}} \label{app:upper_bound_main} 

\iffalse
\ascomment{Things to add in the notation for the proof for this algorithm:  
\begin{enumerate} 
   \item for $s = 0$ we let $h = 0$. 
   \item for $s = 0$ we do not need $\pi_s$ and the next If always holds 
   \item \(s_\bot\) = end
   \item \(s_\top\) = start 
\end{enumerate}
}

\ascomment{Provide a text description of what the algorithms are doing here!} 
\ayush{Zeyu, fix the reward in the algorithm!} 
\fi

\subsection{Algorithmic Details and Preliminaries}   \label{app:algorithm_details} 
\par In this subsection, we provide the details of the subroutines that do not appear in the main body, in Algorithms \ref{alg:sample}, \ref{alg:dp}, and \ref{alg:eval}. The transitions and  reward functions in \pref{line:reward_calculator} in \pref{alg:eval} are computed using Eqs.~\eqref{eq:empirical_MRP_dynamics} and \eqref{eq:empirical_MRP_rewards}, which are specified below, after introducing additional notation. 

\begin{algorithm}[!htp] 
    \caption{$\datacollector$}\label{alg:sample}
    \begin{algorithmic}[1]
            \Require State: \(s\), Reacher policy: \(\pi_s\), Exploration policy set: \(\Picore\), Number of samples: \(n\). 
            \If {$s = s_\top$}  \hfill \algcomment{Uniform sampling for start state \(s_\top\)}  
            \For{$t = 1, \dots, n$}
                \State Sample $\pi' \sim\mathrm{Uniform}(\Picore)$, and run $\pi'$ to collect \(\tau = \prn{s_1, a_1, \cdots, s_H, a_H}\). 
%               \State run $\pi'$ from layer $h$ to $H$ to collect $\tau$;
            \State $\mD_s \leftarrow \mD_s\cup\{\tau\}$. 
            \EndFor 
            \Else \hfill \algcomment{\(\pi_s\)-based sampling for all other states \(s \neq s_\top\)}  
        \State Identify the layer \(h\) such that $s\in\mS_h$. 
        \For{$t = 1, \dots, n$} 
            \State Run $\pi_s$ for the first $h-1$ time steps, and collect trajectory \(\prn{s_1, a_1, \cdots, s_{h-1}, a_{h-1}, s_h}\). 
            \If{$s_h = s$}\label{line:is1} 
                \State Sample $\pi'\sim\mathrm{Uniform}(\Picore)$, and run $\pi'$ to collect remaining \(\prn{s_h, a_h, \cdots, s_H, a_H}\).
%               \State run $\pi'$ from layer $h$ to $H$ to collect $\tau$;
                \State $\mD_s\leftarrow \mD_s\cup\{\tau = \prn{s_1, a_1, \cdots, s_H, a_H}\}$. \label{line:is2}
            \EndIf
        \EndFor  
            \EndIf 

        \State \textbf{Return} dataset $\mD_s$.  
    \end{algorithmic}
\end{algorithm} 

%\ayush{The data collector may end up collecting lesser than \(n_2\) many trajectories} 
 
\begin{algorithm}[!htp]
\caption{$\estreach$}\label{alg:dp}
	\begin{algorithmic}[1]
		\Require State space \(\Stab\), MRP \(\MRPsign\),  State $\bar{s}\in\mS^{\mathrm{tab}}$.
        \State Let $P$ be the transition of $\MRPsign$.
		\State Initialize $V(s) = \ind{s = \bar{s}}$ for all \(s \in \Stab\). 
        \State \textbf{Repeat} \(H + 1\) times:
               % \For{\(H + 1\) times}  
			\State \(\quad\) For all  $s \in \Stab$, calculate $V(s)\leftarrow \sum_{s'\in\mS^{\mathrm{tab}}} P_{s\to s'} \cdot V(s').$ \hfill \algcomment{Dynamic Programming} 
%		\EndFor
		\State \textbf{Return} $V(s_\top)$.
	\end{algorithmic}
\end{algorithm}

\begin{algorithm}[!htp]
	\caption{$\evaluate$}\label{alg:eval}
	\begin{algorithmic}[1] 
            \Require Policy set $\Picore$, Reachable states \(\reachablestates\), Datasets $\{\mD_s\}_{s \in \reachablestates}$, Policy $\pi$ to be evaluated. 
            \State Compute \(\SHP{\pi} \leftarrow \cS_\pi^+ \cap \reachablestates\) and \(\Stab = \SHP{\pi} \cup \crl{s_\bot}\). % \(\SRem{\pi} \leftarrow \cS_\pi^+ \setminus \SHP{\pi}\). 

		\For{$s, s'$ in $\Stab$}  \hfill \algcomment{Compute transitions and rewards on \(\Stab\)}
            \State Let \(h, h'\) be such that \(s \in \cS_h\) and \(s' \in \cS_{h'}\) 
            \If{\(h < h'\)}
           \State Calculate $\widehat{P}_{s\to s'}^\pi, \widehat{r}_{s\to s'}^\pi$ according to \eqref{eq:empirical_MRP_dynamics} and \eqref{eq:empirical_MRP_rewards};\label{line:reward_calculator} 
           \Else 
            \State Set $\widehat{P}_{s\to s'}^\pi \leftarrow 0$, $\widehat{r}_{s\to s'}^\pi \leftarrow 0$.  
           \EndIf
            \EndFor 
		\State Set $\widehat{V}(s) = 0$ for all \(s \in \Stab\).  
		%\For{$h=H:1$} 
            \State \textbf{Repeat} for \(H+1\) times: \hfill \algcomment{Evaluate \(\pi\) by dynamic programming} 
            \State \quad For all $s \in \Stab$, calculate $\widehat{V}(s)\leftarrow \sum_{\Stab} \widehat{P}_{s\to s'}^\pi \cdot \left(\widehat{r}_{s\to s'}^\pi + \widehat{V}(s')\right).$ 
		%\EndFor
            %\State Calculate $\widehat{V}^\pi\leftarrow \sum_{s'\in\SHP{\pi}\cup\{s_\bot\}} \widehat{P}_{s_\top\to s'}^\pi\left(\widehat{r}_{s_\top\to s'}^\pi + \widehat{V}(s')\right)$.
		\State \textbf{Return} $\widehat{V}(s_\top)$.  
	\end{algorithmic}
\end{algorithm}

% \begin{algorithm}[!htp]
% 	\caption{$\evaluate(\{\Picore\}, \{\mD_s\}, \pi)$}\label{alg:eval}
% 	\begin{algorithmic}[1] 
%             \State \textbf{Input: } Exploration policies $\Picore$, datasets $\{\mD_s\}$, policy $\pi$ to be evaluated.
% 		\For{$s\in(\SHP{\pi})\cup \{s_\top\}, s'\in\SHP{\pi}\cup\{s_\bot\}$} 
% 				\State Calculate $\widehat{P}_{s\to s'}^\pi, \widehat{r}_{s\to s'}^\pi$ according to \eqref{eq:empirical} and \eqref{eq:reward};
% 		\EndFor
% 		\State Set $\widehat{V}(s_\bot) = 0$. 
% 		\For{$h=H:1$} 
% 			\State For $\forall s\in\mS_h\cap \SHP{\pi}$, calculate $\widehat{V}(s)\leftarrow \sum_{s'\in\SHP{\pi}\cup\{s_\bot\}} \widehat{P}_{s\to s'}^\pi\left(\widehat{r}_{s\to s'}^\pi + \widehat{V}(s')\right).$
% 		\EndFor
%             \State Calculate $\widehat{V}^\pi\leftarrow \sum_{s'\in\SHP{\pi}\cup\{s_\bot\}} \widehat{P}_{s_\top\to s'}^\pi\left(\widehat{r}_{s_\top\to s'}^\pi + \widehat{V}(s')\right)$.
% 		\State \textbf{Output: } $\widehat{V}^$.
% 	\end{algorithmic}
% \end{algorithm}

%\subsection{Preliminaries and Additional Notation} 

%\ayush{Recall superscript and subscript notations in policies!} 

We recall the definition of petals and sunflowers given in the main body (in Definitions \ref{def:petal_policy} and \ref{def:core_policy}). In the rest of this section, we assume that $\Pi$ is a $(K, D)$-sunflower with $\Picore$ and $\mS_\pi$ for any $\pi\in \Pi$.

\begin{definition}[Petals and Sunflowers (Definitions~\ref{def:petal_policy} and \ref{def:core_policy} in the main body)]  
For a policy set \(\bar{\Pi}\), and states \(\bar{\cS} \subseteq \cS\), a policy \(\pi\) is said to be a \(\bar{\cS}\)-\textit{petal} on \(\bar{\Pi}\) if for all \(h \leq h' \leq H\), and partial trajectories $\tau = (s_h, a_h, \cdots, s_{h'}, a_{h'})$ that are consistent with $\pi$: either \(\tau\) is also consistent with some \(\pi' \in \bar{\Pi}\), or there exists \(i \in (h, h']\) s.t.~$s_i\in \bar{\mS}$.

    A policy class $\Pi$ is said to be a \((K, D)\)-\coreset{} if there exists a set \(\Picore\) of Markovian policies with $|\Picore|\le K$ such that for every policy $\pi\in \Pi$ there exists a set $\mS_\pi \subseteq \mS$, of size at most \(D\), so that \(\pi\) is an \(S_\pi\)-petal on \(\Picore\). 
\end{definition}

\paragraph{Additional notation.} Recall that we assumed that the state space \(\cS = \cS_1 \times \dots \cS_H\) is layered. Thus, given a state \(s\), we can infer the layer \(h\) such that \(s \in \cS_h\). By definition \(s_\top\) belongs to the layer \(h = 0\) and \(s_\bot\) belongs to the layer \(h = H\). In the following, we define additional notation: 

\begin{enumerate}[label=\((\alph*)\)]  
%\item For any set \(\bar{\cS} \subseteq \cS\), we define \(\bar{\cS}^\complement\) to denote its complement i.e.~the \(\bar{\cS}^\complement = \cS \setminus \bar{\cS}\).  \ayush{Do we need this?} \zeyu{I think the only thing we needed here is $\SRem{\pi} = \mS_\pi\backslash \SHP{\pi}$}

\item \textit{Sets \(\event{s}{s'}{\bar \cS}\)}: For any set \(\bar \cS\), and states \(s, s' \in \cS\), we define \(\event{s}{s'}{\bar \cS}\) as the set of all the trajectories that go from \(s\) to \(s'\) without passing through any state in \(\bar \cS\) in between. 

More formally, let state \(s\) be at layer \(h\), and \(s'\) be at layer \(h'\). Then, \(\event{s}{s'}{\bar \cS}\) denotes the set of all the trajectories \(\tau = (s_1, a_1, \cdots, s_H, a_H)\) that satisfy all of the following:  
\begin{enumerate}[label=\(\bullet\)] 
\item \(s_h = s\), where \(s_h\) is the state at timestep \(h\) in \(\tau\). 
\item \(s_{h'} = s'\), where \(s_{h'}\) is the state at timestep \(h'\) in \(\tau\). 
\item For all \(h < \wt{h} < h'\), the state \(s_{\wt{h}}\), at time step \(\wt h\) in \(\tau\), does not lie in the set \(\bar \cS\).  
\end{enumerate}

\noindent 
Note that when \(h' \leq h\), we define 
 \(\event{s}{s'}{\bar \cS} = \emptyset\). Additionally, we define  \(\event{s_\top}{s}{\bar \cS}\) as the set of all trajectories that go to \(s'\) (from a start state) without going through any state in \(\bar \cS\) in between. Finally, we define   \(\event{s}{s_\bot}{\bar \cS}\) as the set of all the trajectories that go from \(s\) at time step \(h\) to the end of the episode without passing through any state in \(\bar \cS\) in between. 

Furthermore, we use the shorthand $\setall{\pi}{s}{s'} \coloneqq \event{s}{s'}{\cS_\pi}$ to denote the set of all the trajectories that go from \(s\) to \(s'\) without passing though any leaf state \(\cS_\pi\). 

% \item \textit{Sets \(\setall{\pi}{s}{s'}\)}: For any policy \(\pi\), and states \(s, s' \in \cS\), we define \(\setall{\pi}{s}{s'}\) as the set of all the trajectories that are consistent with \(\pi\), and that go from \(s\) to \(s'\) without passing through any state in \(\mS_\pi\) in between. 

% More formally, let \(\pi \in \Pi\), state \(s\) be at layer \(h\), and \(s'\) be at layer \(h'\). Then, \(\setall{\pi}{s}{s'}\) denotes the set of all the trajectories \(\tau = (s_1, a_1, \cdots, s_H, a_H)\) that satisfy all of the following:  
% \begin{enumerate}[label=\(\bullet\)] 
% \item \(\tau\) is consistent with \(\pi\), i.e.~\(\pi \cons \tau\). 
% \item \(s_h = s\), where \(s_h\) is the state at timestep \(h\) in \(\tau\). 
% \item \(s_{h'} = s'\), where \(s_{h'}\) is the state at timestep \(h'\) in \(\tau\). 
% \item For all \(h < \wt{h} < h'\), the state \(s_{\wt{h}}\) at timestep \(\wt h\) in \(\tau\), does not lie in the set \(\cS_\pi\).  
% \end{enumerate}

% \noindent 
% Note that when \(h' \leq h\), we define 
% \(\setall{\pi}{s}{s'} = \emptyset\). Additionally, we define \(\setall{\pi}{s_\top}{s'}\) as the set of all trajectories consistent with \(\pi \) that go to \(s'\) (from a start state) without going through any state in \(\cS_\pi\) in between. Finally, we define  \(\setall{\pi}{s}{s_\bot}\) as the set of all the trajectories that are consistent with \(\pi\) and go from \(s\) at time step \(h\) to the end of the episode without passing through any state in \(\cS_\pi\) in between. 

\item Using the above notation, for any \(s \in \cS\) and set \(\bar \cS \subseteq \cS\), we define \(\bard{\pi}{s}{\bar{\cS}}\) as the probability of reaching \(s\) (from a start state)  without passing through any state in \(\bar \cS\) in between, i.e.~
\begin{equation}\label{eq:dbar}\begin{aligned} 
\bard{\pi}{s}{\bar{\cS}} &= \bbP^{\pi}\brk*{\tau~\text{reaches}~s~\text{without passing through any state in \(\bar{\cS}\) before reaching \(s\)}} \\
&= \bbP^{\pi}\brk*{\tau~\in \event{s_\top}{s}{\bar \cS}}.  
%&= \Pr^{\pi, M}\prn*{ \tau \in \setall{\pi}{s_\top}{s}}. 
\end{aligned}\end{equation}

\end{enumerate}

\medskip

We next recall the notation of Markov Reward Process and formally define both the population versions of policy-specific MRPs. 
%\ayush{This should be moved somewhere else. We should define the underlying tabular MDP cleanly somewhere!} 

\paragraph{Markov Reward Process (MRP).} A Markov reward process \(\MRPsign = \mathrm{MRP}(\cS, P, R, H, s_\top, s_\bot)\) is defined over the state space \(\cS\) with start state \(s_\top\) and end state \(s_\bot\),  for trajectory length \(H + 2\). Without loss of generality, we assume that \(\crl{s_\top, s_\bot} \in \cS\). The transition kernel is denoted by \(P: \cS \times \cS \to [0,1] \), such that for any $s \in \cS$, $\sum_{s'} P_{s \to s'} = 1$; the reward kernel is denoted \(R: \cS \times \cS \to \Delta([0,1])\). Throughout, we use the notation $\rightarrow$ to signify that the transitions and rewards are defined along the edges of the MRP. 

A trajectory in $\MRPsign$ is of the form \(\tau = (s_\top, s_1, \cdots, s_{H}, s_\bot)\), where \(s_h \in \cS\) for all \(h \in [H]\). Furthermore, from any state \(s \in \cS\), the MRP transitions\footnote{Our definition of Markov Reward Processes (MRP) deviates from MDPs that we considered in the paper, in the sense that we do not assume that the state space \(\cS\) is layered in an MRP. This variation is only adapted to simplify the proofs and the notation in the rest of the paper.} to another state \(s' \in \cS\) with probability \(P_{s \rightarrow s'}\), and obtains the rewards \(r_{s \rightarrow s'} \sim R_{s \rightarrow s'}\). Thus,  
\begin{align*}
\bbP^{\MRPsign}[\tau] =  P_{s_\top \rightarrow s_1} \cdot \prn*{\prod_{h=1}^{H-1} P_{s_h \rightarrow s_{h+1}} } \cdot P_{s_H \rightarrow s_\bot},  
\end{align*} 
and the rewards   
\begin{align*} 
R^{\MRPsign}(\tau) = r_{s_{\top} \rightarrow s_1} + \sum_{h=1}^H r_{s_{h} \rightarrow s_{h+1}} + r_{s_{H} \rightarrow s_\bot}. 
\end{align*} 
Furthermore, in all the MRPs that we consider in the paper, we have \(P_{s_\bot \rightarrow s_\bot} = 1\) and \(r_{s_\bot \rightarrow s_\bot} = 0\). 

\paragraph{Policy-Specific Markov Reward Processes.} 
A key technical tool in our analysis will be policy-specific MRPs that depend on the set \(\reachablestates\) of the states that we have explored so far. Recall that for any policy \(\pi\), \(\cS_\pi^+ = \cS_\pi \cup \crl{s_\top, s_\bot}\), \(\SHP{\pi} = \cS_\pi^+ \cap \reachablestates\) and \(\SRem{\pi} = \cS_\pi^+  \setminus (\SHP{\pi} \cup \crl{s_\bot})\). We define the expected and the empirical versions of policy-specific MRPs below; see \pref{fig:MRP} for an illustration. 
% Given a set \(\reachablestates\) such that \(s_\top \in \reachablestates\) but \(s_\bot \notin \reachablestates\), recall that for any policy \(\pi\), in our algorithms we define \(\cS_\pi^+ = \cS_\pi \cup \crl{s_\top, s_\bot}\), \(\SHP{\pi} = \cS_\pi^+ \cap \reachablestates\) and \(\SRem{\pi} = \cS_\pi^+  \setminus \SHP{\pi}\). \ayush{Make MRP and MDP look as same as possible where $M = (\cS, \cA, H, P, R, \mu)$ !} 
\begin{enumerate}[label=\((\alph*)\)]
    \item \textbf{Expected Version of Policy-Specific MRP.} We define \(\MRPsign^\pi_{\reachablestates} = \mathrm{MRP}(\cS_\pi^+, P^\pi, r^\pi, H, s_\top, s_\bot)\) where  \ayush{Can I make all \(r^\pi\) to \(R^\pi\)? \\ The notation is inconsistent over here from the main body. wdyt?} 

\begin{enumerate}[label=\(\bullet\)]
    \item \textit{Transition Kernel \(P^\pi\):} For any \(s \in \SHP{\pi}\) and \(s' \in \cS_\pi^+\), we have 
    \begin{align*}
P_{s\to s'}^\pi =  \EE^\pi\left[\ind{\tau\in \setall{\pi}{s}{s'}}\big|s_h = s\right], \numberthis \label{eq:policy_MRP_dynamics} 
\end{align*}
where the expectation above is w.r.t.~the  trajectories drawn using \(\pi\) in the underlying MDP, and \(h\) denotes the time step such that \(s \in \cS_h\) (again, in the underying MDP). Thus, the transition $P_{s\to s'}^\pi$ denotes the probability of taking policy $\pi$ from $s$ and directly transiting to $s'$ without visiting any other states in $\mS_\pi$. Furthermore, \(P^\pi_{s \rightarrow s'} = \ind{s' = s_\bot}\) for all \(s \in \SRem{\pi}\cup \{s_\bot\}\). 

\item \textit{Reward Kernel \(r^\pi\):} 
 For any \(s \in \SHP{\pi}\) and \(s' \in \cS_\pi^+\), we have 
\begin{align*}
r_{s\to s'}^\pi & \coloneqq \EE^\pi\brk*{R(\tau_{h:h'})\ind{\tau\in \setall{\pi}{s}{s'}}\big|s_h = s}, \numberthis \label{eq:policy_MRP_reward} 
\end{align*} 
where \(R(\tau_{h:h'})\) denotes the reward for the partial trajectory \(\tau_{h:h'}\) in the underlying MDP. The reward $r_{s\to s'}^\pi$ denotes the expectation of rewards collected by taking policy $\pi$ from $s$ and directly transiting to $s'$ without visiting any other states in $\mS_\pi$. Furthermore, \(r^\pi_{s \rightarrow s'} = 0\) for all \(s \in \SRem{\pi}\cup \{s_\bot\}\). 
\end{enumerate}
Throughout the analysis, we use $\Pr^\MRPsign[\cdot] \coloneqq \Pr^{\MRPsign^{\pi}_{\reachablestates}}[\cdot]$ and $\En^\MRPsign[\cdot] \coloneqq \En^{\MRPsign^{\pi}_{\reachablestates}}[\cdot]$ as a shorthand, whenever clear from the context.

	% where $P_{s\to s'}^\pi, r_{s\to s'}^\pi$ are defined as
	% \begin{equation}\label{eq:policy_MRP}
	% 	\begin{aligned}
	% 		P_{s\to s'}^\pi &\coloneqq \EE_{\tau\sim \pi}\left[\ind{\tau\in \event{s}{s'}{\mS_\pi}}\big|s_h = s\right],\\
	% r_{s\to s'}^\pi & \coloneqq \EE_{\tau\sim \pi}\left[R(\tau_{h:h'})\ind{\tau\in \event{s}{s'}{\mS_\pi}}\big|s_h = s\right],  \quad \forall s\in \SHP{\pi}, s'\in \mS_\pi\cup\{s_\bot\}\\
	% 	\end{aligned}
	% \end{equation} 

\medskip

\item \textbf{Empirical Version of Policy-Specific MRPs.}  Since the learner only has sampling access to the underlying MDP, it can not directly construct the MRP \(\MRPsign^\pi_{\reachablestates}\). Instead, in  \pref{alg:main}, the learner constructs an empirical estimate for \(\MRPsign^\pi_{\reachablestates}\), defined as   \(\widehat \MRPsign^\pi_{\reachablestates} = \mathrm{MRP}(\cS_\pi^+, \widehat P^\pi, \widehat r^\pi, H, s_\top, s_\bot)\) where  

%Given a set \(\reachablestates\) such that \(s_\top \in \reachablestates\) but \(s_\bot \notin \reachablestates\), recall that, for any policy \(\pi\), \(\cS_\pi^+ = \cS_\pi \cup \crl{s_\top, s_\bot}\), \(\SHP{\pi} = \cS_\pi^+ \cap \reachablestates\) and \(\SRem{\pi} = \cS_\pi^+  \setminus \SHP{\pi}\). 
%In \pref{alg:main} and \pref{alg:eval}, we define an empirical Markov Reward Process \(\widehat \MRPsign^\pi_{\reachablestates} = \mathrm{MRP}(\cS_\pi^+, \widehat P^\pi, \widehat r^\pi, H, s_\top, s_\bot)\) where 

\begin{enumerate}[label=\(\bullet\)]
    \item \textit{Transition Kernel \(\widehat P^\pi\):} For any \(s \in \SHP{\pi}\) and \(s' \in \cS_\pi^+\), we have %\ascomment{Fill similar to the rewards below:} 
\begin{align*} 
	\widehat{P}_{s\to s'}^\pi = \frac{|\Picore|}{|\mD_s|} \sum_{\tau\in \mD_s}\frac{\ind{\pi\cons\tau_{h:h'}}}{\sum_{\pi'\in \Picore}\ind{\pi' \cons\tau_{h:h'}}}\ind{\tau\in \setall{\pi}{s}{s'} }, \numberthis \label{eq:empirical_MRP_dynamics}
\end{align*}
	where \(\Picore\) denotes the core of the sunflower corresponding to \(\Pi\) and \(\cD_s\) denotes a dataset of trajectories collected via \(\datacollector(s, \pi_s, \Picore, n_2)\). Furthermore, \(\widehat{P}^\pi_{s \rightarrow s'} = \ind{s' = s_\bot}\) for all \(s \in \SRem{\pi}\cup \{s_\bot\}\). 
    \item \textit{Reward Kernel \(\widehat r^\pi\):} 
 For any \(s \in \SHP{\pi}\) and \(s' \in \cS_\pi^+\), we have 
\begin{align*} 
	\widehat{r}_{s\to s'}^\pi = \frac{|\Picore|}{|\mD_s|} \sum_{\tau\in \mD_s}\frac{\ind{\pi\cons\tau_{h:h'}}}{\sum_{\pi'\in \Picore}\ind{\pi' \cons\tau_{h:h'}}}\ind{\tau\in \setall{\pi}{s}{s'} }R(\tau_{h:h'}), \numberthis \label{eq:empirical_MRP_rewards} 
\end{align*}
% \begin{align*} 
% 	\widehat{r}_{s\to s'}^\pi = \frac{|\Picore|}{|\mD_s|}\sum_{\tau\in \mD_s}\frac{\ind{\pi\cons\tau_{h:h'}}}{\sum_{\pi'\in \Picore}\ind{\pi_e\cons\tau_{h:h'}}}\ind{\tau\in \event{s}{s'}{\mS_{\pi}}}R(\tau_{h:h'}), 1 \numberthis \label{eq:empirical_MRP_rewards}
% \end{align*}
where \(\Picore\) denotes the core of the sunflower corresponding to \(\Pi\), \(\cD_s\) denotes a dataset of trajectories collected via \(\datacollector(s, \pi_s, \Picore, n_2)\), and  $R(\tau_{h:h'}) = \sum_{i=h}^{h'-1} r_i$. Furthermore, \(\wh r^\pi_{s \rightarrow s'} = 0\) for all \(s \in \SRem{\pi}\).
%\ayush{Do we need both of the indicators above?} 
\end{enumerate}
The above approximates the MRP given by \pref{eq:emp-transitions} and \pref{eq:emp-rewards} in the main body. 
\end{enumerate}

\begin{figure}[!t]
    \centering
\includegraphics[scale=0.4, trim={0cm 12cm 41cm 0cm}, clip]{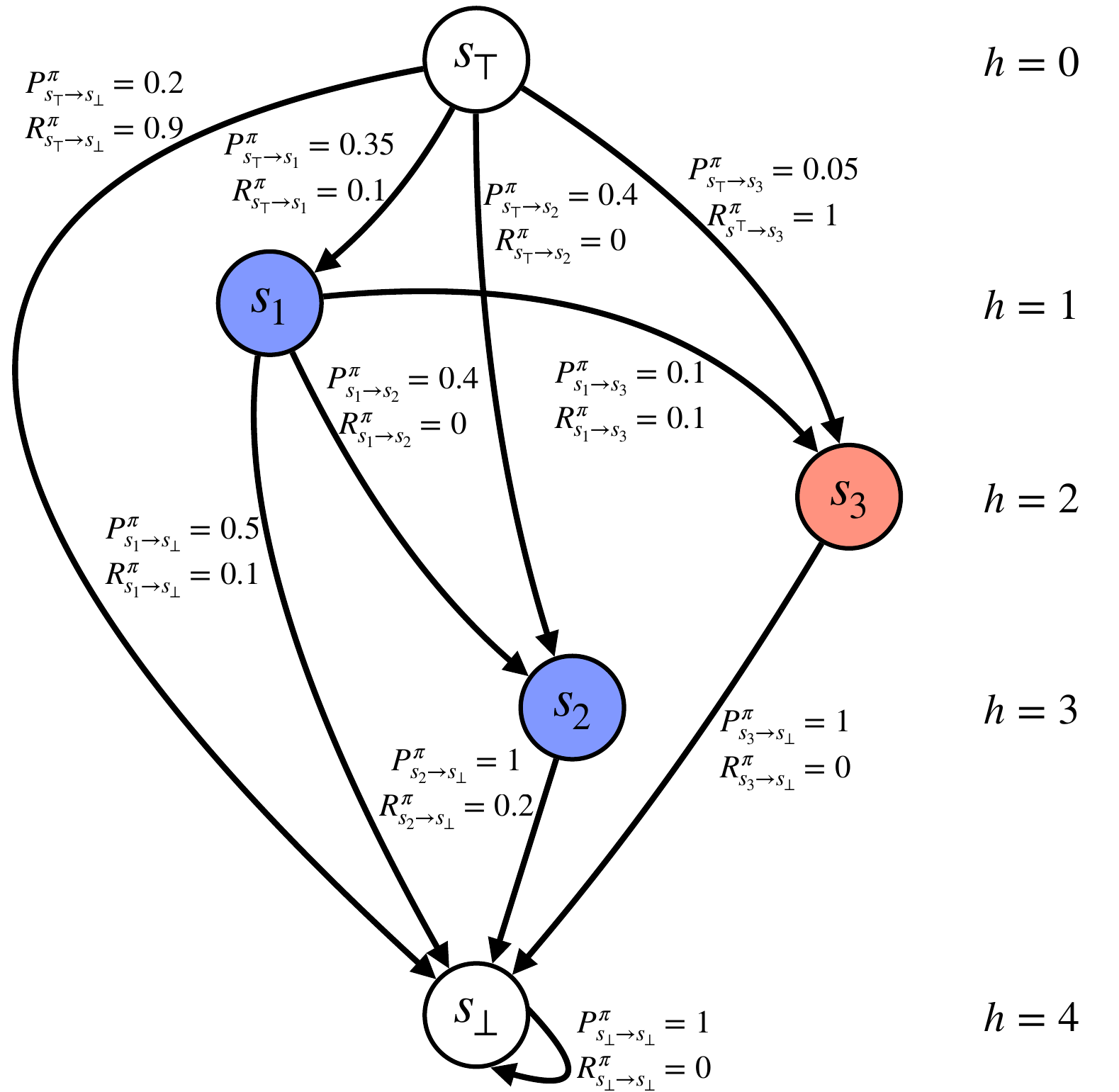}
    \caption{Illustration of an MRP $\MRPsign^\pi_{\reachablestates}$ with $\cS_\pi = \crl{s_1, s_2, s_3}$ and $\reachablestates = \crl{s_1, s_2}$. In the original MDP $M$, $s_1 \in \cS_1$, $s_2 \in \cS_3$, and $s_3 \in \cS_2$. The edges are labeled with the values of $P^\pi_{s \to s'}$ and $R^\pi_{s \to s'}$. Notice that (1) there are no edges from $s_2 \to s_3$ or $s_3 \to s_2$ because trajectories cannot go from later layers to earlier ones; (2) since $s_3 \notin \reachablestates$, there is no edge from $s_3 \to s_2$, and instead we have $P^\pi_{s_3 \to s_\bot} = 1$ and $R^\pi_{s_3 \to s_\bot} = 0$; (3) $s_\bot$ is an absorbing state with no rewards.} 
    \label{fig:MRP}
\end{figure}

\paragraph{Properties of Trajectories in the Policy-Specific MRPs.} We state several properties of trajectories in the policy-specific MRPs, which will be used in the proofs. Let \(\tau = (s_\top, s_1, \cdots, s_{H}, s_\bot)\) denote a trajectory from either $\MRPsign^\pi_{\reachablestates}$ or $\widehat \MRPsign^\pi_{\reachablestates}$.  
\begin{enumerate}[label=\((\arabic*)\)] 
    \item For some $k \leq H$ we have $s_1, \cdots, s_k \in \cS_\pi$ and $s_{k+1} = \cdots = s_H = s_\bot$ (if $k = H$ we say the second condition is trivially met). 
    \item Each state in $s_1, \cdots, s_k$ is unique.
    \item Letting $\mathsf{h}(s)$ denote the layer that a state $s \in \cS_\pi$ is in, we have $\mathsf{h}(s_1) < \cdots < \mathsf{h}(s_k)$.
    \item Either (a) $s_1, \cdots, s_k \in \SHP{\pi}$, or (b) $s_1, \cdots, s_{k-1} \in \SHP{\pi}$ and $s_k \in \SRem{\pi}$.
\end{enumerate}
\ayush{Add an informal description here!} 

\iffalse
\ascomment{Older stuff:  
		\begin{align*}
			\wh P_{s\to s'}^\pi &\coloneqq \EE_{\tau\sim \pi}\left[\ind{\tau\in \event{s}{s'}{\mS_\pi}}\big|s_h = s\right],\\
	\wh r_{s\to s'}^\pi & \coloneqq \EE_{\tau\sim \pi}\left[R(\tau_{h:h'})\ind{\tau\in \event{s}{s'}{\mS_\pi}}\big|s_h = s\right],  \quad \forall s\in \SHP{\pi}, s'\in \mS_\pi\cup\{s_\bot\} \numberthis \label{eq:empirical_MRP_dynamics} 
		\end{align*}
For any state $s \in \mS\cup\{s_\top\}$, policy \(\pi \in \Pi\), exploration set \(\Picore\), and dataset \(\cD_s\) of trajectories, we define for an empirical estimate of the reward ${r}_{s\to s'}^\pi$ for any $s'\in \mS\cup\{s_\bot\}$, as: 
\begin{equation}\label{eq:reward}
	\widehat{r}_{s\to s'}^\pi = \frac{|\Picore|}{|\mD_s|}\sum_{\tau\in \mD_s}\frac{\ind{\pi\cons\tau_{h:h'}}}{\sum_{\pi'\in \Picore}\ind{\pi_e\cons\tau_{h:h'}}}\ind{\tau\in \event{s}{s'}{\mS_{\pi}}}R(\tau_{h:h'}),
\end{equation}
where \(h\) and \(h'\) are the layers corresponding to \(s\) and \(s'\) respectively, and for any trajectory $\tau = (s_1, a_1, r_1, \cdots, s_H, a_H, r_H)$, we define $R(\tau_{h:h'}) = \sum_{i=h}^{h'-1} r_i$. 
} 
\fi

\paragraph{Parameters Used in \pref{alg:main}.}  Here, we list all the parameters that are used in \pref{alg:main} and its subroutines: 
\begin{align*}
n_1 &= C_1 \frac{(D+1)^4K^2\log(|\Pi|(D+1)/\delta)}{\epsilon^2}, \\
 n_2 &= C_2 \frac{D^3(D+1)^2K^2\log(|\Pi| (D+1)^2/\delta)} {\epsilon^3}, \numberthis \label{eq:alg_parameter}
\end{align*}
where $C_1, C_2 > 0$ are absolute numerical constants, which will be specified later in the proofs. 

\begin{comment} 
\begin{definition}
	We define the following events: for any trajectory $\tau = (s_1, a_1, \cdots, s_H, a_H)$, $s\in \mS_{h_1}, s'\in\mS_{h_2}$ and $\bar{\mS}\subset \mS$, we say
$$\tau\in \event{s}{s'}{\bar{\mS}}$$
if $s_{h_1} = s, s_{h_2} = s'$, and for any $h_1 < h < h_2$, $s_{h}\not\in \bar{\mS}$. If $h_1\ge h_2$, $\event{s}{s'}{\bar{\mS}}$ will always be empty. We further define cases for $s = s_\top$, $\tau \in \event{s}{s'}{\bar{\mS}}$ if and only if $s_{h_2} = s'$ and for any $h < h_2$, $s_h\not\in \bar{\mS}$; and also for $s' = s_\bot$, $\tau\in \event{s}{s'}{S}$ if and only if $s_{h_1} = s$ and for any $h > h_1$, $s_h\not\in \bar{\mS}$.
\end{definition}
\par For any $s\in\mS$ and $\bar{S}\subset \mS$, we use $\bar{d}(s, \bar{s})$ to denote
\begin{equation}\label{eq:dbar}\bard{\pi}{s}{\bar{S}}\coloneqq P^{\pi, M}[\tau: \tau\in \event{s_\top}{s}{\bar{S}}].\end{equation}
\par The empirical rewards calculated in \pref{alg:eval} is: for $s\in\mS\cup\{s_\top\}, s'\in \mS\cup\{s_\bot\}$, if $s, s'$ are in layer $h$ and $h'$, 
\begin{equation}\label{eq:reward}
	\widehat{r}_{s\to s'}^\pi = \frac{|\Picore|}{|\mD_s|}\sum_{\tau\in \mD_s}\frac{\ind{\pi\cons\tau_{h:h'}}}{\sum_{\pi'\in \Picore}\ind{\pi_e\cons\tau_{h:h'}}}\ind{\tau\in \event{s}{s'}{\mS_{\pi}}}R(\tau_{h:h'}),
\end{equation}
where $R(\tau_{h:h'}) = \sum_{i=h}^{h'-1} r_i(\tau)$ with $\tau = (s_1, a_1, r_1, s_2, a_2, \cdots, s_H, a_H, r_H)$.
\end{comment}

\subsection{Supporting Technical Results}
We start by stating the following variant of the classical  simulation lemma \citep{kearns2002near, agarwal2019reinforcement, foster2021statistical}.%\ayush{Cite some older paper!} 
\ayush{Make a list here on how the parts are connected.} 

\begin{lemma}[Simulation lemma {\citep[Lemma F.3]{foster2021statistical}}]\label{lem:sim}
Let \(\MRPsign = (\cS, P, r, H, s_\top, s_\bot)\) be a markov reward process. Then, the empirical version \(\widehat \MRPsign = (\cS, \widehat P, \widehat r, H, s_\top, s_\bot)\) corresponding to \(\MRPsign\) satisfies: 
%	Suppose $V = \mathrm{DP}(\mS, P, r)$ and $\widehat{V} = \mathrm{DP}(\mS, \widehat{P}, \widehat{r})$, where we assume $P$ to be a transition model ($\forall s\in\mS, \sum_{s'\in\mS} P_{s\to s'} = 1$) and $\widehat{P}$ doesn't need need to be, then we have
	$$|V_{\mathrm{MRP}} - \widehat{V}_{\mathrm{MRP}}|\le \sum_{s\in \mS} d_{\MRPsign}(s)\cdot \left(\sum_{s'\in\mS} |P_{s\to s'} - \widehat{P}_{s\to s'}| + \left|r_{s\to s'} - \widehat{r}_{s\to s'}\right|\right),$$
	% \zeyu{Actually we only need the following inequality in the proof, and we don't need to deal with $d(s)$. Shall we change into this}
	% $$|V - \widehat{V}|\le |\mS|\cdot \sup_{s, s'\in\mS} \left(|P_{s\to s'} - \widehat{P}_{s\to s'}| + \left|r_{s\to s'} - \widehat{r}_{s\to s'}\right|\right),$$
	where $d_{\MRPsign}(s)$ is the probability of reaching $s$ under $\MRPsign$, and \(V_{\mathrm{MRP}}\) and \(\widehat V_{\mathrm{MRP}}\) denotes the value of \(s_\top\) under \(\MRPsign\) and \(\widehat{\MRPsign}\) respectively. 
\end{lemma}

% \subsubsection{Proof of \pref{lem:dynamics_error}} 

The following technical lemma shows that for any policy \(\pi\), the empirical version of policy-specific MRP closely approximates its expected version. 
\begin{lemma}\label{lem:dynamics_error} 
Let \pref{alg:main} be run with the parameters given in Eq.~\pref{eq:alg_parameter}, and consider any iteration of the while loop in \pref{line:while_loop} with the instantaneous set  \(\reachablestates\). Further, suppose that $|\mD_s|\ge \tfrac{\epsilon n_2}{24D}$ for all $s\in \reachablestates$. Then, with probability at least $1 - \delta$, the following hold: 
\begin{enumerate}[label=\((\alph*)\)] 
\item For all \(\pi \in \Pi\), \(s \in \SHP{\pi}\) and \(s' \in \cS_\pi \cup \crl{s_\bot}\),
\begin{align*}
	\max\crl*{|P_{s\to s'}^\pi - \widehat{P}_{s\to s'}^\pi|, \left|r_{s\to s'}^\pi - \widehat{r}_{s\to s'}^\pi\right|} &\leq  \frac{\epsilon}{12D(D+1)}. 
\end{align*}
\item For all \(\pi \in \Pi\) and \(s' \in \cS_\pi \cup \crl{s_\bot}\),
\begin{align*}
		\max \crl{|P_{s_\top\to s'}^\pi - \widehat{P}_{s_\top\to s'}^\pi|, |r^\pi_{s_{\top} \to s'} - \widehat{r}^\pi_{s_{\top} \to s'}|} \le \frac{\epsilon}{12(D+1)^2}.  
\end{align*} 
\end{enumerate}
\end{lemma}

In the sequel, we define the event that the conclusion of \pref{lem:dynamics_error} holds as $\cE_\mathrm{est}$. 

\begin{proof} Fix any \(\pi \in \Pi\). We first prove the bound for $s\in \SHP{\pi}$. Let \(s\) be at layer \(h\). 
Fix any policy \(\pi \in \Pi\), and consider any state \(s' \in \cS_\pi \cup \crl{s_{\bot}}\), where \(s'\) is at layer \(h'\). % Recall that \(\event{s}{s'}{\mS_{\pi}}\) denotes the set of all the trajectories that go from the state \(s\) (at timestep \(h\))  to the state \(s'\) (at timestep \(h'\)) without passing through  any state in \(\cS_\pi\) in between. For the ease of notation, in the rest of the proof, we define \(\setall{\pi}{s}{s'}\) as the set of all the trajectories \(\tau \in \event{s}{s'}{\mS_{\pi}}\) that are consistent with \(\pi\) between time steps \(h\) and \(h'\), i.e.~\(\pi \cons \tau_{h:h'}\). 
Note that since \(\Pi\) is a \((K, D)\)-sunflower, with its core \(\Picore\) and petals \({\crl{\cS_\pi}}_{\pi \in \Pi}\), we must have that any trajectory \(\tau \in \setall{\pi}{s}{s'}\) is also consistent with at least one \(\pi_e \in \Picore\). Furthermore, for any such \(\pi_e\), we have 
	\begin{align*}
		\Pr^{\pi_e}\brk*{\tau_{h:h'} \mid s_h=s} &= \prod_{i=h}^{h'-1} P\brk*{s_{i+1} \mid s_i, \pi_e(s_i), s_h = s} \\ 
		&= \prod_{i=h}^{h'-1} P\brk*{s_{i+1} \mid s_i, \pi(s_i), s_h = s} = \Pr^{\pi} \brk*{\tau_{h:h'} \mid s_h=s},  \numberthis \label{eq:pi-pie} 
	\end{align*} 
	where the second line holds because both \(\pi \cons \tau_{h:h'}\) and \(\pi_e \cons \tau_{h:h'}\). Next, recall from Eq.~\eqref{eq:policy_MRP_dynamics}, that 
\begin{align*}
P_{s\to s'}^\pi &= \EE^\pi \brk*{\ind{\tau \in \setall{\pi}{s}{s'}} \mid s_h = s }.  \numberthis \label{eq:original_MRP_estimate} 
\end{align*}	 

Furthermore, from Eq.~\eqref{eq:empirical_MRP_dynamics}, recall that the empirical estimate \(\widehat{P}_{s\to s'}^\pi \) of \({P}_{s\to s'}^\pi \) is given by :
\begin{align*}
\widehat{P}_{s\to s'}^\pi = \frac{1}{|\mD_s|}\sum_{\tau\in \mD_s}\frac{\ind{ \tau \in \setall{\pi}{s}{s'} }}{\frac{1}{|\Picore|}\sum_{\pi_e\in \Picore}\ind{\pi_e\cons\tau_{h:h'}}},  \numberthis \label{eq:empirical_estimate} 
\end{align*}
where the dataset \(\cD_s\) consists of i.i.d.~samples, and is collected in  lines \ref{line:is1}-\ref{line:is2} in \pref{alg:sample} ($\datacollector$), by first running the policy \(\pi_s\) for \(h\) timesteps and if the trajectory reaches \(s\), then executing \(\pi_e \sim \unif(\Picore)\) for the remaining time steps (otherwise this trajectory is rejected). Let the law of this process be \(q\). We thus note that,

\begin{align*}
\hspace{0.4in} &\hspace{-0.4in} \En_{\tau\sim q} \brk*{\widehat{P}_{s\to s'}^\pi} \\  
&= \En_{\tau \sim q} \brk*{\frac{\ind{\tau \in \setall{\pi}{s}{s'}}}{\frac{1}{|\Picore|}\sum_{\pi_e\in \Picore}\ind{\pi_e\cons\tau_{h:h'}}} \mid s_h = s}    \\ 
&= \sum_{\asedit{\tau \in \setall{\pi}{s}{s'}}} \Pr_q \brk*{\tau_{h:h'} \mid s_h = s} \cdot \frac{1}{\frac{1}{|\Picore|}\sum_{\pi_e\in \Picore} \ind{\pi_e\cons\tau_{h:h'}}}  \\
&\overeq{\proman{1}}  \sum_{\tau \in \setall{\pi}{s}{s'}} \frac{1}{|\Picore|}\sum_{\pi'_e\in\Picore}\asedit{\ind{\tau \in \setall{\pi'_e}{s}{s'}}} \Pr^{\pi'_e} \brk*{\tau_{h:h'}  \mid s_h=s} \cdot \frac{1}{\frac{1}{|\Picore|}\sum_{\pi_e\in \Picore} \ind{\pi_e\cons\tau_{h:h'}}}  \\ 
&\overeq{\proman{2}}  \sum_{\tau \in \setall{\pi}{s}{s'}} \frac{1}{|\Picore|}\sum_{\pi'_e\in\Picore}\Pr^{\pi} \brk*{\tau_{h:h'} \mid s_h=s}  \cdot \frac{\ind{\pi'_e \cons \tau_{h:h'}}}{\frac{1}{|\Picore|}\sum_{\pi_e\in \Picore}  \ind{\pi_e\cons\tau_{h:h'}}}  \\ 
&= \sum_{\tau \in \setall{\pi}{s}{s'}} \Pr^{\pi}\brk*{\tau_{h:h'} \mid s_h=s}   \\  
&\overeq{\proman{3}}  \EE^\pi \brk*{\ind{\tau \in \setall{\pi}{s}{s'}} \mid s_h = s } =  P_{s\to s'}^\pi, 
\end{align*} 
where \(\proman{1}\) follows from the sampling strategy in \pref{alg:sample} after observing \(s_h = s\), and \(\proman{2}\) simply uses the relation \pref{eq:pi-pie} \asedit{since both \(\pi'_e \cons \tau_{h:h'}\) and \(\pi \cons \tau_{h:h'}\) hold}. Finally, in \(\proman{3}\), we use the relation \pref{eq:original_MRP_estimate}.

We have shown that \(\widehat{P}_{s\to s'}^\pi\) is an unbiased estimate of \(P^\pi_{s \to s'}\) for any \(\pi\) and \(s, s' \in \cS_\pi^+\). Thus, using Hoeffding's inequality (\pref{lem:hoeffding}), followed by a union bound, we get that with probability at least \(1 - \delta/4\), for all \(\pi \in \Pi\), \(s \in \SHP{\pi}\),  and \(s' \in \cS_\pi \cup \crl{s_\bot}\), 
\begin{align*}
|\widehat{P}_{s\to s'}^\pi - P_{s\to s'}^\pi|\le K\sqrt{\frac{2\log(4 |\Pi|D(D+1)/\delta)}{|\mD_s|}}, 
\end{align*}
where the additional factor of \(K\) appears because for any \(\tau \in \setall{\pi}{s}{s'}\), there must exist some \(\pi_e \in \Picore\) that is also consistent with \(\tau\) (as we showed above), which implies that each of the terms in Eq.~\eqref{eq:empirical_estimate} satisfies the bound a.s.: 
\begin{align*}
\abs*{\frac{\ind{ \tau \in \setall{\pi}{s}{s'} }}{\tfrac{1}{|\Picore|}\sum_{\pi_e\in \Picore}\ind{\pi_e\cons\tau_{h:h'}}}} &\leq \abs{\Picore} = K. 
\end{align*} 

Since $|\mD_s|\ge \tfrac{\epsilon n_2}{24D}$ by assumption, we have 
\begin{align*}
|\widehat{P}_{s\to s'}^\pi - P_{s\to s'}^\pi|\le  K\sqrt{\frac{48D\log(4|\Pi|D(D+1)/\delta)}{\epsilon n_2}}.
\end{align*}

Repeating a similar argument for the empirical reward estimation in Eq.~\pref{eq:empirical_MRP_dynamics}, we get that with probability at least \(1 - \delta/4\), for all \(\pi \in \Pi\), and \(s \in \SHP{\pi}\) and \(s' \in \cS_\pi \cup \crl{s_\bot}\), we have that 
\begin{align*}
|\widehat{r}_{s\to s'}^\pi - r_{s\to s'}^\pi|\le  K\sqrt{\frac{48D\log(4 |\Pi|D(D+1)/\delta)}{\epsilon n_2}}. 
\end{align*}

Similarly, we can also get for any $\pi\in \Pi$ and $s'\in\mS_\pi\cup\{s_\bot\}$, with probability at least $1 - \delta/2$, 
\begin{align*}
\max\crl*{|\widehat{r}_{s_\top\to s'}^\pi - r_{s_\top\to s'}^\pi|, |\widehat{P}_{s_\top\to s'}^\pi - P_{s_\top\to s'}^\pi|}  &\le K\sqrt{\frac{2\log(4 |\Pi|(D+1)/\delta)}{|\mD_{s_\top}|}} \\ &= K\sqrt{\frac{2\log(4 |\Pi|(D+1)/\delta)}{n_1}},
\end{align*} 
where the last line simply uses the fact that \(\abs{\cD_{s_\top}} = n_1\). 
%and that
%\begin{align*}
%|\widehat{r}_{s_\top\to s'}^\pi - r_{s_\top\to s'}^\pi|\le K\sqrt{\frac{2\log(|\Pi|(D+1)/\delta)}{n_1}}. 
%\end{align*}
%
%	Adopting union bound over all $\pi\in \Pi, s\in\SHP{\pi}, s'\in \{s_\bot\}\cup\mS_{\pi}$, we know that with our choice of $n_1$ and $n_2$, \eqref{eq:error} holds for all $\pi, s, s'$ holds with probability at least $1 - \delta$.
The final statement is due to a union bound on the above results. This concludes the proof of \pref{lem:dynamics_error}.\end{proof}%\gene{it is a bit awkward going between $K$ and $\abs{\Pi_\mathrm{core}}$ in the proof a bunch of times.} 

\begin{lemma}\label{lem:dbar}
Fix a policy $\pi \in \Pi$ and a set of reachable states $\reachablestates$, and consider the policy-specific MRP $\MRPsign^{\pi}_{\reachablestates}$ (as defined by Eqs.~\eqref{eq:policy_MRP_dynamics} and \eqref{eq:policy_MRP_reward}). Then for any $s\in\SRem{\pi}$, the quantity $\bard{\pi}{s} {\SRem{\pi}} = d^{\MRPsign}(s)$, where $d^{\MRPsign}(s)$ is the occupancy of state $s$ in $\MRPsign^{\pi}_{\reachablestates}$. 
\end{lemma}

\begin{proof}
We use $\bar{\tau}$ to denote a trajectory in $\MRPsign^{\pi}_{\reachablestates}$ and $\tau$ to denote a ``corresponding'' (in a sense which will be described shortly) trajectory in the original MDP $M$. For any $s\in\SRem{\pi}$, we have 
\begin{align*}
    d^{\MRPsign}(s) & = \sum_{\bar{\tau} ~\text{s.t.}~s\in \bar{\tau}} \Pr^{\MRPsign} \brk*{\bar{\tau}} = \sum_{k = 0}^{H-1} \sum_{\bar{s}_{1}, \bar{s}_{2}, \cdots, \bar{s}_{k}\in \SHP{\pi}} \Pr^{\MRPsign} \brk*{\bar{\tau} = (s_{\top}, \bar{s}_{1}, \cdots, \bar{s}_{k}, s, s_\bot, \cdots)}.
\end{align*}
The first equality is simply due to the definition of $d^\MRPsign$. For the second equality, we sum up over all possible sequences which start at $s_\top$, pass through some (variable length) sequence of states $\bar{s}_1,\cdots, \bar{s}_k \in \SHP{\pi}$, then reach $s$ and the terminal state $s_\perp$. By definition of the policy-specific MRP, we know that once the MRP transits to a state $s\in\SRem{\pi}$, it must then transit to $s_\bot$ and repeat $s_\bot$ until the end of the episode.

Now fix a sequence $\bar{s}_1,\cdots, \bar{s}_k \in \SHP{\pi}$. We relate the term in the summand to the probability of corresponding trajectories in the original MDP $M$. To avoid confusion, we let $s_{h_1}, \dots, s_{h_k} \in \SHP{\pi}$ denote the corresponding sequence of states in the original MDP, which are unique and satisfy $h_1 < h_2 < \cdots < h_k$. We also denote $s_{h_{k+1}} = s$.

Using the definition of $\MRPsign^\pi_{\reachablestates}$, we write 
\begin{align*} 
\hspace{0.5in}&\hspace{-0.5in} \Pr^{\MRPsign} \brk*{\bar{\tau} = (s_{\top}, \bar{s}_{1}, \cdots, \bar{s}_{k}, s, s_\bot, \cdots)} \\
&= \prod_{i=1}^k \Pr^{M, \pi}\brk*{ \tau_{h_{i}:h_{i+1}} \in \setall{\pi}{s_{h_i}}{s_{h_{i+1}}} \mid \tau[h_i] =s_{h_i} } \\
&= \Pr^{M, \pi}\brk*{ \forall i\in [k+1],~ \tau[h_i] = s_{h_i}, ~\text{and}~ \forall h \in [h_{k+1}] \backslash \crl{h_1, \cdots, h_{k+1}},~ \tau[h] \notin \cS_\pi}.
\end{align*}
Now we sum over all sequences $\bar{s}_1,\cdots, \bar{s}_k \in \SHP{\pi}$ to get \begin{align*}
\hspace{0.08in} &\hspace{-0.08in} d^\MRPsign(s) \\
     &= \sum_{k=0}^{H-1} \sum_{s_{h_1}, \cdots, s_{h_k} \in \SHP{\pi}} \Pr^{M, \pi}\brk*{ \forall i\in [k+1],~ \tau[h_i] = s_{h_i}, ~\text{and}~ \forall h \in [h_{k+1}] \backslash \crl{h_1, \cdots, h_{k+1}},~ \tau[h] \notin \cS_\pi} \\
    &= \Pr^{M, \pi} \brk*{s \in \tau ~\text{and}~ \forall h \in [h_{k+1}-1],~\tau[h] \notin \SRem{\pi}} \\
    &= \Pr^\pi \brk*{\tau \in \event{s_\top}{s}{\SRem{\pi}}} =  \bard{\pi}{s} {\SRem{\pi}}. 
\end{align*}
The second equality follows from the definition of $\SRem{\pi}$, and the last line is the definition of the $\bar{d}$ notation. This concludes the proof of \pref{lem:dbar}.
\end{proof}

%\subsubsection{Proof of \pref{lem:identify}} \ayush{Remove this section names later on} 
\begin{lemma}\label{lem:identify}
With probability at least $1 - 2\delta$,  any $(\bar{s}, \pi)$ that is added into $\cT$ (in \pref{line:add_s} of \pref{alg:main}) satisfies $d^{\pi}(\bar{s})\ge \nicefrac{\epsilon}{12D}$. 
\end{lemma}
\begin{proof} ~ 
For any $(\bar{s}, \pi)\in\cT$, when we collect $\mD_{\bar{s}}$ in \pref{alg:sample}, the probability that a trajectory will be accepted (i.e.~the trajectory satisfies the ``if'' statement in \pref{line:is1}) is exactly $d^{\pi}({\bar{s}})$. Thus, using Hoeffding's inequality (\pref{lem:hoeffding}), with probability at least $1 - \nicefrac{\delta}{D|\Pi|}$, 
\begin{equation*}%\label{eq:ds}
\left|\frac{|\mD_{\bar{s}}|}{n_2} - d^\pi({\bar{s}})\right|\le \sqrt{\frac{2\log(D|\Pi|/\delta)}{n_2}}.\end{equation*} 
Since $|\cT|\le D|\Pi|$, by union bound, the above holds for every $(\bar{s}, \pi) \in \cT$ with probability at least $1-\delta$. Let us denote this event as $\cE_\mathrm{data}$. Under $\cE_\mathrm{data}$, for any $(\bar{s}, \pi)$ that satisfies $d^\pi(\bar{s}) \ge \tfrac{\eps}{12D}$,
\begin{align*}
|\mD_{\bar{s}} | &\ge n_2 d^{\pi}(\bar{s}) - \sqrt{2 n_2 \log(D|\Pi|/\delta)} \ge \frac{\epsilon n_2}{12D} - \frac{\epsilon n_2}{24D} = \frac{\epsilon n_2}{24D}, \numberthis \label{eq:ds}
\end{align*} 
where the second inequality follows by the bound on \(d^\pi(\bar{s})\) and our choice of parameter \(n_2\) in Eq.~\eqref{eq:alg_parameter}. 

In the following, we prove by induction that every $(\bar{s}, \pi)$ that is added into $\cT$ in the while loop from lines \ref{line:while_loop_start}-\ref{line:while_loop_end} in \pref{alg:main} satisfies $d^{\pi}(\bar{s})\ge \tfrac{\epsilon}{12D}$. This is trivially true at initialization when $\cT = \{(s_\top, \mathrm{Null})\}$, since every trajectory starts at the dummy state $s_\top$, for which we have $d^{\mathrm{Null}}(s_\top) = 1$. 

We now proceed to the induction hypothesis. Suppose that in some iteration of the while loop, every tuple $(\bar{s}, \pi) \in \cT$ satisfies $d^{\pi}(\bar{s})\ge \nicefrac{\epsilon}{12D}$, and that \(\prn{ \bar{s}', \pi'}\) is a new tuple that will be added to \(\cT\). We will show that \(\prn{\bar{s}', \pi'}\) will also satisfy $d^{\pi'}({\bar{s}'})\ge \nicefrac{\epsilon}{12D}$. 

Recall that \(\cS_{\pi'}^+ = \cS_{\pi'} \cup \crl{s_\top, s_\bot}\), \(\SHP{\pi'} = \cS_{\pi'}^+ \cap \reachablestates\), and \(\SRem{\pi'} = \cS_{\pi'}^+ \setminus \SHP{\pi'}\). Let \(\MRPsign^{\pi'}_{\reachablestates} =  \mathrm{MRP}(\cS_{\pi'}^+, P^{\pi'}, r^{\pi' }, H, s_\top, s_\bot)\) %\MRP{\pi'}(\SHP{\pi'}, \cS_{\pi'}^+)\)\gene{Did we introduce this notation?} 
    be the policy-specific MRP, where \(P^{\pi'}\) and \(r^{\pi'}\) are defined in Eqs.~\pref{eq:policy_MRP_dynamics} and \pref{eq:policy_MRP_reward} respectively for the policy \(\pi'\). Similarly let $\wh{\MRPsign}^{\pi'}_{\reachablestates} = \mathrm{MRP}(\cS_{\pi'}^+, \widehat P^{\pi'}, \widehat r^{\pi'}, H, s_\top, s_\bot)$ denote the estimated policy-specific MRP, where \(\wh P^{\pi'}\) and \(r^{\pi'}\) are defined using \pref{eq:empirical_MRP_dynamics} and \pref{eq:empirical_MRP_rewards} respectively. Note that for any state \(s \in \SHP{\pi'}\), the bound in \pref{eq:ds} holds. 

    For the rest of the proof, we assume that the event $\cE_\mathrm{est}$, defined in \pref{lem:dynamics_error}, holds (this happens with probability at least $1-\delta$). By definition of $\cE_\mathrm{est}$, we have
	\begin{equation}\label{eq:error1}|P_{s\to s'}^{\pi'} - \widehat{P}_{s\to s'}^{\pi'}|\le \frac{\epsilon}{12D(D+1)},  \qquad \text{for all} \qquad s'\in \mS_{\pi'} \cup\{s_\bot\}.
 \end{equation}
 
\iffalse
Additionally, note that for any state \(s \in \SRem{\bar \pi}\), its occupancy \(d^{\MRPsign_1}(s)\) in \(\MRPsign_1\) satisfies \(d^{\MRPsign_1}(s) = \bard{\bar \pi}{s} {\SRem{\bar \pi}}\)  where $\bard{\bar \pi}{s} {\SRem{\bar \pi}}$ is defined in \eqref{eq:dbar}.\gene{the notation with superscript of $\MRPsign$ for $d$ is inconsistent here, have we introduced this before?} If we use $\bar{\tau}$ to denote a trajectory in $\MRPsign_1$ and $\tau$ to denote a trajectory in the original MDP $M$, this follows from \gene{I'm not sure I follow these calculations, maybe due to inconsistency in notation.}
 	\begin{align*}
		d^{\MRPsign_1}(s) & = \sum_{\bar{\tau} ~\text{s.t.}~s\in \bar{\tau}}P^{\MRPsign_1}(\bar{\tau}) \\
  &= \sum_{s_{h_1} = s_\top, s_{h_2}, \cdots, s_{h_t}\in \SHP{\pi}}P^{\MRPsign_1}(\bar{\tau}\coloneqq (s_\top, s_{h_2}, \cdots, s_{h_t}, s))\\
  &= \sum_{s_{h_1} = s_\top, s_{h_2}, \cdots, s_{h_t}\in \SHP{\pi}}P^{\pi, M}(\tau:\tau\cap \mS_\pi = \{s_{h_2}, \cdots, s_{h_t}, s\})\\
		& = \sum_{s_{h_1} = s_\top, s_{h_2}, \cdots, s_{h_t}\in \SHP{\pi}}\prod_{i=1}^{t} P^{\pi, M}(\tau:\tau\in \event{s_{h_i}}{s_{h_{i+1}}}{\mS_\pi}|\tau[h_i] =s_{h_i})\qquad (s_{h_{t+1}}\coloneqq s)\\
		& = P^{\pi, M}(\tau: \tau[h_i] = s_{h_i}, \forall 1\le i\le t+1, \tau[h]\not\in \mS_\pi, \forall h\neq h_1, \cdots, h_t)\\
		& = P^{\pi, M}(\tau:s\in \tau, s_h\not\in \SRem{\pi}, \forall 1\le h\le h_{t+1}) = P^{\pi, M}[\tau: \tau\in \event{s_\top}{s}{\bar{S}}] = \bard{\pi}{s} {\SRem{\pi}}. 
	\end{align*}\fi

Furthermore, note that $\widehat{d}^{\pi'} (\bar{s}')\leftarrow \estreach(\mS_{\pi'}^+, \wh{\MRPsign}^{\pi'}_{\reachablestates}, \bar{s}')$ since in \pref{alg:dp} we start with \(V(s) = \ind{s = \bar{s}'}\). Furthermore, using  \pref{lem:sim}, we have 
\begin{equation}\label{eq:derror}
\begin{aligned}
    |\widehat{d}^{\pi'}(\bar s') - d^{\MRPsign}(\bar s')| & \le (D+1) \sup_{s\in\SHP{\pi'}, s'\in \mS_{\pi'}\cup\{s_\bot\}}|\widehat{P}_{s\to s'}^{\pi'} - P_{s\to s'}^{\pi'}| \\ 
        &\le \frac{\epsilon}{12D(D+1)}\cdot (D+1) = \frac{\epsilon}{12D}.
\end{aligned}\end{equation}
where the second inequality follows from \eqref{eq:error1}. Additionally, \pref{lem:dbar} states that \(d^{\MRPsign}(\bar s') = \bard{\pi'}{\bar s'} {\SRem{\pi'}}\). Therefore we obtain
$$|\bard{\pi'}{\bar s'} {\SRem{\pi'}} - \widehat{d}^{\pi'}(\bar s')|\le \frac{\epsilon}{12D}.$$
	
Thus, if the new state-policy pair $(\bar s', \pi')$ is added into $\cT$, we will have 
$$\bard{\pi'}{\bar s'}{\SRem{\pi'}}\ge \frac{\epsilon}{6D} - \frac{\epsilon}{12D} = \frac{\epsilon}{12D}.$$
Furthermore, by definition of $\bar{d}$ we have
$$\bard{\pi'}{\bar s'}{\SRem{\pi'}} = \Pr^{\pi'}\brk*{\tau\in \event{s_\top}{\bar s'}{\SRem{\pi'}}}\le 
\Pr^{\pi'}[\bar s'\in \tau] = d^{\pi'}(\bar s'),$$
so we have proved the induction hypothesis $d^{\pi'}(\bar{s}')\ge \nicefrac{\epsilon}{12D}$ for the next round. This concludes the proof of \pref{lem:identify}.

\end{proof}

The next lemma establishes that \pref{alg:main} will terminate after finitely many rounds, and that after termination will have explored all sufficiently reachable states. 
\begin{lemma}\label{lem:while_terminate} With probability at least \(1 - 2\delta\),  
\begin{enumerate}[label=\((\alph*)\)] 
\item The while loop in \pref{line:while_loop} in \pref{alg:main} will terminate after at most $\tfrac{12HD\dimRL(\Pi)}{\epsilon}$ rounds. 
\item  After the termination of the while loop, for any $\pi\in \Pi$, the remaining states $s\in\SRem{\pi}$ that are not added to \(\reachablestates\)
satisfy $\bard{\pi}{s}{\SRem{\pi}}\le \nicefrac{\epsilon}{4D}$.
\end{enumerate}
\end{lemma}
Notice that according to our algorithm, the same state cannot be added multiple times into $\reachablestates$. Therefore, $|\reachablestates| \le D|\Pi|$, and the maximum number of rounds of the while loop is $D|\Pi|$ (i.e., the while loop eventually terminates).  

\begin{proof}
% According to \pref{lem:dynamics_error} and \pref{lem:identify}, \eqref{eq:derror} holds with probability at least $1 - 2\delta$.\gene{a bit weird to refer to equations in other lemmas.}
We prove each part separately.

\begin{enumerate}[label=\((\alph*)\)]
\item 
First, note that from the definition of coverability and \pref{lem:coverability}, we have  
$$\sum_{s\in\mS}\sup_{\pi\in \Pi}d^{\pi}(s)\le HC^\mathsf{cov}(\Pi; M)\le H\dimRL(\Pi).$$
Furthermore, \pref{lem:identify} states that every $(s, \pi_s)\in \cT$ satisfies $d^{\pi_s}(s)\ge \nicefrac{\epsilon}{12D}$. Thus, at any point in \pref{alg:main}, we have
\begin{align*}
\sum_{s\in\reachablestates}\sup_{\pi\in \Pi}d^{\pi}(s) &\geq \sum_{s\in\reachablestates} d^{\pi_s}(s) \geq \abs{\cT} \cdot \frac{\epsilon}{12D}. 
\end{align*}
Since, \(\reachablestates \subseteq \cS\), the two bounds indicate that 
$$|\cT| \le \frac{12HD\dimRL(\Pi)}{\epsilon}.$$

Since every iteration of the while loop adds one new $(s, \pi_s)$ to $\cT$, the while loop terminates after at most $\nicefrac{12HD\dimRL(\Pi)}{\epsilon}$ many rounds.

\item  We know that once the while loop has terminated, for every $\pi\in \Pi$ and $\bar{s}\in\SRem{\pi}$, we must have $\widehat{d}^\pi(\bar{s})\le \nicefrac{\epsilon}{6D}$, or else the condition in \pref{line:test_and_add} in \pref{alg:main} is violated. 

Fix any such $(\bar{s},\pi)$ pair. Inspecting the proof of \pref{lem:identify}, we see that 
\begin{align*}
|\bard{\pi}{\bar s} {\SRem{\pi}} - \widehat{d}^{\pi}(\bar s)|\le \frac{\epsilon}{12D}.
\end{align*}
To conclude,  we get 
$$\bard{\pi}{s}{\SRem{\pi}}\le \frac{\epsilon}{6D} + \frac{\epsilon}{12D} = \frac{\epsilon}{4D}.$$
\end{enumerate}
\end{proof} 

\begin{lemma} \label{lem:eval_error}
Suppose that the conclusions of Lemmas \ref{lem:dynamics_error} and \ref{lem:while_terminate} hold. Then for every $\pi \in \Pi$, the estimated value $\widehat{V}^\pi$ computed in \pref{alg:main} satisfies
	$$|\widehat{V}^\pi - V^\pi|\le \epsilon.$$ 
\end{lemma}
\begin{proof}
We will break up the proof into two steps. First, we show that for any \(\pi\), the value estimate $\wh{V}^\pi$ obtained using the empirical policy-specific MRP $\wh \MRPsign^\pi_{\reachablestates}$ is close to its value in the  policy-specific MRP $\MRPsign^\pi_{\reachablestates}$, as defined via~\eqref{eq:policy_MRP_dynamics} and \eqref{eq:policy_MRP_reward}. We denote this quantity as $V_{\mathrm{MRP}}^\pi$. Then, we will show that $V_{\mathrm{MRP}}^\pi$ is close to $V^\pi$, the value of the policy \(\pi\) in the original MDP $M$.  

\paragraph{Part 1: $\widehat{V}^\pi$  is close to  $V_{\mathrm{MRP}}^\pi$.} Note that the output $\widehat{V}^\pi$ of \pref{alg:eval} is exact the value function of MRP $\widehat \MRPsign^\pi_{\reachablestates}$ defined by Eqs.~\eqref{eq:empirical_MRP_dynamics} and \eqref{eq:empirical_MRP_rewards}. When $D = 0$, by part (b) of \pref{lem:dynamics_error}, we obtain
$$|\widehat{V}^{\pi} - V_{\mathrm{MRP}}^\pi| = |\widehat{r}^{\pi}_{s_\top\to s_\bot} - r^\pi_{s_\top\to s_\bot}|\le \frac{\epsilon}{12(D+1)^2}\le \frac{\epsilon}{2}.$$

 When $D \ge 1$, using \pref{lem:dynamics_error}, we have 
% When $D\ge 1$, we have $\frac{\epsilon}{12D(D+1)}\le \frac{\epsilon}{8(D+2)}$. Additionally, according to \pref{lem:dynamics_error}, we have
\begin{align*}
    &|r_{s_\top\to s'}^\pi - \widehat{r}_{s_\top\to s'}^\pi| \le \frac{\epsilon}{12(D+1)^2}, \quad
    |P_{s_\top\to s'}^\pi - \widehat{P}_{s_\top\to s'}^\pi| \le \frac{\epsilon}{12(D+1)^2}, \quad && \forall s'\in \mS_\pi \cup \crl{s_\bot} \\
    &|r_{s\to s'}^\pi - \widehat{r}_{s\to s'}^\pi| \le \frac{\epsilon}{12D(D+1)}, \quad
    |P_{s\to s'}^\pi - \widehat{P}_{s\to s'}^\pi| \le \frac{\epsilon}{12D(D+1)}, \quad && \forall s\in \SHP{\pi}, s'\in \mS_\pi^+ \cup \crl{s_\bot}. 
\end{align*}
By the simulation lemma (\pref{lem:sim}), we get
\begin{align*}
    |\widehat{V}^{\pi} - V_{\mathrm{MRP}}^\pi| & \le 2(D+2)\max_{s, s'\in\mS_\pi^+}\left(\left|P_{s\to s'}^\pi - \widehat{P}_{s\to s'}^\pi\right| + \left|r_{s\to s'}^\pi - \widehat{r}_{s\to s'}^\pi\right|\right)\\
    & \le 2(D+2)\left(\frac{\epsilon}{12D(D+1)} + \frac{\epsilon}{12D(D+1)}\right)\le \frac{\epsilon}{2}.
\end{align*}

\paragraph{Part 2 : $V_{\mathrm{MRP}}^\pi$ is close to $V^\pi$.} As in the proof of \pref{lem:dbar}, let us consider different trajectories $\bar{\tau}$ that are possible in $\MRPsign^\pi_{\reachablestates}$. We can represent $\bar{\tau} = (s_\top, \bar{s}_1,\cdots, \bar{s}_k, s_\bot, \cdots)$ where the states $\bar{s}_1,\cdots, \bar{s}_k$ are distinct and all except possibly $\bar{s}_k$ belong to $\SHP{\pi}$, and the states after $s_\bot$ are just repeats of $s_\bot$ until the end of the episode. Let $s_{h_1}, s_{h_2}, \dots, s_{h_k}$ be the same sequence (in the original MDP $M$) Again, we have %\gene{I had to change up some of the proof, which might have hampered readability. The issue is that trajectories are possible in $\MRPsign$ which contain a single state in $\SRem{\pi}$ that have nonzero reward due to the way rewards are being defined.}
\begin{align*}
    \hspace{1in}&\hspace{-1in}\Pr^\MRPsign[\bar{\tau} = (s_\top, \bar{s}_1,\cdots, \bar{s}_k, s_\bot, \cdots)] \\
    &= \Pr^{\pi}\brk*{ \forall i\in [k],~ \tau[h_i] = s_{h_i}, ~\text{and}~ \forall h \in [H] \backslash \crl{h_1, \cdots, h_{k}},~ \tau[h] \notin \cS_\pi}, 
\end{align*}
where recall that \( \Pr^\MRPsign\) denotes probability under the $\MRPsign^\pi_{\reachablestates}$, and \(\Pr^{\pi}\) denotes the probability under trajectories drawn according to \(\pi\) in the underlying MDP; \( \EE^\MRPsign\)  and \(\EE^{\pi}\) are defined similarly. 

% Additionally for any $s_{h_1} = s_\top, s_{h_2}, \cdots, s_{h_{t-1}}, s_{h_t} = s_\bot\in\{s_\bot\}\cup \SHP{\pi}$, the probability of seeing trajectory $\bar{\tau} = (s_{h_1}, s_{h_2}, \cdots, s_{h_t})$ in $\MRPsign^\pi_{\reachablestates}$ is\gene{notation consistency.}
% 	\begin{align*}
% 		\quad P^{\MRPsign^\pi_{\reachablestates}}(\bar{\tau} = (s_{h_1}, s_{h_2}, \cdots, s_{h_t})) &= \prod_{i=1}^{t-1}P_{s_{h_i}\to s_{h_{i+1}}}^\pi\\
% 		& = \prod_{i=1}^{t-1} P^{\pi, M}(\tau:\tau\in \event{s_{h_i}}{s_{h_{i+1}}}{\mS_\pi}|\tau[h_i] =s_{h_i})\\
% 		& = P^{\pi, M}(\tau: \tau[h_i] = s_{h_i}, \forall 1\le i\le t, \tau[h]\not\in \mS_\pi, \forall h\neq h_1, \cdots, h_t).
% 	\end{align*}
Furthermore, the expectation of rewards we collected in $\MRPsign^\pi_{\reachablestates}$ with trajectories $\bar{\tau}$ is
\begin{align*}
 \hspace{1in}&\hspace{-1in} \EE^{\MRPsign}\brk*{ R[\bar{\tau}] \ind{ (s_\top, \bar{s}_1,\cdots, \bar{s}_k, s_\bot, \cdots) } } \\
= &\EE^{\pi}\brk*{ R[\tau]\ind{ \forall i\in [k],~ \tau[h_i] = s_{h_i}, ~\text{and}~ \forall h \in [H] \backslash \crl{h_1, \cdots, h_{k}},~ \tau[h] \notin \cS_\pi }}. 
\end{align*}

Next, we sum over all possible trajectories. However, note that the only trajectories that are possible in $M$ whose corresponding trajectories are \emph{not accounted for} in $\MRPsign^\pi_{\reachablestates}$ are precisely those that visit states in $\cS_\pi$, after visiting some $s_{h_k}$ in the remaining states $\SRem{\pi}$ (since, by construction, the MRP transitions directly to $s_\bot$ after encountering a state in $\SRem{\pi}$). Thus, 
	% Summing over all possible $s_{h_1}, \cdots, s_{h_t}$, we will get
$$V_{\mathrm{MRP}}^\pi = \EE^{\pi}\brk*{ R[\tau] \prn*{\ind{\tau \cap \SRem{\pi} = \emptyset} + \ind{\exists k\in [H]: s_{h_k} \in \SRem{\pi} \text{ and } \forall h > h_k: s_h \notin \cS_\pi} }},$$ 
where the first term corresponds to trajectories that do not pass through \(\SRem{\pi}\), and the second term corresponds to trajectories that passes through some state in \(\SRem{\pi}\) but then does not go through any other state in \(\cS_\pi\). On the other hand, 
\begin{align*}
V^\pi = \EE^{\pi}\brk*{ R[\tau]}. 
\end{align*}

Clearly, $V_{\mathrm{MRP}}^\pi \le V^{\pi}$. Furthermore, we also have
\begin{align*}
    V^{\pi} - V_{\mathrm{MRP}}^\pi
    &= \EE^{\pi}\brk*{R[\tau] \ind{\tau \cap \SRem{\pi} \ne \emptyset} - \ind{\exists k\in [H]: s_{h_k} \in \SRem{\pi} \text{ and } \forall h > h_k: s_h \notin \cS_\pi}}  \\
     &\le \EE^{\pi}\brk*{R[\tau] \ind{\tau \cap \SRem{\pi} \ne \emptyset}} \\ 
    &\le D \cdot \frac{\eps}{4D} = \frac{\eps}{4},
\end{align*}
where the first inequality follows by just ignoring the second indicator term, and the second inequality follows by taking a  union bound over all possible values of $\SRem{\pi}$ as well as the conclusion of \pref{lem:while_terminate}. 

Putting it all together, we get that 
\begin{align*}
 |\widehat{V}^\pi - V^\pi|\le |V^{\pi} - V_{\mathrm{MRP}}^\pi| + |\widehat{V}^{\pi} - V_{\mathrm{MRP}}^\pi|\le \frac{\epsilon}{4} + \frac{\epsilon}{2} < \epsilon.  
\end{align*}
This concludes the proof of \pref{lem:eval_error}.
\end{proof}
	% Hence since $\forall s\in\SRem{\pi}, \bard{\pi}{s}{\SRem{\pi}}\le \frac{\epsilon}{4D}$, we get
	% \begin{align*}
	% 	|V^{\pi} - V_{\mathrm{MRP}}^\pi| & = \EE^{\pi, M}\left[R[\tau]\right] - \EE^{\pi, M}\left[R[\tau]\ind{\tau\cap \SRem{\pi} = \emptyset}\right]\\
	% 	& = \EE^{\pi, M}\left[R[\tau]\ind{\tau\cap \SRem{\pi} \neq \emptyset}\right]\\
	% 	& = \sum_{s\in\SRem{\pi}}\EE^{\pi, M}\left[R[\tau]\ind{s\in\tau, \tau[0:s]\cap \SRem{\pi} = \emptyset}\right]\\
	% 	& \le \sum_{s\in\SRem{\pi}}\bard{\pi}{s}{\SRem{\pi}}\le D\cdot \frac{\epsilon}{4D} = \frac{\epsilon}{4},
	% \end{align*}
	% which indicates that
	% $$|\widehat{V}^\pi - V^\pi|\le |V^{\pi} - V_{\mathrm{MRP}}^\pi| + |\widehat{V}^{\pi} - V_{\mathrm{MRP}}^\pi|\le \frac{\epsilon}{2} + \frac{\epsilon}{4} < \epsilon.$$
% \end{proof}

\subsection{Proof of \pref{thm:sunflower}} 
%\begin{proof}[Proof of \pref{thm:sunflower}]
We assume the events defined in Lemmas \ref{lem:dynamics_error}, \ref{lem:identify} and \ref{lem:while_terminate} hold (which happens with probability at least $1 - 2 \delta$). With our choices of $n_1, n_2$ in Eq.~\eqref{eq:alg_parameter}, the total number of samples used in our algorithm is at most
	% $$n_1 = \frac{C_1(D+1)^4K^2\log(|\Pi|(D+1)/\delta)}{\epsilon^2}, \quad n_2 = \frac{C_2D^3(D+1)^2K^2\log(|\Pi|(D+1)^2/\delta)}{\epsilon^3},$$
	% if further noticing that the while loop runs at most $\frac{12HD\dimRL(\Pi)}{\epsilon}$ rounds (\pref{lem:while_terminate}), the total number of samples used in our algorithm is upper bounded by
$$n_1 + n_2\cdot \frac{12HD\dimRL(\Pi)}{\epsilon} = \widetilde{\cO}\prn*{\prn*{\frac{1}{\epsilon^2} + \frac{HD^6 \dimRL(\Pi)}{\epsilon^4}} \cdot K^2 \log\frac{|\Pi|}{\delta}}.$$
After the termination of the while loop, we know that for any policy \(\pi \in \Pi\) and $s\in\SRem{\pi}$ we have 
$$\bard{\pi}{s}{\SRem{\pi}}\le \frac{\epsilon}{4D}.$$
Therefore, by \pref{lem:eval_error}, we know for every $\pi\in \Pi$, $|\widehat{V}^\pi - V^\pi|\le \epsilon$. Hence the output policy $\widehat{\pi} \in \argmax_\pi \hV^\pi$ satisfies
$$\max_{\pi\in\Pi} V^\pi - V^{\widehat{\pi}}\le 2\epsilon + \hV^\pi - \hV^{\widehat{\pi}}\le 2\epsilon.$$
Rescaling $\epsilon$ by $2\epsilon$ and $\delta$ by $2\delta$ concludes the proof of \pref{thm:sunflower}.\qed
% \end{proof} 

\subsection{Sunflower Property is Insufficient By Itself}
We give an example of a policy class $\Pi$ for which the sunflower property holds for $K,D = \poly(H)$ but $\dimRL(\Pi) = 2^H$. Therefore, in light of \pref{thm:generative_lower_bound}, the sunflower property by itself cannot ensure statistically efficient agnostic PAC RL in the online access model. 

The example is as follows: Consider a binary tree MDP with $2^H-1$ states and action space $\cA = \crl{0,1}$. The policy class $\Pi$ will be able to get to every $(s,a)$ pair in layer $H$. To define the policies, we consider each possible trajectory $\tau = (s_1, a_1, \cdots, s_H, a_H)$ and let: 
\begin{align*} 
    \Pi \coloneqq \crl*{\pi_{\tau} : \pi_\tau(s) = \begin{cases}
        a_i &\text{if } s_i \in \tau,\\
        0 &\text{otherwise}, 
    \end{cases}}.
\end{align*}
Thus it is clear that $\dimRL(\Pi) = 2^H$, but the sunflower property holds with $K=1$, $D=H$ by taking $\Picore = \crl{\pi_0}$ (the policy which always picks $a=0$).

\newpage
\section{Infinite Policy Classes}\label{app:infinite_policy_classes}

In this section we discuss the extensions of our results to infinite policy classes.

\subsection{Definitions and Preliminary Lemmas}
We will state our results in terms of the Natarajan dimension, which is a generalization of the VC dimension used to study multiclass learning. We note that the results in this section could be stated in terms of other complexity measures from multiclass learning such as the graph dimension and DS dimension \citep[see, e.g.,][]{natarajan1989learning, shalev2014understanding, daniely2014optimal, brukhim2022characterization}; for simplicity we analyze guarantees in terms of the Natarajan dimension. 

\begin{definition}[Natarajan Dimension \citep{natarajan1989learning}]
Let $\cX$ be an instance space and $\cY$ be a finite label space. Given a class $\cH \subseteq \cY^\cX$, we define its \emph{Natarajan dimension}, denoted $\natarajan(\cH)$, to be the maximum cardinality of a set $C \subseteq \cX$ that satisfies the following: there exists $h_0, h_1: C \to \cY$ such that (1) for all $x \in C$, $h_0(x) \ne h_1(x)$, and (2) for all $B \subseteq C$, there exists $h \in \cH$ such that for all $x \in B$, $h(x) = h_0(x)$ and for all $x \in C \backslash B$, $h(x) = h_1(x)$.
\end{definition}

A notation we will use throughout is the projection operator. For a hypothesis class $\cH \subseteq \cY^\cX$ and a finite set $X = (x_1, \cdots, x_n) \in \cX^n$, we define the projection of $\cH$ on to $X$ as
\begin{align*}
    \cH\big|_X \coloneqq \crl*{\prn*{h(x_1),\cdots, h(x_n)} : h \in \cH }. 
\end{align*}

\begin{lemma}[Sauer's Lemma for Natarajan Classes \citep{haussler1995generalization}]\label{lem:sauer-natarajan}
Given a hypothesis class $\cH \subseteq \cY^\cX$ with $\abs{\cY} = K$ and $\natarajan(\cH) \le d$, we have for every $X = (x_1, \cdots, x_n) \in \cX^n$,
\begin{align*} 
    \abs[\Big]{\cH\big|_X} \le \prn*{\frac{ne (K+1)^2}{2d}}^d. 
\end{align*} 
\end{lemma} 

\begin{theorem}[Multiclass Fundamental Theorem \citep{shalev2014understanding}]\label{thm:multiclass-fun}
For any class $\cH \subseteq \cY^\cX$ with $\natarajan(\cH) = d$ and $\abs{\cY} = K$, the minimax sample complexity of $(\eps,\delta)$ agnostic PAC learning $\cH$ can be bounded as 
\begin{align*}
    \Omega\prn*{\frac{d + \log(1/\delta)}{\eps^2}} \le n(\Pi; \eps, \delta) \le \cO\prn*{ \frac{d \log K + \log(1/\delta)}{\eps^2}}.
\end{align*}
\end{theorem}

\begin{definition}[Pseudodimension] Let $\cX$ be an instance space. Given a hypothesis class $\cH \subseteq \bbR^\cX$, its pseudodimension, denoted $\pseudo(\cH)$, is defined as $\pseudo(\cH) \coloneqq \mathrm{VC}(\cH^\plus)$, where $\cH^\plus \coloneqq \crl*{(x, \theta) \mapsto \ind{h(x) \le \theta} : h \in \cH}$.
\end{definition}

\begin{definition}[Covering Numbers] Given a hypothesis class $\cH \subseteq \bbR^\cX$, $\alpha > 0$, and $X = (x_1, \cdots, x_n) \in \cX^n$, the covering number $\cN_1(\cH, \alpha, X)$ is the minimum cardinality of a set $C \subset \bbR^n$ such that for any $h \in \cH$ there exists a $c \in C$ such that $\tfrac{1}{n} \sum_{i=1}^n \abs{h(x_i) - c_i} \le \alpha$.
\end{definition}

\begin{lemma}[\cite{jiang2017contextual}, see also \cite{pollard2012convergence, luc1996probabilistic}]\label{lem:uniform-convergence}
Let $\cH \subset [0,1]^\cX$ be a real-valued hypothesis class, and let $X = \prn*{x_1, \cdots, x_n}$ be i.i.d.~samples drawn from some distribution $\cD$ on $\cX$. Then for any $\alpha > 0$
\begin{align*}
    \Pr \brk*{\sup_{h \in \cH} \abs*{\frac{1}{n}\sum_{i=1}^n h(x_i) - \En[h(x)]} > \alpha} \le 8 \En \brk*{\cN_1(\cH, \alpha/8, X)} \cdot \exp \prn*{ -\frac{n \alpha^2}{128} }.
\end{align*}
Furthermore if $\mathrm{Pdim}(\cH) \le d$ then we have the bound 
\begin{align*}
\Pr \brk*{\sup_{h \in \cH} \abs*{\frac{1}{n}\sum_{i=1}^n h(x_i) - \En[h(x)]} > \alpha} \le 8e(d+1) \prn*{\frac{16e}{\alpha}}^d \cdot \exp \prn*{ -\frac{n \alpha^2}{128} },
\end{align*}
which is at most $\delta$ as long as $n \ge \tfrac{128}{\alpha^2} \prn*{d \log \tfrac{16e}{\alpha} + \log(8e(d+1)) + \log \frac{1}{\delta}}$.
\end{lemma}

\subsection{Generative Model Lower Bound}
First we address the lower bound. Observe that it is possible to achieve a lower bound that depends on $\natarajan(\Pi)$ with the following construction. First, identify the layer $h\in [H]$ such that the witnessing set $C$ contains the maximal number of states in $\cS_h$; by pigeonhole principle there must be at least $\natarajan(\Pi)/H$ such states in layer $h$. Then, we construct an MDP which ``embeds'' a hard multiclass learning problem at layer $h$ over these states. A lower bound of $\Omega \prn*{ \tfrac{\natarajan(\Pi)}{H\eps^2} \cdot \log \tfrac{1}{\delta} }$ follows from \pref{thm:multiclass-fun}.

By combining \pref{thm:generative_lower_bound} with the above we get the following corollary. 

\begin{corollary}[Lower Bound for Generative Model with Infinite Policy Classes] 
For any policy class $\Pi$, the minimax sample complexity $(\eps,\delta)$-PAC learning $\Pi$ is at least
\begin{align*}
     \ngen(\Pi; \eps, \delta) \ge  \Omega \prn*{\frac{\dimRL(\Pi) + \natarajan(\Pi)/H}{\eps^2} \cdot \log \frac{1}{\delta}}.
\end{align*}
\end{corollary}

Again, since the generative model setting is easier than online RL, this lower bound also extends to the online RL setting.

Our bound is additive in $\dimRL(\Pi)$ and $\natarajan(\Pi)$; we do not know if it is possible to strengthen this to be a product of the two factors, as we will achieve in the upper bound in the next section.

\subsection{Generative Model Upper Bound}
For the upper bounds, we can replace the dependence on $\log \abs{\Pi}$ with $\natarajan(\Pi)$ (and additional log factors). In particular, we can modify the analysis of the $\mathsf{TrajectoryTree}$ to account for infinite policy classes. Recall that our analysis of $\mathsf{TrajectoryTree}$ required us to prove a uniform convergence guarantee for the estimate $\wh{V}^\pi$: with probability at least $1-\delta$, for all $\pi \in \Pi$, we have $\abs{\wh{V}^\pi - V^\pi} \lesssim \eps$. We previously used a union bound over $\abs{\Pi}$, which gave us the $\log \abs{\Pi}$ dependence. Now we sketch an argument to replace it with $\natarajan(\Pi)$. 

Let $\cT$ be the set of all possible trajectory trees. We introduce the notation $v^\pi: \cT \to \bbR$ to denote the function that takes as input a trajectory tree $\wh{T}$ (for example, as sampled by $\mathsf{TrajectoryTree}$) and returns the value of running $\pi$ on it. Then we can rewrite the desired uniform convergence guarantee:
\begin{align*}
    \text{w.p.~at least}~1-\delta, \quad \sup_{\pi \in \Pi} ~ \abs*{ \frac{1}{n} \sum_{i=1}^n v^\pi(\wh{T}_i) - \En\brk*{ v^\pi(\wh{T}) } } \le \eps. \numberthis\label{eq:unif-convergence-statement}
\end{align*}
In light of \pref{lem:uniform-convergence}, we will compute the pseudodimension for the function class $\cV^\Pi = \crl*{v^\pi: \pi \in \Pi}$. Define the subgraph class
\begin{align*}
    \cV^{\Pi,+} \coloneqq \crl*{(\wh{T}, \theta) \mapsto \ind{v^\pi(\wh{T}) \le \theta} : \pi \in \Pi} \subseteq \crl{0,1}^{\cT \times \bbR}
\end{align*}
By definition, $\mathrm{Pdim}(\cV^\Pi) = \mathrm{VC}(\cV^{\Pi,+})$. Fix any $X = \crl*{(\wh{T}_1, \theta_1), \cdots, (\wh{T}_d, \theta_d)} \in (\cT \times \bbR)^d$. In order to show that $\mathrm{VC}(\cV^{\Pi,+}) \le d$ for some value of $d$ it suffices to prove that $\abs{\cV^{\Pi,+}\big|_X} < 2^d$.

For any index $t\in [d]$, we also denote $\pi(\vec{s}_i) \in \cA^{\le H \dimRL(\Pi)}$ to be the vector of actions selected by $\pi$ on all $\Pi$-reachable states in $\wh{T}_i$ (of which there are at most $H \cdot \dimRL(\Pi)$). We claim that
\begin{align}\label{eq:bound-pseudo}
    \abs[\Big]{\cV^{\Pi,+}\big|_{X}} \le \abs*{ \crl*{\prn*{\pi(\vec{s}_1), \cdots, \pi(\vec{s}_d) } : \pi \in \Pi} } =: \abs[\Big]{\Pi\big|_X}. 
\end{align}
This is true because once the $d$ trajectory trees are fixed, for every $\pi \in \Pi$, the value of the vector $\cV^{\crl{\pi}, +}\big|_{X} \in \crl{0,1}^d$ is determined by the trajectory that $\pi$ takes in every trajectory tree. This in turn is determined by the assignment of actions to every reachable state in all the $d$ trajectory trees, of which there are at most $\dimRL(\Pi) \cdot H \cdot d$ of. Therefore, we can upper bound the size of $\cV^{\Pi, +}\big|_{X}$ by the number of ways any $\pi \in \Pi$ assign actions to every state in $\wh{T}_1, \cdots, \wh{T}_d$.

Applying \pref{lem:sauer-natarajan} to Eq.~\eqref{eq:bound-pseudo}, we get that 
\begin{align*}
    \abs[\Big]{\cV^{\Pi,+}\big|_{X}} \le \prn*{\frac{H \dimRL(\Pi) d \cdot e \cdot  (A+1)^2}{2 \natarajan(\Pi)}}^{\natarajan(\Pi)}. 
\end{align*}
For the choice of $d = \wt{\cO}\prn*{\natarajan(\Pi)}$, the previous display is at most $2^d$, thus proving the bound on $\pseudo(\cV^\Pi)$. Lastly, the bound can be plugged back into \pref{lem:uniform-convergence} to get a bound on the error of $\mathsf{TrajectoryTree}$: the statement in Eq.~\eqref{eq:unif-convergence-statement} holds using
\begin{align*}
    n = \wt{\cO}\prn*{ H \dimRL(\Pi) \cdot \frac{\natarajan(\Pi) + \log \frac{1}{\delta}}{\eps^2}} \quad \text{samples}.
\end{align*}
This in turn yields a guarantee on $\wh{\pi}$ returned by $\mathsf{TrajectoryTree}$.

\subsection{Online RL Upper Bound} 
\ayush{I need to take another careful pass on this section!} 
The modified analysis for the online RL upper bound (\pref{thm:sunflower}) proceeds similarly; we sketch the ideas below.

There are two places in the proof of \pref{thm:sunflower} which require a union bound over $\abs{\Pi}$: the event $\cE_\mathrm{est}$ (defined by \pref{lem:dynamics_error}) that the estimated transitions and rewards of the MRPs are close to their population versions, and the event $\cE_\mathrm{data}$ (defined by \pref{lem:identify}) that the datasets collected are large enough. The latter is easy to address, since we can simply modify the algorithm's while loop to break after $\cO\prn*{\tfrac{HD\dimRL(\Pi)}{\epsilon}}$ iterations and union bound over the size of the set $\abs{\cT}$ instead of the worst-case bound on the size $D \abs{\Pi}$. For $\cE_\mathrm{data}$, we follow a similar strategy as the analysis for the generative model upper bound.

Fix a state $s$. Recall that the estimate for the probability transition kernel in the MDP in Eq.~\eqref{eq:empirical_MRP_dynamics} takes the form
\begin{align*}
	\widehat{P}_{s\to s'}^\pi = \frac{1}{|\mD_s|} \sum_{\tau\in \mD_s}\frac{\ind{\pi\cons\tau_{h:h'}}}{\tfrac{1}{|\Picore|}\sum_{\pi'\in \Picore}\ind{\pi_e\cons\tau_{h:h'}}}\ind{\tau\in \setall{\pi}{s}{s'} }.
\end{align*}
(The analysis for the rewards is similar, so we omit it from this proof sketch.)

We set up some notation. Define the function $p^{\pi}_{s\to s'}: (\cS\times \cA \times \bbR)^H \to [0, \abs{\Picore}]$ as
\begin{align}\label{eq:is-function-def}
    p^\pi_{s\to s'}(\tau) \coloneqq \frac{\ind{\pi\cons\tau_{h:h'}}}{\tfrac{1}{|\Picore|}\sum_{\pi'\in \Picore}\ind{\pi_e\cons\tau_{h:h'}}}\ind{\tau\in \setall{\pi}{s}{s'} },
\end{align}
with the implicit restriction of the domain to trajectories $\tau$ for which the denominator is nonzero. We have $\En[p^\pi_{s\to s'}(\tau)] = P^\pi_{s\rightarrow s'}$. Also let $\Pi_s = \crl*{\pi \in \Pi: s\in \cS_\pi}$. 

Restated in this notation, our objective is to show the uniform convergence guarantee
\begin{align}\label{eq:uniform-convergence-probabilities}
    \text{w.p.~at least}~ 1-\delta, \quad \sup_{\pi \in \Pi_s, s' \in \cS_\pi} \Big|\frac{1}{|\mD_s|} \sum_{\tau\in \mD_s} p^\pi_{s\to s'}(\tau) - \En[p^\pi_{s\to s'}(\tau)]\Big| \le \eps.
\end{align}
Again, in light of \pref{lem:uniform-convergence}, we need to compute the pseudodimension for the function class $\cP^{\Pi_s} = \crl*{p^\pi_{s\to s'}: \pi \in \Pi_s,  s' \in \cS_\pi}$, since these are all possible transitions that we might use the dataset $\cD_s$ to evaluate. Define the subgraph class
\begin{align*}
    \cP^{\Pi_s, +} \coloneqq \crl{(\tau, \theta) \mapsto \ind{p^\pi_{s\to s'}(\tau) \le \theta}: \pi \in \Pi_s, s' \in \cS_\pi }.
\end{align*}
Fix the set $X = \crl*{(\tau_1, \theta_1), \cdots, (\tau_d, \theta_d)} \in ((\cS \times \cA \times \bbR)^H \times \bbR)^d$, where the trajectories $\tau_1, \cdots, \tau_d$ pass through $s$. We also denote $\cS_X$ to be the union of all states which appear in $\tau_1, \cdots, \tau_d$. In order to show a bound that $\pseudo(\cP^{\Pi_s}) \le d$ it suffices to prove that $\abs{\cP^{\Pi_s, +}\big|_X} < 2^d$.

We first observe that
\begin{align*}
    \abs{\cP^{\Pi_s, +}\big|_X} 
    &\le 1 + \sum_{s' \in \cS_X} \abs*{ \crl*{  \prn*{ \ind{p^\pi_{s \to s'}(\tau_1) \le \theta_1 }, \cdots, \ind{p^\pi_{s \to s'}(\tau_d) \le \theta_d }  }  : \pi \in \Pi_s } }.
\end{align*}
The inequality follows because for any choice $s' \notin \cS_X$, we have
\begin{align*}
    \prn*{ \ind{p^\pi_{s \to s'}(\tau_1) \le \theta_1 }, \cdots, \ind{p^\pi_{s \to s'}(\tau_d) \le \theta_d }  } = \vec{0},
\end{align*}
no matter what $\pi$ is, contributing at most 1 to the count. Furthermore, once we have fixed $s'$ and the $\crl{\tau_1, \cdots, \tau_d}$ the quantities $\tfrac{1}{|\Picore|}\sum_{\pi'\in \Picore}\ind{\pi_e\cons\tau_{i, h:h'}}$ for every $i \in [d]$ are constant (do not depend on $\pi$), so we can reparameterize $\theta_i' \coloneqq \theta_i \cdot \tfrac{1}{|\Picore|}\sum_{\pi'\in \Picore}\ind{\pi_e\cons\tau_{i, h:h'}} $ to get:
\begin{align*}
\abs{\cP^{\Pi_s, +}\big|_X}  &\le 1 + \sum_{s' \in \cS_X} \abs*{ \crl*{ (  b_1(\pi) , \cdots,
 b_d(\pi) )  : \pi \in \Pi } }, \numberthis\label{eq:projection-bound}\\
&\quad \text{where} \quad b_i(\pi) \coloneqq \ind{\ind{\pi\cons\tau_{i, h:h'}} \ind{\tau_i \in \setall{\pi}{s}{s'} } \le \theta_i'}.
\end{align*}
Now we count how many values the vector $(b_1(\pi), \cdots, b_d(\pi))$ can take for different $\pi \in \Pi_s$. Without loss of generality, we can (1) assume that the $\theta_i'=0$ (since a product of indicators can only take values in $\crl{0,1}$, and if $\theta' \ge 1$ then we must have $b_i(\pi) = 1$ for every $\pi$), and (2) $s' \in \tau_i$ for each $i \in [d]$ (otherwise $b_i(\pi) = 0$ for every $\pi \in \Pi$). So we can rewrite $b_i(\pi) = \ind{\pi\cons\tau_{i, h:h'}} \ind{\tau_i \in \setall{\pi}{s}{s'} }$. For every fixed choice of $s'$ we upper bound the size of the set as:
\begin{align*}
&\abs*{ \crl*{ \prn*{  b_1(\pi) , \cdots,
 b_d(\pi) }  : \pi \in \Pi } } \\ \overleq{\proman{1}}  & \abs*{ \crl*{ \prn*{ \ind{\pi\cons\tau_{1, h:h'}}, \cdots, \ind{\pi\cons\tau_{d, h:h'}} } : \pi \in \Pi }} \\
 &\quad\quad \times \abs*{ \crl*{ \prn*{ \ind{\tau_1 \in \setall{\pi}{s}{s'} }, \cdots, \ind{\tau_d \in \setall{\pi}{s}{s'} } } : \pi \in \Pi }} \\
 \overleq{\proman{2}} & \abs*{ \crl*{ \prn*{\pi(s^{(1)}), \pi(s^{(2)}), \cdots, \pi(s^{(dH)}) } : \pi \in \Pi } } \times \abs*{ \crl*{  \prn*{\ind{s^{(1)} \in \cS_\pi}, \cdots, \ind{s^{(dH)} \in \cS_\pi} } : \pi \in \Pi} } \\
 \overleq{\proman{3}} & \prn*{\frac{dH \cdot e (A+1)^2}{2\natarajan(\Pi)}}^{\natarajan(\Pi)} \times \prn{dH}^D. \numberthis\label{eq:b-bound}
\end{align*}
The inequality $\proman{1}$ follows by upper bounding by the Cartesian product. The inequality $\proman{2}$ follows because (1) for the first term, the vector $\prn*{ \ind{\pi\cons\tau_{1, h:h'}}, \cdots, \ind{\pi\cons\tau_{d, h:h'}} }$ is determined by the number of possible behaviors $\pi$ has over all $dH$ states in the trajectories, and (2) for the second term, the vector $\prn*{ \ind{\tau_1 \in \setall{\pi}{s}{s'} }, \cdots, \ind{\tau_d \in \setall{\pi}{s}{s'} } }$ is determined by which of the $dH$ states lie in the petal set for $\cS_\pi$. The inequality $\proman{3}$ follows by applying \pref{lem:sauer-natarajan} to the first term and Sauer's Lemma to the second term, further noting that every petal $\cS_\pi$ set has cardinality at most $D$.

Combining Eqs.~\eqref{eq:projection-bound} and \eqref{eq:b-bound} we get the final bound that
\begin{align*}
\abs{\cP^{\Pi_s, +}\big|_X}  &\le 1 + (dH)^{D+1} \cdot  \prn*{\frac{dH \cdot e (A+1)^2}{2\natarajan(\Pi)}}^{\natarajan(\Pi)}.
\end{align*}
To conclude the calculation, we observe that this bound is $< 2^d$ whenever $d = \wt{\cO}(D + \natarajan(\Pi))$, which we can again use in conjunction with \pref{lem:uniform-convergence} to prove the desired uniform convergence statement found in Eq.~\eqref{eq:uniform-convergence-probabilities}. Ultimately this allows us to replace the $\log \abs{\Pi}$ with $\wt{\cO}(D + \natarajan(\Pi))$ in the upper bound of \pref{thm:sunflower}; the precise details are omitted.

\section{Connections to Other Complexity Measures}\label{app:complexity-measures}
We show relationships between the \Compname{} and several other combinatorial measures of complexity. 

For every $h \in [H]$ we denote the state space at layer $h$ as $\cS_h \coloneqq \crl{s_{(j,h)} : j \in [K]}$ for some $K \in \bbN$. We will restrict ourselves to binary action spaces $\cA = \crl{0,1}$, but the definitions and results can be extended to larger (but finite) action spaces. In addition, we will henceforth assume that all policy classes $\Pi$ under consideration satisfy the following stationarity assumption.
\begin{assumption}\label{ass:stationary}
The policy class $\Pi$ satisfies \emph{stationarity}: for every $\pi \in \Pi$ we have 
\begin{align*}
    \pi(s_{(j,1)}) = \pi(s_{(j,2)}) = \cdots = \pi(s_{(j,H)}) \quad \text{for every}~j\in [K].
\end{align*} 
For any $\pi \in \Pi$ and $j \in [K]$, we use $\pi(j)$ as a shorthand to denote the value of $\pi(s_{(j,h)})$ for every $h$.
\end{assumption} 
The stationarity assumption is not required but is useful for simplifying the definitions and results.

\subsection{Definitions and Relationships}
First, we state several complexity measures in learning theory. For further discussion on these quantities, see \citep{foster2020instance,li2022understanding}.

\begin{definition}[Combinatorial Eluder Dimension]
Fix any stationary base policy $\bar{\pi}$. The \emph{combinatorial eluder dimension} of $\Pi$ w.r.t.~$\bar{\pi}$, denoted $\eluder(\Pi;\bar{\pi})$, is the length of the longest sequence $(j_1, \pi_1), \dots, (j_N, \pi_N)$ such that for every $\ell \in [N]$:
\begin{align*}
   \pi_\ell(j_\ell) \ne \bar{\pi}(j_\ell), \quad \text{and} \quad {\forall k < \ell},~ \pi_\ell(j_{k}) = \bar{\pi}(j_{k}).
\end{align*}
We define the combinatorial eluder dimension of $\Pi$ as $\eluder(\Pi) \coloneqq \sup_{\pi \in \Pi} \eluder(\Pi; \bar{\pi})$.\footnote{Our definition of the combinatorial eluder dimension comes from \cite{li2022understanding} and is also called the ``policy eluder dimension'' in the paper \cite{foster2020instance}. In particular, it is defined with respect to a base function $\bar{\pi}$. This differs in spirit from the original definition \citep{russo2013eluder} as well as the combinatorial variant \citep{mou2020sample}, which for every $\ell$ asks for witnessing \emph{pairs} of policies $\pi_\ell, \pi_\ell'$. Our definition is never larger than the original version since we require that $\pi_\ell' = \bar{\pi}$ to be fixed for every $\ell \in [N]$.} 
\end{definition}

\begin{definition}[Star Number \citep{hanneke2015minimax}]
Fix any stationary base policy $\bar{\pi}$. The \emph{star number} of $\Pi$ w.r.t.~$\bar{\pi}$, denoted $\starr(\Pi;\bar{\pi})$, is the length of the longest sequence $(j_1, \pi_1), \dots, (j_N, \pi_N)$ such that for every $\ell \in [N]$:
\begin{align*}
   \pi_\ell(j_\ell) \ne \bar{\pi}(j_\ell), \quad \text{and} \quad {\forall k \ne \ell},~ \pi_\ell(j_{k}) = \bar{\pi}(j_{k}).
\end{align*}
We define the star number of $\Pi$ as $\starr(\Pi) \coloneqq \sup_{\pi \in \Pi} \starr(\Pi; \bar{\pi})$.
\end{definition}

\begin{definition}[Threshold Dimension \citep{alon2019private, li2022understanding}] 
Fix any stationary base policy $\bar{\pi}$. The \emph{threshold dimension} of $\Pi$ w.r.t.~$\bar{\pi}$, denoted $\thres(\Pi;\bar{\pi})$, is the length of the longest sequence $(j_1, \pi_1), \dots, (j_N, \pi_N)$ such that for every $\ell \in [N]$:
\begin{align*} 
   {\forall m \ge \ell},~ \pi_\ell(j_{m}) \ne \bar{\pi}(j_{m}), \quad \text{and} \quad {\forall k < \ell},~ \pi_\ell(j_{k}) = \bar{\pi}(j_{k}).
\end{align*}
We define the threshold dimension of $\Pi$ as $\thres(\Pi) \coloneqq \sup_{\pi \in \Pi} \thres(\Pi; \bar{\pi})$.
\end{definition}

\paragraph{Relationships Between Complexity Measures.} From \citep[Theorem 8]{li2022understanding} we have the relationship for every $\Pi$:
\begin{align*}
    \max \crl*{\starr(\Pi), \thres(\Pi)} \le \eluder(\Pi) \le 4^{\max \crl*{\starr(\Pi), \thres(\Pi)}}. \numberthis\label{eq:equiv}
\end{align*}
The lower bound is obvious from the definitions and in general cannot be improved; the upper bound also cannot be improved beyond constant factors in the exponent \citep{li2022understanding}.

We also remark that it is clear from the definitions that VC dimension is a lower bound on all three (eluder, star, threshold); however, $\mathrm{VC}(\Pi)$ can be arbitrarily smaller.

\subsection{Bounds on Spanning Capacity}
Now we investigate bounds on the \Compname{} in terms of the aforementioned quantities.

\begin{theorem}\label{thm:bounds-on-spanning}
For any policy class $\Pi$ satisfying \pref{ass:stationary} we have
\begin{align*}
  \max\crl*{~\min\crl*{\starr(\Pi), H + 1},~\min\crl*{ 2^{\floor{ \log_2 \thres(\Pi)}}, 2^H}~} \le \dimRL(\Pi) \le 2^{\eluder(\Pi)}.
\end{align*}
\end{theorem}

We give several remarks on \pref{thm:bounds-on-spanning}. The proof is deferred to the following subsection. 

It is interesting to understand to what degree we can improve the bounds in \pref{thm:bounds-on-spanning}. On the lower bound side, we note that the each of the terms individually cannot be sharpened:
\begin{itemize}
    \item For the singleton class $\Pisingleton$ we have $\dimRL(\Pisingleton) = \min\crl{K, H + 1}$ and $\starr(\Pisingleton) = K$.
    \item For the threshold class $\Pi_\mathrm{thres} \coloneqq \crl{\pi_i(j) \mapsto \ind{j \ge i} : i \in [K]}$, when $K$ is a power of two, it can be shown that $\dimRL(\Pi_\mathrm{thres}) = \min\crl{K, 2^H}$ and $\thres(\Pi_\mathrm{thres}) = 
    K$.
\end{itemize}

While we also provide an upper bound in terms of $\eluder(\Pi)$, we note that there can be a huge gap between the lower bound and the upper bound. In fact, our upper bound is likely very loose since we are not aware of any policy class for which the upper bound is non-vacuous, i.e.~$2^{\eluder(\Pi)} \ll \min\crl*{2^H, \abs{\Pi}, 2KH}$ (implying that our bound improves on \pref{prop:dimension-ub}). It would be interesting to understand how to improve the upper bound (possibly, to scale polynomially with $\eluder(\Pi)$, or more directly in terms of some function of $\starr(\Pi)$ and $\thres(\Pi)$); we leave this as a direction for future research. 

% Using Eq.~\eqref{eq:equiv}, one can state the bounds solely in terms of $\eluder$ or both of $\starr$ and $\thres$, but there is a huge gap. It may be possible to state the upper bound more directly as some function of $\starr$ and $\thres$ that captures the interplay between the two quantities. Note that we cannot have better than exponential dependence on $\thres$ in such an upper bound because we know that for $\PiLton$, $\starr(\PiLton) = O(S/\ell)$ and $\thres(\PiLton) = \ell$, while we have $\dimRL(\PiLton) = \Theta(H^\ell)$.

Lastly, we remark that the lower bound of $\dimRL(\Pi) \ge \min\crl{\Omega(\thres(\Pi)), 2^H}$ is a generalization of previous bounds which show that linear policies cannot be learned with $\poly(H)$ sample complexity \citep[e.g.,][]{du2019good}, since linear policies (even in 2 dimensions) have infinite threshold dimension.

\subsubsection{Proof of  \pref{thm:bounds-on-spanning}} 
We will prove each bound separately.

\textbf{Star Number Lower Bound.} 
Let $\bar{\pi}\in \Pi$ and the sequence $(j_1, \pi_1), \dots, (j_N, \pi_N)$ witness $\starr(\Pi) = N$. We construct a deterministic MDP $M$ for which the cumulative reachability at layer $h_\mathrm{max} \coloneqq \min\crl{N, H}$ (\pref{def:dimension}) is at least $\min\crl{N, H+1}$. The transition dynamics of $M$ are as follows; we will only specify the transitions until $h_\mathrm{max} - 1$ (afterwards, the transitions can be arbitrary).
\begin{itemize}
    \item The starting state of $M$ at layer $h=1$ is $s_{(j_1, 1)}$.
    \item (On-Chain Transitions): For every $h < h_\mathrm{max}$,
    \begin{align*}
        P(s' \mid s_{(j_h, h)}, a) &= \begin{cases}
            \ind{s' = s_{(j_{h+1}, h+1)}} & \text{if}~a = \bar{\pi}(s_{(j_h, h)}), \\[0.5em]
            \ind{s' = s_{(j_{h}, h+1)}} & \text{if}~a \ne \bar{\pi}(s_{(j_h, h)}).
        \end{cases}
    \end{align*}
    \item (Off-Chain Transitions): For every $h < h_\mathrm{max}$, state index $\tilde{j} \ne j_h$, and action $a \in \cA$,
    \begin{align*}
        P(s' \mid s_{(\tilde{j}, h)}, a) =
            \ind{s' = s_{(\tilde{j}, h+1)}}. 
    \end{align*}
\end{itemize}
We now compute the cumulative reachability at layer $h_\mathrm{max}$. If $N \le H$, the the number of $(s,a)$ pairs that $\Pi$ can reach in $M$ is $N$ (namely the pairs $(s_{(j_1, N)}, 1), \cdots, (s_{(j_N, N)}, 1)$). On the other hand, if $N > H$, then the number of $(s,a)$ pairs that $\Pi$ can reach in $M$ is $H+1$ (namely the pairs $(s_{(j_1, H)}, 1), \cdots, (s_{(j_H, H)}, 1), (s_{(j_H, H)}, 0)$). Thus we have shown that $\dimRL(\Pi) \geq \min \crl{N, H+1}$.

\textbf{Threshold Dimension Lower Bound.}
Let $\bar{\pi}\in \Pi$ and the sequence $(j_1, \pi_1), \dots, (j_N, \pi_N)$ witness $\thres(\Pi) = N$. We define a deterministic MDP $M$ as follows. Set $h_\mathrm{max} = \min\crl{\floor{\log_2 N}, H}$. \gledit{Up until layer $h_\mathrm{max}$, the MDP will be a full binary tree of depth $h_\mathrm{max}$; afterward, the transitions will be arbitrary. It remains to assign state labels to the nodes of the binary tree (of which there are $2^{h_\mathrm{max}} - 1 \le N$). We claim that it is possible to do so in a way so that every policy $\pi_\ell$ for $\ell \in [2^{h_\mathrm{max}}]$ reaches a different state-action pair at layer $h_\mathrm{max}$. Therefore the cumulative reachability of $\Pi$ on $M$ is at least $2^{h_\mathrm{max}} = \min \crl{ 2^{\floor{\log_2 N}}, 2^H }$ as claimed.}

\gledit{It remains to prove the claim. The states of $M$ are labeled $j_2, \cdots, j_{2^{h_\mathrm{max}}}$ according to the order they are traversed using inorder traversal of a full binary tree of depth $h_\mathrm{max}$ \citep{cormen2022introduction}. One can view the MDP $M$ as a binary search tree where the action 0 corresponds to going left and the action 1 corresponds to going right. Furthermore, if we imagine that the leaves of the binary search tree at depth $h_\mathrm{max}$ are labeled from left to right with the values $1.5, 2.5, \cdots, 2^{h_\mathrm{max}} + 0.5$, then it is clear that for any $\ell \in [2^{h_\mathrm{max}}]$, the trajectory generated by running $\pi_{\ell}$ on $M$ is exactly the path obtained by searching for the value $\ell + 0.5$ in the binary search tree. Thus we have shown that the cumulative reachability of $\Pi$ on $M$ is the number of leaves at depth $h_\mathrm{max}$, thus proving the claim.}

\textbf{Eluder Dimension Upper Bound.} \gene{I also edited this, ptal.}
Let $\eluder(\Pi) = N$. We only need to prove this statement when $N \le H$, as otherwise the statement already follows from \pref{prop:dimension-ub}. Let $(M^\star, h^\star)$ be the MDP and layer which witness $\dimRL(\Pi)$. Also denote $s_1$ to be the starting state of $M^\star$. For any state $s$, we denote $\mathrm{child}_0(s)$ and $\mathrm{child}_1(s)$ to be the states in the next layer which are reachable by taking $a=0$ and $a=1$ respectively.

For any reachable state $s$ at layer $h$ in the MDP $M^\star$ we define the function $f(s)$ as follows. For any state $s$ at layer $h^\star$, we set $f(s) \coloneqq 1$ if the state-action pairs $(s,0)$ and $(s,1)$ are both reachable by $\Pi$; otherwise we set $f(s) \coloneqq 0$. For states in layers $h < h^\star$ we set 
\begin{align*}
    f(s) \coloneqq \begin{cases}
        \max \crl{f(\mathrm{child}_0(s)), f(\mathrm{child}_1(s))} + 1 &\text{if both $(s,0)$ and $(s,1)$ are reachable by $\Pi$}, \\
        f(\mathrm{child}_0(s)) &\text{if only $(s,0)$ is reachable by $\Pi$},\\
        f(\mathrm{child}_1(s)) &\text{if only $(s,1)$ is reachable by $\Pi$}.
    \end{cases}
\end{align*}

We claim that for any state $s$, the contribution to $\dimRL(\Pi)$ by policies that pass through $s$ is at most $2^{f(s)}$. We prove this by induction. Clearly, the base case of $f(s) = 0$ or $f(s) = 1$ holds. If only one child of $s$ is reachable by $\Pi$ then the contribution to $\dimRL(\Pi)$ by policies that pass through $s$ equal to the contribution to $\dimRL(\Pi)$ by policies that pass through the child of $s$. If both children of $s$ are reachable by $\Pi$ then the contribution towards $\dimRL(\Pi)$ by policies that pass through $s$ is upper bounded by the sum of the contribution towards $\dimRL(\Pi)$ by policies that pass through the two children, i.e.~it is at most $2^{f(\mathrm{child}_0(s))} + 2^{f(\mathrm{child}_1(s))} \le 2^{f(s)}$. This concludes the inductive argument.

Now we bound $f(s_1)$. Observe that the quantity $f(s_1)$ counts the maximum number of layers $h_1, h_2, \cdots, h_L$ that satisfy the following property: there exists a trajectory $\tau = (s_1, a_1, \cdots, s_H, a_H)$ for which we can find $L$ policies $\pi_1, \pi_2, \cdots, \pi_L$ so that each policy $\pi_\ell$ when run on $M^\star$ (a) reaches $s_{h_\ell}$, and (b) takes action $\pi_\ell(s_{h_\ell}) \ne a_{h_\ell}$. Thus, by definition of the eluder dimension we have $f(s_1) \le N$. Therefore, we have shown that the cumulative reachability of $\Pi$ in $M^\star$ is at most $2^N$. \qed

\section{Extension: Can Expert Feedback Help in Agnostic PAC RL?}  \label{app:expert_feedback} 
\ayush{Need another pass on this section!}  
\ayush{Will come back to this after the modified lower bound has been added!}

For several policy classes, the spanning capacity may be quite large, and our lower bounds (\pref{thm:generative_lower_bound} and \ref{thm:lower-bound-online}) demonstrate an unavoidable dependence on $\dimRL(\Pi)$. In this section, we investigate whether it is possible to achieve bounds which are independent of $\dimRL(\Pi)$ and instead only depend on $\poly(A,H, \log \abs{\Pi})$ under a stronger feedback model.

Our motivation comes from practice. It is usually uncommon to learn from scratch: often we would like to utilize domain expertise or prior knowledge to learn with fewer samples. For example, during training one might have access to a simulator which can roll out trajectories to estimate the optimal value function $Q^\star$, or one might have access to expert advice / demonstrations. However, this access does not come for free; estimating value functions with a simulator requires some computation, or the ``expert'' might be a human who is providing labels or feedback on the performance of the algorithm. Motivated by this, we consider additional feedback in the form of an expert oracle.

\begin{definition}[Expert oracle]\label{def:expert-advice-oracle}
An expert oracle $\oracle: \cS \times \cA \to \bbR$ is a function which given an $(s,a)$ pair as input returns the $Q$ value of some expert policy $\pi_\circ$, denoted $Q^{\pi_\circ}(s,a)$.
\end{definition}

\pref{def:expert-advice-oracle} is a natural formulation for understanding how expert feedback can be used for agnostic RL in large state spaces. We do not require $\pi_\circ$ to be the optimal policy (either over the given policy class $\Pi$ or over all $\cA^\cS$). The objective is to compete with $\pi_\circ$, i.e., with probability at least $1-\delta$, return a policy $\wh{\pi}$ such that $V^{\wh{\pi}} \ge V^{\pi_\circ} - \eps$ using few online interactions with the MDP and calls to $\oracle$.

A sample efficient algorithm (one which uses at most $\poly(A, H, \log \abs{\Pi}, \eps^{-1}, \delta^{-1})$ online trajectories and calls to the oracle) must use \emph{both} forms of access. Our lower bounds (\pref{thm:generative_lower_bound} and \ref{thm:lower-bound-online}) show that an algorithm which only uses online access to the MDP must use $\Omega(\dimRL(\Pi))$ samples. Likewise, an algorithm which only queries the expert oracle must use $\Omega(SA)$ queries because it does not know the dynamics of the MDP, so the best it can do is just learn the optimal action on every state.

\paragraph{Relationship to Prior Works.} The oracle $\oracle$ is closely related to several previously considered settings \citep{golowich2022can, gupta2022unpacking, amortila2022few}. Prior work \citep{golowich2022can, gupta2022unpacking} has studied tabular RL with \emph{inexact}  predictions for either the optimal $Q^\star$ or $V^\star$. They assume access to the entire table of values; since we study the agnostic RL setting with a large state space, we formalize access to predictions via the expert oracle. \cite{amortila2022few} study a related expert action oracle under the assumption of linear value functions. They show that in a generative model, with $\poly(d)$ resets and queries to an expert action oracle, one can learn an $\eps$-optimal policy, thus circumventing known hardness results for the linear value function setting. Up to a factor of $A$, one can simulate queries to the expert action oracle by querying $\oracle(s,a)$ for each $a\in \cA$.

\subsection{Upper Bound under Realizability}

Under realizability (namely, $\pi_\circ \in \Pi$), it is known that the dependence on $\dimRL(\Pi)$ can be entirely removed with few queries to the expert oracle.

\begin{theorem} \label{thm:aggrevate}
For any $\Pi$ such that $\pi_\circ \in \Pi$, with probability at least $1-\delta$, the AggreVaTe algorithm \citep{ross2014reinforcement}
computes an $\eps$-optimal policy using 
\begin{align*}
    n_1 = O\prn*{\frac{A^2H^2}{\eps^2} \cdot \log \frac{\abs{\Pi}}{\delta}} ~\text{online trajectories} \quad \text{and} \quad n_2 = O\prn*{\frac{A^2H^2}{\eps^2} \cdot \log \frac{\abs{\Pi}}{\delta}} ~\text{calls to } \oracle.
\end{align*}
\end{theorem}

The proof is omitted; it can be found in \citep{ross2014reinforcement, agarwal2019reinforcement}. We also note that actually we require a slightly weaker oracle than $\oracle$: the AggreVaTe algorithm only queries the value of $Q^{\pi_\circ}$ on $(s,a)$ pairs which are encountered in online trajectories.

\subsection{Lower Bound in Agnostic Setting}

Realizability of the expert policy used for $\oracle$ is a rather strong assumption in practice. For example, one might choose to parameterize $\Pi$ as a class of neural networks, but one would like to use human annotators to give expert feedback on the actions taken by the learner; here, it is unreasonable to assume that realizability of the expert policy holds. 

We sketch a lower bound in \pref{thm:lower-bound-expert} that shows that without realizability ($\pi_\circ \notin \Pi$), we can do no better than $\Omega(\dimRL(\Pi))$ queries to a generative model or queries to $\oracle$.

\begin{theorem}[informal]\label{thm:lower-bound-expert}
For any $H \in \bbN$, $C \in [2^H]$, there exists a policy class $\Pi$ with $\dimRL(\Pi) = \abs{\Pi} = C$, expert policy $\pi_\circ \notin \Pi$, and family of MDPs $\cM$ with state space $\cS$ of size $O(2^H)$, binary action space, and horizon $H$ such that any algorithm that returns a $1/4$-optimal policy must either use $\Omega(C)$ queries to a generative model or $\Omega(C)$ queries to the $\oracle$.
\end{theorem}

Before sketching the proof, several remarks are in order.
\begin{itemize}
    \item By comparing with \pref{thm:aggrevate}, \pref{thm:lower-bound-expert} demonstrates that realizability of the expert policy is crucial for circumventing the dependence on spanning capacity via the expert oracle.
    \item In the lower bound construction of \pref{thm:lower-bound-expert}, $\pi_\circ$ is the optimal policy. Furthermore, while $\pi_\circ \notin \Pi$, the lower bound still has the property that $V^{\pi_\circ} = V^\star = \max_{\pi \in \Pi} V^\pi$; that is, the best-in-class policy $\wt{\pi} \coloneqq \argmax_{\pi \in \Pi} V^{\pi}$ attains the same value as the optimal policy. This is possible because there exist multiple states for which $\wt{\pi}(s) \ne \pi_\circ(s)$, however these states have $d^{\wt{\pi}}(s) = 0$. Thus, we also rule out guarantees of the form $V^{\wh{\pi}} \ge \max_{\pi \in \Pi} V^\pi - \eps$. 
    \item Since the oracle $\oracle$ is stronger than the expert action oracle \citep{amortila2022few} (up to a factor of $A$), the lower bound extends to this weaker feedback model. Investigating further assumptions that enable statistically tractable agnostic learning with expert feedback is an interesting direction for future work.
\end{itemize}

\begin{proof}[Proof Sketch of \pref{thm:lower-bound-expert}]
We present the construction as well as intuition for the lower bound, leaving out a formal information-theoretic proof.

\paragraph{Construction of MDP Family.} We describe the family of MDPs $\cM$. In every layer, the state space is $\cS_h = \crl{s_{(j,h)} : j \in [2^h]}$, except at $\cS_H$ where we have an additional terminating state, $\cS_H = \crl{s_{(j,h)} : j \in [2^H]} \cup \crl{s_\bot}$. The action space is $\cA = \crl{0,1}$.

The MDP family $\cM = \crl{M_{b, f^\star}}_{b \in \cA^{H-1}, f^\star \in \cA^{\cS_H}}$ is parameterized by a bit sequence $b \in \cA^{H-1}$ as well as a labeling function $f^\star \in \cA^{\cS_H}$. The size of $\cM$ is $2^{H-1} \cdot 2^{2^H}$. We now describe the transitions and rewards for any $M_{b, f^\star}$. In the following, let $s_b \in \cS_{H-1}$ be the state that is reached by playing the sequence of actions $(b[1], b[2], \cdots, b[H-2])$ for the first $H-2$ layers.

\begin{itemize}
    \item \textbf{Transitions.} For the first $H-2$ layers, the transitions are the same for $M_{b, f^\star} \in \cM$. At layer $H-1$, the transition depends on $b$.
    \begin{itemize}
        \item For any $h \in \crl{1, 2, \dots, H-2}$, the transitions are deterministic and given by a tree process: namely 
    \begin{align*}
      P(s' \mid s_{(j,h)}, a) = \begin{cases}
          \ind{s' = s_{(2j -1, h+1)}} &\text{if}~a =0,\\
          \ind{s' = s_{(2j, h+1)}} &\text{if}~a =1.
      \end{cases}  
    \end{align*}
        \item At layer $H-1$, for the state $s_b$, the transition is $P(s' \mid s_b, a) = \ind{s' = s_\bot}$ for any $a \in \cA$. For all other states, the transitions are uniform to $\cS_H$, i.e., for any $s \in \cS_{H-1} \backslash \crl{s_b}$, $a \in \cA$, the transition is $P(\cdot \mid s, a) = \unif(\cS_H \backslash \crl{s_\bot})$.
    \end{itemize}
    \item \textbf{Rewards.} The rewards depend on the $b \in \cA^{H-1}$ and $f^\star \in \cA^{\cS_H}$. 
    \begin{itemize}
        \item The reward at layer $H-1$ is $R(s,a) = \ind{s = s_b, a = b[H-1]}$.
        \item The reward at layer $H$ is 
        \begin{align*}
            &R(s_\bot, a) = 0 &\text{for any $a \in \cA$,}\\
            &R(s,a) = \ind{a = f^\star(s)} &\text{for any $s \ne s_\bot$, $a \in \cA$.}
        \end{align*}
    \end{itemize}
\end{itemize}

From the description of the transitions and rewards, we can compute the value of $Q^\star(\cdot,\cdot)$.
\begin{itemize}
    \item \textit{Layers $1, \cdots, H-2$:} For any $s \in \cS_1 \cup \cS_2 \cup \cdots \cup \cS_{H-2}$ and $a \in \cA$, the $Q$-value is $Q^\star(s,a) = 1$.
    \item \textit{Layer $H-1$:} At $s_b$, the $Q$-value is $Q^\star(s_b,a) = \ind{a = b[H-1]}$. For other states $s \in \cS_{H-1} \backslash \crl{s_b}$, the $Q$-value is $Q^\star(s,a) = 1$ for any $a \in \cA$.
    \item \textit{Layer $H$:} At $s_\bot$, the $Q$-value is $Q^\star(s_\bot,a) = 0$ for any $a \in \cA$. For other states $s \in \cS_{H} \backslash \crl{s_\bot}$, the $Q$-value is $Q^\star(s,a) = \ind{a = f^\star(s)}$.
\end{itemize}

Lastly, the optimal value is $V^\star = 1$.

\paragraph{Expert Oracle.}
The oracle $\oracle$ returns the value of $Q^\star(s,a)$.

\paragraph{Policy Class.} The policy class $\Pi$ is parameterized by bit sequences of length $H-1$. Denote the function $\mathrm{bin}: \crl{0, 1, \dots, 2^{H-1}} \mapsto \cA^{H-1}$ that returns the binary representation of the input. Specifically,
\begin{align*}
    \Pi \coloneqq \crl*{\pi_b: b \in \crl*{\mathrm{bin}(i) : i \in \crl*{0, 1, \dots, C-1}}},
\end{align*}
where each $\pi_b$ is defined such that $\pi_b(s) \coloneqq b[h]$ if $s \in \cS_h$, and $\pi_b(s) \coloneqq 0$ otherwise. By construction it is clear that $\dimRL(\Pi) = \abs{\Pi} = C$.

\paragraph{Lower Bound Argument.}
Consider any $M_{b, f^\star}$ where $b \in \crl*{\mathrm{bin}(i) : i \in \crl*{0, 1, \dots, C-1}}$ and $f^\star \in \cA^{\cS_H}$. There are two ways for the learner to identify a $1/4$-optimal policy in $M_{b, f^\star}$:
\begin{itemize}
    \item Find the value of $b$, and return the policy $\pi_b$, which has $V^{\pi_b} = 1$.
    \item Estimate $\wh{f} \approx f^\star$, and return the policy $\pi_{\wh{f}}$ which picks arbitrary actions for any $s \in \cS_1 \cup \cS_2 \cup \dots \cup \cS_{H-1}$ and picks $\pi_{\wh{f}}(s) = \wh{f}(s)$ on $s \in \cS_H$.
\end{itemize}

We claim that in any case, the learner must either use many samples from a generative model or many calls to $\oracle$. First, observe that since the transitions and rewards at layers $1, \cdots, H-2$ are known and identical for all $M_{b, f^\star} \in \cM$, querying the generative model on these states does not provide the learner with any information. Furthermore, in layers $1, \cdots, H-2$, every $(s,a)$ pair has $Q^\star(s,a) = 1$, so querying $\oracle$ on these $(s,a)$ pairs also does not provide any information to the learner. Thus, we consider learners which query the generative model or the expert oracle at states in layers $H-1$ and $H$. 

In order to identify $b$, the learner must identify which $(s,a)$ pair at layer $H-1$ achieves reward of 1. They can do this either by (1) querying the generative model at a particular $(s,a)$ pair and observing if $r(s,a) = 1$ (or if the transition goes to $s_\bot$); or (2) querying $\oracle$ at a particular $(s,a)$ pair and observing if $Q^\star(s,a) = 0$ (which informs the learner that $s_b = s$ and $b[H-1] = 1-a$). In either case, the learner must expend $\Omega(C)$ queries in total in order to identify $b$.

To learn $f^\star$, the learner must solve a supervised learning problem over $\cS_H \backslash s_\bot$. They can learn the identity of $f^\star(s)$ by querying either the generative model or the expert oracle on $\cS_H$. Due to classical supervised learning lower bounds, learning $f^\star$ requires $\Omega \prn*{ \mathrm{VC}(\cA^{\cS_{H}}) } = \Omega(2^H)$ queries.
\end{proof}  
\section{Technical Tools}
\begin{lemma}[Hoeffding's Inequality]\label{lem:hoeffding}
Let $Z_1, \cdots, Z_n$ be independent bounded random variables with $Z_i \in \brk{a,b}$ for all $i \in [n]$. Then
\begin{align*}
    \Pr \brk*{\abs{\frac{1}{n} \sum_{i=1}^n Z_i - \En[Z_i]} \ge t} \le 2\exp\prn*{-\frac{2nt^2}{(b-a)^2}}.
\end{align*}
\end{lemma} 

\begin{lemma}[Multiplicative Chernoff Bound]\label{lem:chernoff}
Let $Z_1, \cdots, Z_n$ be i.i.d.~random variables taking values in $\crl{0,1}$ with expectation $\mu$. Then for any $\delta > 0$,
\begin{align*}
    \Pr \brk*{\frac{1}{n} \sum_{i=1}^n Z_i \ge  \prn*{1+ \delta} \cdot \mu} \le \exp\prn*{-\frac{\delta^2 \mu n }{2+\delta}}.
\end{align*}
Furthermore for any $\delta \in (0,1)$,
\begin{align*}
    \Pr \brk*{\frac{1}{n} \sum_{i=1}^n Z_i \le  \prn*{1- \delta} \cdot \mu} \le \exp\prn*{-\frac{\delta^2 \mu n }{2}}.
\end{align*}
\end{lemma} 

\end{document}